\documentclass[letterpaper]{article} 
\usepackage{aaai2026}  
\usepackage{times}  
\usepackage{helvet}  
\usepackage{courier}  
\usepackage[hyphens]{url}  
\usepackage{graphicx} 
\usepackage{natbib}  
\usepackage{caption} 
\usepackage{algorithm}
\usepackage{algorithmic}

\usepackage{newfloat}
\usepackage{listings}
\DeclareCaptionStyle{ruled}{labelfont=normalfont,labelsep=colon,strut=off} 
\floatstyle{ruled}
\newfloat{listing}{tb}{lst}{}
\floatname{listing}{Listing}
\usepackage{imports_aaai}

\usepackage{microtype}
\usepackage{graphicx}
\usepackage{booktabs} 

\usepackage{style_files/algorithm}
\usepackage{style_files/algorithmic} 

\usepackage{amsmath}
\usepackage{amssymb}
\usepackage{mathtools}
\usepackage{amsthm}

\usepackage[capitalize,noabbrev]{cleveref}

\usepackage{xcolor}
\usepackage{url}
\usepackage[export]{adjustbox}
\numberwithin{equation}{section}
\usepackage{xltabular}
\usepackage{booktabs}
\usepackage[font={footnotesize}]{caption}
\usepackage[en-GB,showdow=false]{datetime2}

\usepackage{enumerate}
\usepackage[shortlabels]{enumitem}
\usepackage{multicol}

\usepackage[textsize=tiny]{todonotes}

\title{Riemannian Manifold Learning for Stackelberg Games with Neural Flow Representations}

\author{
    Larkin Liu\textsuperscript{\rm 1,2, *} 
    Kashif Rasul\textsuperscript{\rm 3},
    Yutong Chao\textsuperscript{\rm 1},
    Jalal Etesami\textsuperscript{\rm 1}
}
\affiliations{
    \textsuperscript{\rm 1}Technische Universität München\\
    \textsuperscript{\rm 2}Riebaki AI\\
    \textsuperscript{\rm 3}Hugging Face, Inc.\\
    \textsuperscript{\rm *}Corresponding author: \texttt{larkin.liu@tum.de}
}

\begin{document}

\maketitle

\begin{abstract}
We present a novel framework for online learning in Stackelberg general-sum games, where two agents, the leader and follower, engage in sequential turn-based interactions. At the core of this approach is a learned diffeomorphism that maps the joint action space to a smooth spherical Riemannian manifold, referred to as the \textit{Stackelberg manifold}. This mapping, facilitated by neural normalizing flows, ensures the formation of tractable isoplanar subspaces, enabling efficient techniques for online learning. Leveraging the linearity of the agents' reward functions on the Stackelberg manifold, our construct allows the application of linear bandit algorithms. We then provide a rigorous theoretical basis for regret minimization on the learned manifold and establish bounds on the simple regret for learning Stackelberg equilibrium. This integration of manifold learning into game theory uncovers a previously unrecognized potential for neural normalizing flows as an effective tool for multi-agent learning. We present empirical results demonstrating the effectiveness of our approach compared to standard baselines, with applications spanning domains such as cybersecurity and economic supply chain optimization.
\end{abstract}

\section{Introduction}


A Stackelberg game consists of a sequential decision-making process involving two agents, a leader and a follower. This framework, introduced in \citep{stackelberg:1934} models hierarchical strategic interactions where the leader moves first, anticipating the follower's best response, and then the follower reacts accordingly. These games have become central to understanding interactions in various fields, from economics to societal security, providing a formal method for analyzing situations where one party commits to a strategy before the other, affecting the subsequent decision-making process and reward outcomes. Over time, Stackelberg games have evolved to address more complex environments, incorporating factors like imperfect information and no-regret learning of system parameters. The solution revolves around finding a Stackelberg equilibrium, where the leader optimizes her strategy assuming the follower type, which affects how the follower optimizes his utility based on the leader's action. \citep{Korzhyk:2010,Tambe:2015_behave_stack}.


Practical applications of Stackelberg games often face key challenges, primarily due to uncertainty about the follower's rationality or preferences, and imperfect information on reward outcomes. These issues complicate the leader's decision-making. In security domains, randomized strategies and robust optimization mitigate risks from incomplete information, as seen in systems like ARMOR and PROTECT \citep{Jiang:2013_monotonic_maxmin, Tambe:2015_behave_stack, kar:2017_trends_ssg, Jain:2011_double_oracle, Shieh:2012}. Stackelberg games are also used in supply chain optimization, addressing uncertainties like demand \citep{liu:2024_stacknews_adt, cesa:2022-supp-chain-games}, and in conversational AI, where agents anticipate user behaviour to adjust responses \citep{Nguyen:2014}. For non-cooperative multi-agent games with additive noise, sublinear regret is achievable via gradient-based methods like AdaGrad \citep{duchi:2011-adagrad}, though constrained by noise magnitude \citep{hsieh:2023-convexgames}. These settings often extend to unlimited players, with regret worsening as the player count grows. We focus on two-player Stackelberg games with well-defined best response functions, common in economics and adversarial machine learning \citep{zhou:2016_stack_model, wang:2024_stackRL}. 


\textbf{Problem Setting:} We consider a two-player Stackelberg game where player A leads and player B responds. Stackelberg games are sequential, meaning that the players take turns and, the follower can \textit{best respond} to the leader's action, given information available to him. The best response of player B lies on a manifold within a subspace of the joint action space $\setA \cross \setB$. We define this Stackelberg game setting in the framework of optimal transport, where the structure of the best response function $\bR{\cdot}$ gleans simplifications to the solution methodology to obtain Stackelberg regret. This research focuses on applying multi-armed bandit (MAB) methods, particularly in Stackelberg equilibrium settings, to achieve sublinear regret. It explores the utilization of geometric topologies to better understand agent behaviour and simplify computations in a game theoretic manner. 


\textbf{Key Contributions:} We introduce a novel algorithm that advances Stackelberg learning under imperfect information, akin to \citep{balcan:2015-stack-learn-sec} and \citep{haghtalab:2022-stack-non-myo}, providing a systematic framework for efficiently solving equilibrium in such settings. Central to our work is a feature map using neural normalizing flows, transforming the joint action space into a tractable embedding, the \textit{Stackelberg manifold}. Leveraging its geodesic properties and exploiting the linearity of the agents' reward functions on the manifold, our approach enables efficient computation of Stackelberg equilibria under no-regret learning, especially with parameter uncertainty. We also establish a rigorous theoretical foundation for optimizing Stackelberg games on spherical manifolds. Empirical simulations in supply chain management and cybersecurity validate our method, showing superior computational efficiency and regret minimization compared to standard baselines.

\section{Formal Definitions}

In a Stackelberg game, two players take turns executing their actions. Player A is the leader, she acts first with action $\aB$ selected from her action space $\setA$. Player B is the follower, he acts second with action $\bB\in\setB$. The follower acts in response to the leader's action, and both players earn a joint payoff as a function of their actions. 

\subsection{Repeated Stackelberg Games}

In a repeated Stackelberg game, the leader chooses actions $\mathbf{a}^t \in \mathcal{A}$, and the follower reacts with actions $\mathbf{b}^t \in \mathcal{B}$ at each round $t\!=\! 1, \dots, T$. 
The leader's strategy $\pi_A(\cdot|\mathcal{H}_t)$ is a probability distribution over the action space $\mathcal{A}$ which selects $\mathbf{a}^t$ based on past joint actions up to time $t$, i.e., $\mathcal{H}_t\!:=\!\{(\mathbf{a}^\tau, \mathbf{b}^\tau)| \tau\!<\!t\}$. 
Similarly, the follower's strategy $\pi_B(\cdot|\mathcal{H}_t)$ is a conditional probability distribution over $\mathcal{B}$ which determines $\mathbf{b}^t$ given the full history,   $\mathcal{H}_t:=\mathcal{H}_t\cup\{\mathbf{a}^t\}$. 

\textbf{Best Response Strategy of the Follower:} To be specific, the follower selects his best response strategy at round $t$ by maximizing his expected reward function $\utlB(\aB, \bB):\mathcal{A}\times\mathcal{B}\rightarrow\mathbb{R}$ given that the leader has played action $\aB^t$. 
Since we assume that the reward function solely depends on the most recent pairs of actions, the follower's best strategy is first order Markov, i.e., $\pi_B(\cdot|\mathcal{H}_t)=\pi_B(\cdot|\aB^t)$. 
Formally, the follower's best response at round $t$ is given by,

\vspace{-0.4cm}

\begin{align}
    \pi^*_{B}(\bB|\aB^t) &\equiv \underset{ \polB \in \Pi_{\setB} }{\mathrm{argmax}} \ \expeC_{\pi_B} [\utlB(\aB, \bB)| \aB=\aB^t], \label{eq:best_str_b} \\
    \mathfrak{B}(\aB^t) &\equiv \{\bB\in\setB| \pi^*_{B}(\bB|\aB^t)>0\}. \label{eq:br-def}
\end{align}

\vspace{-0.1cm}




where $\Pi_{\setB}$ is the space of probability distributions over the action space $\mathcal{B}$ and the expectation is taken with respect to the strategy of the follower. 
In this case, we can define the set of follower's best responses in Eq. \eqref{eq:br-def}. Analogously, the leader aims at maximizing the expected utility $\expeC [\utlA(\aB^t, \bB^t)]:\mathcal{A}\times\mathcal{B}\rightarrow\mathbb{R}$ that is a deterministic function solely driven by her action $\aB^t$ followed by the reaction of the follower $\bB^t$.





\textbf{Stackelberg Equilibrium:} Consider a follower whose best response is optimal. We denote this scenario as Stackelberg Oracle (SOC) learning. From the leader's perspective, the uncertainty is not necessarily over the system, but rather the strategy of the follower $\polB(\cdot)$. \textit{Stackelberg equilibrium} $(\pi^*_A,\pi_B^*)$ is achieved when the follower is best responding, according to Eq. \eqref{eq:br-def}, and the leader acts with an optimal policy given the best response of the follower,

\vspace{-0.25cm}
\begin{align}\label{eq:opt_leader_mu_a}
    \pi^*_A &\equiv \arg\max_{\pi_A \in \Pi_{\setA}} \expeC_{\pi_A,\pi^*_{B}}[\mu_A(\aB, \bB)],
 \\ \notag
\expeC_{\pi_A, \pi^*_B}[\mu_A] &= \int_{\setA} \pi_A(\aB) \int_{\setB} \mu_A(\aB, \bB) \pi^*_B(\bB | \aB) \, d\bB \, d\aB.
\end{align}

\vspace{-0.2cm}

\begin{figure}
    \centering
    \includegraphics[width=0.9\linewidth]{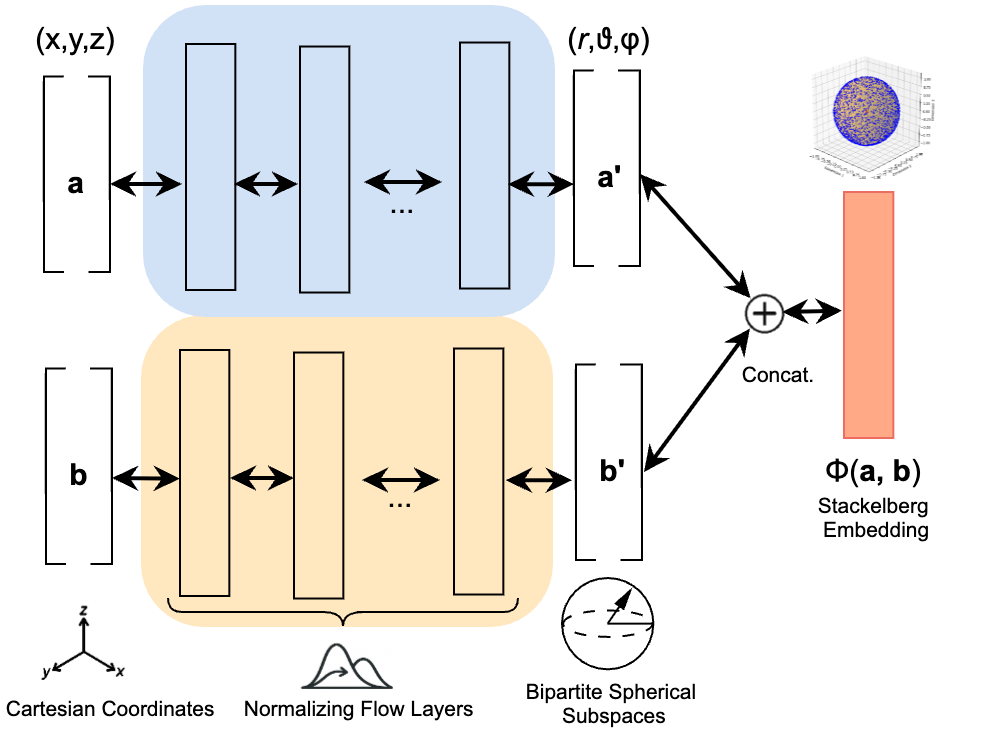}
    \caption{\textbf{Bipartite \& Bijective Neural Flow Architecture:} We illustrate the bipartite structure of the normalizing flow architecture. Two players present joint actions $(\aB, \bB)$, where each vector is mapped separately through a series of bijective transforms consisting of normalizing flow layers. Each player's action independently controls one subspace of the spherical manifold. The sequence of bijective transformations retain a fully bijective network from the ambient joint action space to the manifold space $\stEmb(\aB, \bB)$. The network is invertible by design, and features a bipartite input.}
    \label{fig:bipartite_nf}
    \vspace{-11.8pt}
\end{figure}

\subsection{The Stackelberg Manifold} \label{sec:stack_embedding}

To address the complexity of solving for Stackelberg equilibrium under uncertainty, we propose mapping actions from the ambient joint action space onto a well behaved spherical manifold $\stEmb$. This approach offers key advantages. First it simplifies the problem by optimizing on an intrinsic geometric structure (e.g., a unit sphere) enabling faster computation and convenient constraint enforcement. The manifold's smoothness also allows for efficient optimization via methods like Riemannian gradient descent \citep{bonnabel:2013_riemman_sgd}. The core idea is to shift the complexity of learning in Stackelberg games by transforming the action space into a more tractable representation. Rather than relying on classical multi-agent learning, we instead learn a neural representation for $\stEmb$ to enable simplified equilibrium learning.



This concept of mapping the data from the ambient space, in our case defined by the joint action space $\setA \cross \setB$, onto a latent space $\stEmb$ has been explored in several prior works. For a well defined manifold, the typical approach is to learn a diffeomorphism between the ambient data space, and the objective manifold, which is a subspace of the ambient data space \citep{rezende:2020_tori_sphere} \citep{gemici:2016_riemann_mani}. Suppose the manifold is not given, or there lies flexibility in defining the structure of such a manifold, certain manifold learning techniques could be devised \citep{brehmer:2020_manifold_flows}. These approaches typically define invertible probability density maps between the ambient data space, the latent space, and the manifold space.



\subsection{Normalizing Flows for Action Space Projection}

We leverage \textit{normalizing flows} to map a compact joint action space $\mathcal{A} \times \mathcal{B} \subseteq [0, K]^D$ with arbitrarily large $K \in \mathbb{R}$, onto a manifold, $\stEmb$ embedded in $\mathbb{R}^D$ \citep{rezende:2015variational, dinh:2016density, papamakarios:2021normalizing_flows}. Normalizing flows are a class of generative models that transform a high dimensional simple distribution into a complex one through a series of invertible bijective mappings using neural networks that are computationally tractable. The joint action space consists of actions taken by two agents, denoted as $\aB \in \mathcal{A}$ and $\bB \in \mathcal{B}$, modelled via normalizing flows to ensure bijectivity and a tractable density estimate. Let $x \in \setA \cross \setB$, the model density $p_X(x)$ for a data point $x \in \mathbb{R}^D$ is given by,
\begin{equation}
    p_X(x) = p_Z(\funcNf(x)) \left| \det \left( \frac{\partial \funcNf(x)}{\partial x} \right) \right|. \label{eq:jacobian_norm_flow}
\end{equation}

Here, $Z$ represents the latent space with a simple distribution, and $\left| \det \left( \partial \funcNf(x) / \partial x \right) \right|$ is the Jacobian determinant of the transformation $\funcNf: \mathbb{R}^D \rightarrow \mathbb{R}^D$. Several open-source methodologies and codebases have been developed to address this manifold mapping problem via normalizing flows \citep{brehmer:2020_manifold_flows}. We adopt the \texttt{nflows} package from \citep{durkan:2020_nflows} into our approach. The key contribution of our application is the isolation of the input heads into two separate partitions of normalizing flows followed by concatenation of the outputs (see Fig. \ref{fig:bipartite_nf}). This allows us to control the subspace induced by the leader's action, $\aB \in \setA$, independently. (We provide detailed model specifications in Appendix \ref{sec:nn_arch_details}.) The neural flow network is designed to be invertible, but could  experience some reconstruction error due to numerical instability or rounding errors - which we aim to minimize. Our empirical results demonstrate that this error is negligible (see Table \ref{tab:flow_results}).

\subsection{Specifications of the Feature Map $\phi(\aB,\bB)$} \label{sec:stack_emb_ass}

\paragraph{Feature Map $\phi(\cdot)$:} Common in the linear bandit literature, we propose a \textit{feature map} $\phi: \mathcal{A} \cross \mathcal{B} \mapsto \mathbb{R}^D$ which maps the joint action spaces of the agents, $\mathcal{A} \cross \mathcal{B}$, to $ \mathbb{R}^D$  \citep{zanette:2021-doe_bandit, moradipari:2022feature_map, amani:2019linear_bandit_safety}. Further, we introduce a concept known as the \textit{Stackelberg manifold}, denoted by $\stEmb$, which is defined as the image of $\phi$ over the joint action space domain $\mathcal{A} \cross \mathcal{B}$,
\begin{align}
    \stEmb\equiv Im(\phi)=\{\phi(\aB,\bB) | \aB\in\mathcal{A}, \bB\in\mathcal{B}\}.
\end{align}

In principle, the mapping $\phi$ can be constructed via any means. In our case, a normalizing neural flow network (but possibly any other architecture), but should abide by imposed characteristics described later in this Section. To be precise, $\hat{\phi}$ should denote our learned representation of the \textit{ideal} map $\phi$. Provided that we only have access to $\hat{\phi}$, purely for notational convenience, we will use $\phi$ to represent $\hat{\phi}$ moving forward. 



\begin{definition} \label{def:bipartite_sphere_map}
    \textbf{Bipartite Spherical Map $\phi(\aB, \bB)$:} Let $\mathbf{a} \in \setA$ and $\mathbf{b} \in \setB$, and define a mapping $\phi: \setA \times \setB \to \mathcal{S}^{(D-1)}$ from Cartesian coordinates to spherical coordinates on the $D$-dimensional unit sphere $\mathcal{S}^{(D-1)}$. The spherical coordinates are partitioned such that , $\mathbf{a}$ parametrizes a subset of the spherical coordinates $\nu_\mathbf{a}(\mathbf{a})$, and $\mathbf{b}$ parametrizes the remaining coordinates $\nu_\mathbf{b}(\mathbf{b})$. Also, $\nu_\mathbf{a} \cap \nu_\mathbf{b} = \emptyset$, meaning the partitions are disjoint. Thus, the full mapping is given by, 
    \[
    \phi(\mathbf{a}, \mathbf{b}) := \big(
    \nu_\mathbf{a}(\mathbf{a}), \nu_\mathbf{b}(\mathbf{b})\big)^\intercal\in\mathcal{S}^{(D-1)}.
    \]
\end{definition}

\textbf{Mapping to a Spherical Manifold:} The transformation from spherical coordinates to Cartesian coordinates is used to map input features onto an $D$-dimensional spherical manifold. Therefore, in addition to the properties of our feature map $\phi$, we also enforce $\phi$ as a bipartite spherical map from Def. \ref{def:bipartite_sphere_map}. This bipartite spherical map constructs a disjoint spherical mapping to parameterize two subspaces in $\stEmb$. Given two heads in the neural architecture, the head from A specifically controls the azimuthal spherical coordinate and the head from B controls other coordinates. The justification for this mapping involves trade-off between learning an optimal embedding and reducing the complexity inherent in the native multi-agent  problem. When specific multi-agent optimization problems are too complex to solve in their native forms (see Section \ref{sec:empirical_experiments_main} for examples), we leverage normalizing flows as an enabling link \citep{brehmer:2020_manifold_flows, durkan:2020_nflows} to transform the problem into a simpler representation. (A visualization of the empirical mapping results, showcasing the learned bipartite mapping to $\stEmb$ as a 3D spherical surface, is provided in Appendix \ref{sec:isoplanes_computational_viz} and \ref{sec:ambient_to_stemb_viz}. This visualization is generated by varying $\aB$ or $\bB$ to create longitudinal or latitudinal subspaces.)



\textbf{Manifold Learning:} To construct the Stackelberg manifold $\stEmb$, data is first sampled uniformly from the ambient Cartesian space. Given the architecture (see Fig. \ref{fig:bipartite_nf}), the Cartesian data is fed through the neural flow architecture to the corresponding image. We train the model parameters such that the image satisfies properties of invertibility for accurate reconstruction. In addition, we train the model to produce an ideal $\stEmb$, which should be a smooth Riemannian manifold, be  measurable, compact, and forms Lipschitz-continuous image corresponding to the domain $\setA \cross \setB$, naturally admitting geometric analysis (see Appendix \ref{sec:ideal_stack_manifold} for details). Furthermore, we would like to ensure maximal spread across the surface, and ensure stability under small perturbations. We construct this manifold by minimizing a loss function $\mathcal{L}(\phi)$, denoted in Eq. \eqref{eq:manifold_total_loss}.

\vspace{-0.2cm}

\begin{align} 
    \mathcal{L}(\phi) &= \alpha_N \nfLoss 
    + \alpha_R \repuLoss   + \alpha_L \, \underbrace{\Big| \norm{\nabla_{\aB}\phi} + \norm{\nabla_{\bB}\phi} - C \Big|}_{\text{Lipschitz Loss:} \, \lipschitzLoss} \nonumber \\
    &+ \alpha_P \, \underbrace{\text{Var}\Big(\phi(\aB, \bB) - \phi(\mathfrak{J}_\sigma(\aB, \bB)) \Big)}_{\text{Perturbation Loss:} \, \perturbLoss}. \label{eq:manifold_total_loss}
\end{align}

\vspace{-0.1cm}


\textbf{Loss Function Descriptions:} $\mathcal{L}(\phi)$ is an aggregate loss function composed of a convex combination of separate loss functions designed to achieve the ideal manifold behaviour. The normalizing flow loss $\nfLoss$ penalizes misalignment of transformed data with the base distribution while accounting for volume changes from invertible transformations, per Eq. \eqref{eq:jacobian_norm_flow}. Minimizing $\nfLoss$ enables efficient bijective mapping from complex data to simpler distributions (see Appendix \ref{sec:nll_loss_defn}). $\repuLoss$ is the geodesic repulsion loss, penalizing close pairwise elements to maximize coverage over the target manifold (see Appendix \ref{sec:geodesic_repul_loss_defn}). The Lipschitz loss, $\lipschitzLoss$, penalizes large gradient deviations with respect to $\aB$ and $\bB$, keeping the sum of absolute gradients near target $C \in \mathbb{R}$. $\mathfrak{J}_\sigma(\aB, \bB): \setA \times \setB \mapsto \setA \times \setB$ is a Gaussian perturbation function on the joint action space (defined in Appendix \ref{sec:perturb_function_defn}). The perturbation loss, $\perturbLoss$, denoted as the variance between $\phi(\aB, \bB)$ and $\phi(\mathfrak{J}\sigma(\aB, \bB))$, should be minimized further ensure Lipschitzness. (We present detailed justification of the loss function design in Appendix \ref{sec:ideal_stack_manifold}.)

\begin{table}[h]
\centering
\small
\begin{tabular}{
    >{\centering\arraybackslash}p{0.4cm}  
    >{\centering\arraybackslash}p{0.6cm}  
    >{\centering\arraybackslash}p{0.5cm}  
    >{\centering\arraybackslash}p{0.5cm}  
    >{\centering\arraybackslash}p{0.2cm}  
    >{\centering\arraybackslash}p{1.7cm}  
    >{\centering\arraybackslash}p{1.5cm}  
}
\toprule
\textbf{NFL} & $N$ & $\mathsf{Dim}_\setA$ & $\mathsf{Dim}_\setB$ & $D$ & $\nfLoss$ & $\mathcal{L}(\phi)$ \\
\midrule
2  & 1,000  & 2  & 2  & 4  & $1.78 \times 10^{-7}$  & $2.9 \times 10^{-2}$ \\
4  & 10,000 & 10 & 5  & 15 & $8.68 \times 10^{-3}$  & $4.5 \times 10^{-1}$ \\
6  & 50,000 & 10 & 10 & 30 & $9.66 \times 10^{-2}$  & $1.1 \times 10^{-3}$ \\
4  & 10,000 & 3  & 3  & 7  & $9.38 \times 10^{-8}$  & $5.6 \times 10^{-2}$ \\
6  & 50,000 & 5  & 5  & 10 & $9.86 \times 10^{-2}$  & $6.8 \times 10^{-1}$ \\
2  & 1,000  & 10 & 10 & 20 & $1.08 \times 10^{-6}$  & $2.9 \times 10^{-3}$ \\
\bottomrule
\end{tabular}
\caption[Flow layer configuration results]{Reconstruction error across neural flow configurations and joint action dimensions. All samples, with size $N$, were trained on a space uniformly sampled from $[0,1]$, on dimensions $\mathsf{Dim}_\setA$ and $\mathsf{Dim}_\setB$, varying the number of NF layers (NFL), and Stackelberg manifold dimension $D$. All samples trained on 20,000 epochs, and reconstruction error conducted on separately sampled test set of the same sample size.}
\label{tab:flow_results}
\end{table}

\vspace{-0.4cm}

\subsection{Reward Function} \label{sec:reward_function}

\paragraph{Reward Mechanisms:} A Stackelberg game provides two reward functions $\utlA(\aB, \bB)$ and $\utlB(\aB, \bB)$. Both of which are linearizable with sub-Gaussian noises, $\subA$ and $\subB$, i.e.,

\vspace{-0.2cm}

\begin{align}
    \utlA(\aB, \bB) = \langle \theta^*_A, \phi(\aB, \bB) \rangle + \subA, \label{eq:A_reward_inner_prod} \\
    \utlB(\aB, \bB) = \langle \theta^*_B, \phi(\aB, \bB) \rangle + \subB. \label{eq:B_reward_inner_prod}
\end{align}




We assume zero-mean sub-Gaussian distribution for both $\subA$ and $\subB$. The objective is to learn the parameters $\theta^*_A\in\mathbb{R}^D$, and possibly as an extension problem $\theta^*_B$. The parameters of the model can be estimated via parameterized regression,

\begin{align}
    \hat{\theta}_t = (\phi_{1:t} \phi_{1:t}^\top + \lambda_{\text{reg}} I)^{-1} \phi_{1:t}^\top \, \mu_{1:t},  \label{eq:theta_estimation_xtx}
\end{align}

where, for A and B, respectively, $\phi_{1:t}$ represents the sequence of $\phi(\cdot)$ values via the feature map given the action sequences $\aB_{1:t}$ and $\bB_{1:t}$, $\lambda_{\text{reg}}$ serves as a regularization parameter, $I$ is the identity matrix, and $\mu_{1:t}$ are the historical rewards of players A or B (depending on the subscript). Here, we extend the reward structure of classical linear bandits in \citep{abbasi:2011ofu} to a setting where two players jointly decide on the action sequence. (We provide the conditions for which the estimator is consistent in Appendix \ref{sec:ass_lin_reward}.)

\begin{lemma} \label{lem:existence_of_theta}
    \textbf{Linear Relation for Smooth Invertible Maps:} Suppose that \(Y\) can be expressed as \(Y = \langle \tilde{\theta}, \tilde{\phi}(X) \rangle\), where \(\tilde{\phi}: \mathbb{R}^d \to \mathbb{R}^d\) is smooth and bijective, and \(\tilde{\theta} \in \mathbb{R}^d\). Then, for any \(k \geq d\), there exists an alternative set of parameters \(\theta \in \mathbb{R}^k\) and corresponding map  \(\phi: \mathbb{R}^d \to \mathbb{R}^k\) such that, $Y = \langle \theta, \phi(X) \rangle$. (Please see Appendix \ref{prf:existence_of_theta} for proof.)
\end{lemma}

\textit{Proof Sketch:} The result follows by embedding $\tilde{\phi}(X)$ into $\mathbb{R}^k$ ($k \geq d$) via an alternative map while preserving the inner product structure. Then we construct a diffeomorphism $T(y) = \phi(\tilde{\phi}^{-1}(y))$ between $\tilde{\phi}(X)$ and any alternative smooth bijection $\phi(X)$. By defining $\theta = (J_T^\top)^{-1}\tilde{\theta}$ via the Jacobian of $T$ to ensure $Y = \langle \theta, \phi(X) \rangle$ holds.

\paragraph{Linearity by Design:} Lemma \ref{lem:existence_of_theta} provides us a theoretical basis for exchanging one feature map $\tilde{\phi}$, to another $\phi$, so long as the original feature map, $\tilde{\phi}$, is smooth and bijective. This preserves the functional representation of the primal expression $\innerP{\tilde{\theta}, \tilde{\phi}}$ under $\innerP{\theta, \phi}$. Therefore, in principle, an infinite amount of valid alternative mappings of $\phi$ exist, allowing us to construct a mapping with desired properties. Subsequently, we could adopt the linear bandit framework \citep{chu:2011-linucb, cesa:2006_pred_learn_games, lattimore:2020_bandit_book} for our Stackelberg manifold $\stEmb$, ensuring that the structure of the reward function remains equivalent for any alterative feature map $\phi$.

\section{Optimization of Stackelberg Games} \label{sec:optim_of_stackelberg_games}

\paragraph{Optimization under Parameter Uncertainty:} In general we can solve Stackelberg games as a bilevel optimization problem. Suppose that after observing $t$ samples under  parameter uncertainty, for some no-regret learning algorithm, the uncertainty among parameters $\theta$, is characterized by,
\begin{align}\label{eq:param_uncertainty}
    \Ball(\theta^*,\uncerBall{t}):=\Big\{\theta: \norm{\theta^* - \theta} \leq \uncerBall{t}\Big\},
\end{align}

with probability at least $1-\delta(t)$. In this formulation, $||\cdot||$ denotes some norm in the space of parameters. Assuming a \textit{pessimistic} leader, the optimization problem under parameter uncertainty at round $t$ can be expressed as,


\vspace{-0.4cm}

\begin{align}
     &\pi_A^* \equiv \arg\max_{\polA\in \Pi_A} \min_{\theta_A}\, \mathbb{E}_{\polA,\polB^*(\polA)}\big[\innerP{\theta_A, \, \phi(\aB, \bB) }\big],\label{eq:pi_a_star}\\
     &\pi_B^*(\polA) \equiv \argmaX{\polB\in \Pi_B} \max_{\theta_B}\,\mathbb{E}_{\polA,\polB}\big[\innerP{\theta_B, \, \phi(\aB, \bB) }\big], \label{eq:pi_b_star} 
 \end{align}



Given $\pi_B^*(\cdot)$ in Eq. \eqref{eq:pi_b_star}, let us define,

\vspace{-0.4cm}

\begin{align}
   & \underline{\mathcal{H}}(\theta_A^*,t) \equiv \max_{\pi_A\in\Pi_A}\min_{\theta_A}\, \mathbb{E}_{\polA,\polB^*(\polA)}\big[\innerP{\theta_A, \, \phi(\aB, \bB) }\big], \label{eq:min_H} \\
     & \overline{\mathcal{H}}(\theta_A^*,t) \equiv \max_{\pi_A\in\Pi_A}\max_{\theta_A}\, \mathbb{E}_{\polA,\polB^*(\polA)}\big[\innerP{\theta_A, \, \phi(\aB, \bB) }\big]. \label{eq:max_H}
\end{align}

\vspace{-0.4cm}

such that, $\quad  \theta_A \in \Ball(\theta^*_A,\uncerBall{t})$, and $\theta_B \in \Ball(\theta^*_B,\uncerBall{t})$. 

The above expressions represent the pessimistic and optimistic leader's estimates of her average reward. We can see from the structure of Eq. \eqref{eq:pi_a_star} to Eq. \eqref{eq:max_H}, the resemblance to a bi-level optimization problem. The solutions to such problems are often computationally demanding and/or complex to formulate \citep{beck:2023_bilevel_survey, sinha:2017_bilevel_review}. (We further provide a discussion of such optimization methods in Appendices \ref{sec:stack_games_opt_app} and \ref{sec:bilevel_extra_discuss}.)

\section{Geodesic Online Learning}

To enable efficient multi-agent online learning on the Stackelberg manifold, $\stEmb$, we enforce $\stEmb$ to be a \textit{locally convex manifold}. Our definition of a locally convex manifold is a one where the geodesic between any two points on the manifold is contained within or forms a geodesically convex set (see Appendix Def. \ref{def:geodesic_convex_subsets}). More specifically, we employ the use of a \textit{spherical manifold}, which is locally convex by nature. Under geodesic convexity, the follower's optimal best response strategy is uniquely deterministic (see Appendix Lemma \ref{lem:pure_strategy_convex_manifold}). 






\subsection{Bilevel Optimization}

Provided that we can transform data from the joint action space (or ambient space) onto a spherical manifold, we can leverage the properties of the $D$-sphere to determine the best response for the follower and minimize the corresponding Stackelberg regret. Consider the reward function structure outlined in Section \ref{sec:reward_function}. In general, the reward of each agent has the form $\mu = \innerP{\theta, \phi(\aB,\bB)}$, where $\theta$ represents a $D$-dimensional vector, and $\phi(\aB,\bB) \in \stEmb$ encodes the joint action representation. In the Stackelberg game, $\theta_A$ and $\theta_B$, referred to as \textit{objective vectors}, characterize each agent's reward parameters. The leader's first-mover advantage defines a restricted subspace on $\stEmb$, within which the follower must then optimize his reward. Specifically, he must find the element in $\stEmb$ that maximizes the inner product subject to the leader's constraints given the follower's best response, $\bR{\aB}$, defined in Eq. \eqref{eq:br-def}, thereby attaining a Stackelberg equilibrium.



We denote the \textit{geodesic distance} between two vectors, denoted as $\geoDis(\theta_A, \theta_B)$, for a unit-spherical manifold, as,
\vspace{-0.1cm}

\begin{align}
    \geoDis(\theta_A, \theta_B) = \arccos\Big( \frac{\innerP{\theta_A, \theta_B}}{\|\theta_A\| \|\theta_B\|} \Big). \label{eq:div_angle_cos_defn}
\end{align}

In a $D$-dimensional sphere, a fully cooperative game exhibits co-directional objective vectors (i.e. without divergence), implying that the inner-product-maximizing solution in $\stEmb$ must be collinear with $\theta_A$, mutatis mutandis for $\theta_B$. Moving forward, we use the convention $\xi_{\theta_A}$ and $\xi_{\theta_B}$ to denote the projection of the $\theta_A$ and $\theta_B$ onto $\stEmb$.





\begin{lemma} \label{lem:geodesic_and_closeness_phi}
    \textbf{Geodesic Distance and Closeness:} Let $\Phi \subset \mathbb{R}^D$ be a manifold serving as a boundary of a convex set in $\mathbb{R}^D$. Given $\theta$, let $\xi_\theta \in \Phi$ be the point on the manifold that maximizes $\langle \theta, \xi_\theta \rangle$, and is orthogonal to $\stEmb$ at the point of intersection. For any two points on the manifold $\xi_{\theta_A}, \xi_{\theta_B} \in \Phi$, if the geodesic distance between $\xi_\theta$ and $\xi_{\theta_A}$ is greater than the geodesic distance between $\xi_\theta$ and $\xi_{\theta_B}$, $\geoDis(\xi_\theta, \xi_{\theta_A}) > \geoDis(\xi_\theta, \xi_{\theta_B})$, then the dot product satisfies $\langle \theta, \xi_{\theta_A} \rangle < \langle \theta, \xi_{\theta_B} \rangle$. (Proof in Appendix \ref{prf:geodesic_and_closeness_phi}.)
\end{lemma}

\textit{Sketch of Proof:} We establish that on a spherical manifold, the dot product with $\xi_{\theta}$ decreases as geodesic distance from any point on the manifold to the $\xi_\theta$ increases, linking vector alignment to geodesic distance, and establish the inverse relation between dot product $\langle \xi_\theta, \xi \rangle$ and the cosine relationship. This allows us to frame our optimization problem as a geodesic distance minimization problem.

\subsection{Regret Definitions}

\begin{definition} \label{def:stack_regret}
    \textbf{Stackelberg Regret:} We define Stackelberg regret, denoted as $R_A^T$ for the leader, measuring the difference in cumulative rewards between a best responding follower and an optimal leader in a perfect information setting, against best responding follower and leader exhibiting \textit{bounded rationality}, the leader acts rationally given the estimates of her expected reward function. Specifically,
\end{definition}

\begin{align}\label{eq:leader_regret_def_generic}
    R_A^T = \sum_{t = 1}^{T} \mathbb{E}\Big[\underset{\aB \in \mathcal{A} }{\mathrm{max}} \ \utlA(\aB, \mathfrak{B}(\aB)) -  \, \utlA(\aB^t, \bR{\aB^t}) \Big] \nonumber
\end{align}
The leader selects $\aB^t$ from policy $\pi_A$ based on their estimates $\hat{\theta}_A$ and $\hat{\theta}_B$, following Eq. \eqref{eq:pi_a_star} and Eq. \eqref{eq:pi_b_star}. Committing to $\pi_A$, the leader aims to maximize their reward while accounting for uncertainty in the follower's response, which is estimated rationally within a confidence interval. Our algorithm ensures a no-regret learning process by minimizing \textit{Stackelberg regret} bounded by the proposed learning algorithm (Alg. \ref{alg:gisa}). To evaluate this, we derive a closed-form expression for the gap between the leader's expected reward under the optimal policy and any algorithm.



\begin{definition}
    \textbf{Simple Regret:} We define the simple regret, where with probability $1\!\!-\!\delta$ at time $t$,
    \begin{align}
        \text{reg}(t) &\equiv \innerP{\theta_A^*, \phi(\aB^*, \bR{\aB^*})} - \innerP{\theta_A^*, \phi(\aB^t, \bR{\aB^t})} \\
        &\leq \bar{\mathcal{H}}(\theta_A^*,t) - \underline{\mathcal{H}}(\theta_A^*,t) \label{eq:simple_regret_defn}
    \end{align}
    This assumes that the leader is acting under the bounded rationality assumption. 
\end{definition}





\begin{figure*}
\minipage{0.5\textwidth}
\centering
    
    \tikzset{every picture/.style={line width=0.75pt}} 
    
    \begin{tikzpicture}[x=0.75pt,y=0.75pt,yscale=-0.58,xscale=0.58]
    
    \draw   (719.4,172) .. controls (719.4,101.86) and (776.26,45) .. (846.4,45) .. controls (916.54,45) and (973.4,101.86) .. (973.4,172) .. controls (973.4,242.14) and (916.54,299) .. (846.4,299) .. controls (776.26,299) and (719.4,242.14) .. (719.4,172) -- cycle ;
    \shade[ball color = gray!40, opacity = 0.3] (846.4,172) circle (3.37cm);
    \draw  [color={rgb, 255:red, 208; green, 2; blue, 27 }  ,draw opacity=1 ] [fill={rgb, 255:red, 208; green, 2; blue, 27 }  ,fill opacity=0.01 ](751.81,172) .. controls (751.81,101.86) and (794.16,45) .. (846.4,45) .. controls (898.64,45) and (940.99,101.86) .. (940.99,172) .. controls (940.99,242.14) and (898.64,299) .. (846.4,299) .. controls (794.16,299) and (751.81,242.14) .. (751.81,172) -- cycle ;
    \draw  [color={rgb, 255:red, 74; green, 144; blue, 226 }  ,draw opacity=1 ][fill={rgb, 255:red, 74; green, 144; blue, 226 }  ,fill opacity=0.07 ][dash pattern={on 4.5pt off 4.5pt}] (722.5,199.5) .. controls (722.5,187.83) and (778.02,178.37) .. (846.5,178.37) .. controls (914.98,178.37) and (970.5,187.83) .. (970.5,199.5) .. controls (970.5,211.17) and (914.98,220.62) .. (846.5,220.62) .. controls (778.02,220.62) and (722.5,211.17) .. (722.5,199.5) -- cycle ;

    \draw  [color={rgb, 255:red, 74; green, 144; blue, 226 }, draw opacity=1 ] (722.5,199.5) .. controls (722.5,211.17) and (778.02,220.62) .. (846.5,220.62) .. controls (914.98,220.62) and (970.5,211.17) .. (970.5,199.5);

    \draw  [color={rgb, 255:red, 74; green, 144; blue, 226 }, draw opacity=1, dash pattern={on 4.5pt off 4.5pt} ] (722.5,199.5) .. controls (722.5,187.83) and (778.02,178.37) .. (846.5,178.37) .. controls (914.98,178.37) and (970.5,187.83) .. (970.5,199.5);
        
    \draw  [color={rgb, 255:red, 74; green, 144; blue, 226 }  ,draw opacity=1 ][fill={rgb, 255:red, 74; green, 144; blue, 226 }  ,fill opacity=0.07 ][dash pattern={on 4.5pt off 4.5pt}] (726.5,129.87) .. controls (726.5,118.21) and (780,108.75) .. (846,108.75) .. controls (912,108.75) and (965.5,118.21) .. (965.5,129.87) .. controls (965.5,141.54) and (912,151) .. (846,151) .. controls (780,151) and (726.5,141.54) .. (726.5,129.87) -- cycle ;

    \draw  [color={rgb, 255:red, 74; green, 144; blue, 226 }, draw opacity=1 ] (726.5,129.87) .. controls (726.5,141.54) and (780,151) .. (846,151) .. controls (912,151) and (965.5,141.54) .. (965.5,129.87);
    
    \draw  [color={rgb, 255:red, 74; green, 144; blue, 226 }, draw opacity=1, dash pattern={on 4.5pt off 4.5pt} ] (726.5,129.87) .. controls (726.5,118.21) and (780,108.75) .. (846,108.75) .. controls (912,108.75) and (965.5,118.21) .. (965.5,129.87);
    
    \draw [color={rgb, 255:red, 165; green, 0; blue, 20 }  ,draw opacity=1 ]   (928.83,236.83) .. controls (890.75,246.68) and (806.74,246.18) .. (767.27,236.61) ;
    \draw [shift={(765.5,236.17)}, rotate = 14.5] [color={rgb, 255:red, 165; green, 0; blue, 20 }  ,draw opacity=1 ][line width=0.75]    (10.93,-3.29) .. controls (6.95,-1.4) and (3.31,-0.3) .. (0,0) .. controls (3.31,0.3) and (6.95,1.4) .. (10.93,3.29)   ;
    \draw [color={rgb, 255:red, 126; green, 211; blue, 33 }  ,draw opacity=1 ][fill={rgb, 255:red, 65; green, 117; blue, 5 }  ,fill opacity=1 ]   (846.4,172) -- (786,126.45) ;
    \draw [shift={(784.4,125.25)}, rotate = 37.02] [color={rgb, 255:red, 126; green, 211; blue, 33 }  ,draw opacity=1 ][line width=0.75]    (10.93,-3.29) .. controls (6.95,-1.4) and (3.31,-0.3) .. (0,0) .. controls (3.31,0.3) and (6.95,1.4) .. (10.93,3.29)   ;
    \draw [color={rgb, 255:red, 245; green, 166; blue, 35 }  ,draw opacity=1 ]   (846.4,172) -- (909.46,187.76) ;
    \draw [shift={(911.4,188.25)}, rotate = 194.04] [color={rgb, 255:red, 245; green, 166; blue, 35 }  ,draw opacity=1 ][line width=0.75]    (10.93,-3.29) .. controls (6.95,-1.4) and (3.31,-0.3) .. (0,0) .. controls (3.31,0.3) and (6.95,1.4) .. (10.93,3.29)   ;
    \draw    (958.67,79.5) .. controls (925.21,78.21) and (911.01,130.16) .. (905.56,147.28) ;
    \draw [shift={(905,149)}, rotate = 288.43] [color={rgb, 255:red, 0; green, 0; blue, 0 }  ][line width=0.75]    (10.93,-3.29) .. controls (6.95,-1.4) and (3.31,-0.3) .. (0,0) .. controls (3.31,0.3) and (6.95,1.4) .. (10.93,3.29)   ;
    \draw  [draw opacity=0] (838.37,165.05) .. controls (842.13,161.98) and (847.53,161.11) .. (852.36,163.25) .. controls (857.52,165.54) and (860.55,170.6) .. (860.41,175.79) -- (847.03,175.22) -- cycle ; \draw   (838.37,165.05) .. controls (842.13,161.98) and (847.53,161.11) .. (852.36,163.25) .. controls (857.52,165.54) and (860.55,170.6) .. (860.41,175.79) ;  
    \draw [color={rgb, 255:red, 29; green, 95; blue, 173 }  ,draw opacity=1 ]   (738.33,208.83) .. controls (733.81,190.08) and (734.61,159.88) .. (738.01,141.2) ;
    \draw [shift={(738.33,139.5)}, rotate = 101.21] [color={rgb, 255:red, 29; green, 95; blue, 173 }  ,draw opacity=1 ][line width=0.75]    (10.93,-3.29) .. controls (6.95,-1.4) and (3.31,-0.3) .. (0,0) .. controls (3.31,0.3) and (6.95,1.4) .. (10.93,3.29)   ;
    
    \draw (966.33,58.5) node [anchor=north west][inner sep=0.75pt]   [align=left] {$\IsoPL{\bB}$};
    \draw (863,156.17) node [anchor=north west][inner sep=0.75pt]   [align=left] {$\dvAng$};
    \draw (779.67,99.17) node [anchor=north west][inner sep=0.75pt]   [align=left] {$\theta_B'$};
    \draw (913.67,180.5) node [anchor=north west][inner sep=0.75pt]   [align=left] {$\theta_A'$};
    \draw (680.67,158.17) node [anchor=north west][inner sep=0.75pt]   [align=left] {};
    \draw (867.67,249.17) node [anchor=north west][inner sep=0.75pt]   [align=left] {$\IsoPL{\aB}$};
    
    \fill[fill=black] (846.4,172) circle (2pt);
    \end{tikzpicture}
    \caption{Isoplanar subspaces for players A and B.} \label{fig:isoplane}
\endminipage\hfill
\minipage{0.5\textwidth}
\centering

    \tikzset{every picture/.style={line width=0.75pt}} 
    
    \begin{tikzpicture}[x=0.75pt,y=0.75pt,yscale=-0.58,xscale=0.58]
    
    \draw [color={rgb, 255:red, 126; green, 211; blue, 33 }  ,draw opacity=1 ][fill={rgb, 255:red, 65; green, 117; blue, 5 }  ,fill opacity=1 ]   (494.4,164) -- (434,118.45) ;
    \draw [shift={(432.4,117.25)}, rotate = 37.02] [color={rgb, 255:red, 126; green, 211; blue, 33 }  ,draw opacity=1 ][line width=0.75]    (10.93,-3.29) .. controls (6.95,-1.4) and (3.31,-0.3) .. (0,0) .. controls (3.31,0.3) and (6.95,1.4) .. (10.93,3.29)   ;
    \draw [color={rgb, 255:red, 245; green, 166; blue, 35 }  ,draw opacity=1 ]   (494.4,164) -- (557.46,179.76) ;
    \draw [shift={(559.4,180.25)}, rotate = 194.04] [color={rgb, 255:red, 245; green, 166; blue, 35 }  ,draw opacity=1 ][line width=0.75]    (10.93,-3.29) .. controls (6.95,-1.4) and (3.31,-0.3) .. (0,0) .. controls (3.31,0.3) and (6.95,1.4) .. (10.93,3.29)   ;
    \draw  [color={rgb, 255:red, 245; green, 166; blue, 35 }  ,draw opacity=1 ] [fill={rgb, 255:red, 245; green, 166; blue, 35 }  ,fill opacity=0.3 ] (534.06,177) .. controls (537.96,165.49) and (550.94,158.18) .. (563.06,160.67) .. controls (575.18,163.15) and (581.84,174.49) .. (577.94,186) .. controls (574.04,197.51) and (561.06,204.82) .. (548.94,202.33) .. controls (536.82,199.85) and (530.16,188.51) .. (534.06,177) -- cycle ;
    \draw  [color={rgb, 255:red, 126; green, 211; blue, 33 }  ,draw opacity=1 ] [fill={rgb, 255:red, 184; green, 233; blue, 134 }  ,fill opacity=0.47 ] (413.44,110.89) .. controls (419.21,99.47) and (432.38,93.05) .. (442.86,96.56) .. controls (453.33,100.07) and (457.14,112.18) .. (451.36,123.61) .. controls (445.59,135.03) and (432.42,141.45) .. (421.94,137.94) .. controls (411.47,134.43) and (407.66,122.32) .. (413.44,110.89) -- cycle ;
    \draw [color={rgb, 255:red, 208; green, 2; blue, 27 }  ,draw opacity=1 ]   (494.4,37) .. controls (540.8,50.6) and (550.4,257.4) .. (494.4,291) ;
    \draw [color={rgb, 255:red, 208; green, 2; blue, 27 }  ,draw opacity=1 ]  (494.4,37) .. controls (575.2,34.6) and (636,245) .. (494.4,291) ;
    \draw [color={rgb, 255:red, 74; green, 144; blue, 226 }  ,draw opacity=1 ]  (374.5,124.5) .. controls (461.1,155.9) and (543.1,158) .. (612.7,120.4) ;
    \draw [color={rgb, 255:red, 74; green, 144; blue, 226 }  ,draw opacity=1 ]  (397,83.6) .. controls (433,102) and (548.2,106.7) .. (591.4,82) ;
    \draw  [draw opacity=0][fill={rgb, 255:red, 189; green, 16; blue, 224 }  ,fill opacity=0.09 ] (527.67,97.83) -- (546.33,95.17) -- (562.33,92.5) -- (571,115.83) -- (576.33,135.83) -- (555.67,141.83) -- (533,146.5) -- (531,123.83) -- cycle ;
    \draw   (367.4,164) .. controls (367.4,93.86) and (424.26,37) .. (494.4,37) .. controls (564.54,37) and (621.4,93.86) .. (621.4,164) .. controls (621.4,234.14) and (564.54,291) .. (494.4,291) .. controls (424.26,291) and (367.4,234.14) .. (367.4,164) -- cycle ;
    \shade[ball color = gray!40, opacity = 0.3] (494.4,164) circle (3.37cm);
    \draw [color={rgb, 255:red, 165; green, 0; blue, 20 }  ,draw opacity=1 ] [dash pattern={on 0.84pt off 2.51pt}]  (571,115.83) .. controls (562.84,120.79) and (548.69,123.76) .. (532.97,123.84) ;
    \draw [shift={(531,123.83)}, rotate = 0.58] [color={rgb, 255:red, 165; green, 0; blue, 20 }  ,draw opacity=1 ][line width=0.75]    (10.93,-3.29) .. controls (6.95,-1.4) and (3.31,-0.3) .. (0,0) .. controls (3.31,0.3) and (6.95,1.4) .. (10.93,3.29)   ;
    \draw [color={rgb, 255:red, 165; green, 0; blue, 20 }  ,draw opacity=1 ] [dash pattern={on 0.84pt off 2.51pt}]  (531,123.83) .. controls (547.71,123.99) and (557.13,121.72) .. (569.27,116.58) ;
    \draw [shift={(571,115.83)}, rotate = 156.45] [color={rgb, 255:red, 165; green, 0; blue, 20 }  ,draw opacity=1 ][line width=0.75]    (10.93,-3.29) .. controls (6.95,-1.4) and (3.31,-0.3) .. (0,0) .. controls (3.31,0.3) and (6.95,1.4) .. (10.93,3.29)   ;
    
    \draw (427.67,100.17) node [anchor=north west][inner sep=0.75pt]   [align=left] {$\theta_B'$};
    \draw (546.67,183.5) node [anchor=north west][inner sep=0.75pt]   [align=left] {$\theta_A'$};
    \draw (536.33,85.17) node [anchor=north west][inner sep=0.75pt]   [align=left] {$\IsoPL{\aB}$};

    \fill[fill=black] (494.4,164) circle (2.5pt);
    \end{tikzpicture}
    \caption{Geodesic confidence balls for players A and B.} \label{fig:conf-ball}
\endminipage\hfill
\vspace{-5pt} 
\caption*{\textbf{Diagram Description:} A visualization of the isoplanes $\IsoPL{\aB}$ and $\IsoPL{\bB}$ on a 2-sphere embedded in three dimensions is shown in Fig. \ref{fig:isoplane}. The isoplanes are depicted relative to the normalized objective vectors $\theta_A'$ and $\theta_B'$, which lie on the manifold surface, separated by a divergence angle $\alpha_{Div}$. Figure \ref{fig:conf-ball} illustrates the geodesic confidence balls, positioned on the surface of the spherical manifold. In three dimensions, it becomes evident that $\IsoPL{\aB}$ and $\IsoPL{\bB}$ are orthogonal at any point of intersection. This intersection, denoted by $\IsoPL{\bB_{\aB}}$, is where the joint action emerges, represented by a purple geodesic square indicating the uncertainty region.} 
\vspace{-10pt} 
\end{figure*}

\vspace{-0.3cm}

\subsection{Quantifying Uncertainty on the Manifold} \label{sec:quantify_uncertainty_stackmans}

We now revisit the parameter uncertainty constraints introduced in Sec. \ref{sec:optim_of_stackelberg_games}, which dictate the uncertainty of a given learning algorithm, characterized by an uncertainty radius $\uncerBall{t}$. Given the feature map $\phi(\cdot)$, which adheres to the linear reward assumptions, particularly with respect to the covariance matrix of the regression (as outlined in Sec. \ref{sec:reward_function}), the learning leader can apply any bandit learning algorithm that imposes a high-probability bound on the parameter estimate. This constraint is formalized in Eq. \eqref{eq:param_uncertainty} by the uncertainty region $\uncerBall{t}$. Let us define $\IsoPL{\aB}$ and $\IsoPL{\bB}$ as two subspaces, which we will use to analyze the leader's actions under these uncertainty constraints.
\begin{align}
   \!\!\! \IsoPL{\aB}\! :=\! \{ \phi(\aB, \bB') | \bB' \in \setB \}, \, \IsoPL{\bB}\! :=\! \{ \phi(\aB', \bB) | \aB' \in \setA \}
\end{align}
where $\IsoPL{\aB}$ and $\IsoPL{\bB}$ are the sub-spaces formed when  fixing one of the agents' action, and varying the other action freely. 

\begin{lemma} \label{lem:intersect_submanifold_AB}
    \textbf{Intersection of $\IsoPL{\aB}$ and $\IsoPL{\bB}$:} Given a bipartite spherical map from Def. \ref{def:bipartite_sphere_map}, with $\aB$ parameterizing the azimuthal (latitudinal) coordinates, the cardinality of the intersect between $\IsoPL{\aB}$ and $\IsoPL{\bB}$ will be non-empty. That is, $|\IsoPL{\aB} \cap \IsoPL{\bB}| > 0$. (Proof is in Appendix \ref{prf:intersect_submanifold_AB}.)
\end{lemma}

The purpose of Lemma \ref{lem:intersect_submanifold_AB} is to highlight that, given the bipartite map, subspaces are guaranteed to intersect on the manifold. This is easy to visualize on a spherical manifold in the 2-sphere setting (e.g., longitudinal and latitudinal lines) but becomes challenging to perceive in higher dimensions. We rigorously argue that, just as in the 2-sphere case, the same principle holds in a D-sphere setting. The derivation of Lemma \ref{lem:intersect_submanifold_AB} first comes by isolating the subspaces in terms of angular coordinates. Next, due to the \textit{Poincaré-Hopf theorem} \citep{poincare:1885, hopf:1927}, the compactness of the smooth Riemmanian manifold imposes strong geometric constraints such that the two subspaces cannot avoid each other.


\begin{lemma} \label{lem:leader_pure_strategy_spherical}
    \textbf{Pure Strategy of the Leader:} Given a spherical manifold, $\stEmb$, and isoplanar subspace, $\IsoPL{\aB}$ and $\IsoPL{\bB}$ for the longitudinal and lattitudinal subspaces respectively, the optimal strategy of the leader is that of a pure strategy, that is, $\polA^*(\aB) \in \{0, 1\}$. (Proof is provided in Appendix \ref{prf:leader_pure_strategy_spherical}.)
\end{lemma}

Lemma \ref{lem:leader_pure_strategy_spherical} argues that the intersection between $\IsoPL{\aB}$ and $\IsoPL{\bB}$ contains at most one element due to their orthogonality (see Appendix Lemma \ref{lem:orthogonality_submanifold_AB}). Consequently, no other actions on the manifold can further maximize the leader's reward. Intuitively, the positive curvature of the manifold ensures that once two non-degenerate isoplanes intersect, the intersection is a unique point that maximizes the dot product between the action and the objective vector. 



\begin{figure*}[!htb]
\centering
\subfloat[$\mathbb{R}^1$ Stackelberg Regret.]{%
  \includegraphics[width=41mm]{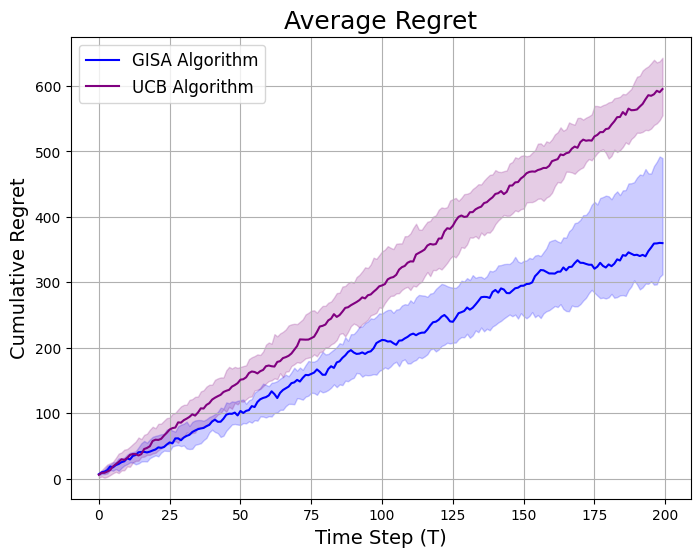}%
  \label{fig:market3}%
}\hspace{10pt}
\subfloat[NPG Regret.]{%
  \includegraphics[width=41mm]{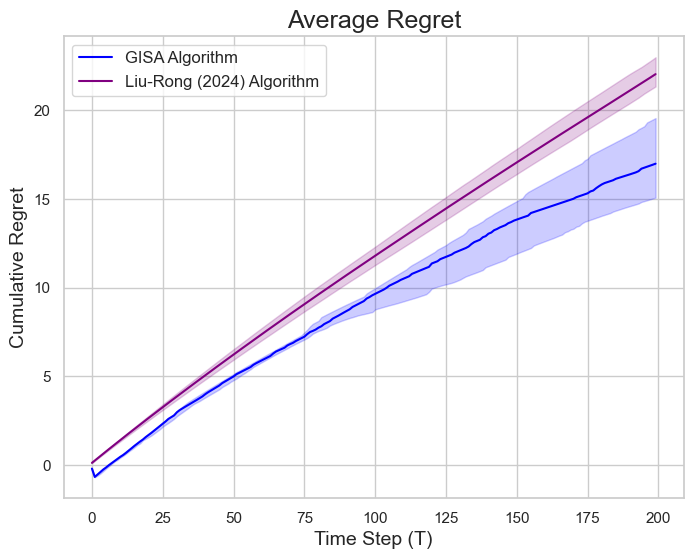}%
  \label{fig:market2}%
}\hspace{10pt}
\subfloat[SSG Regret.]{%
  \includegraphics[width=41mm]{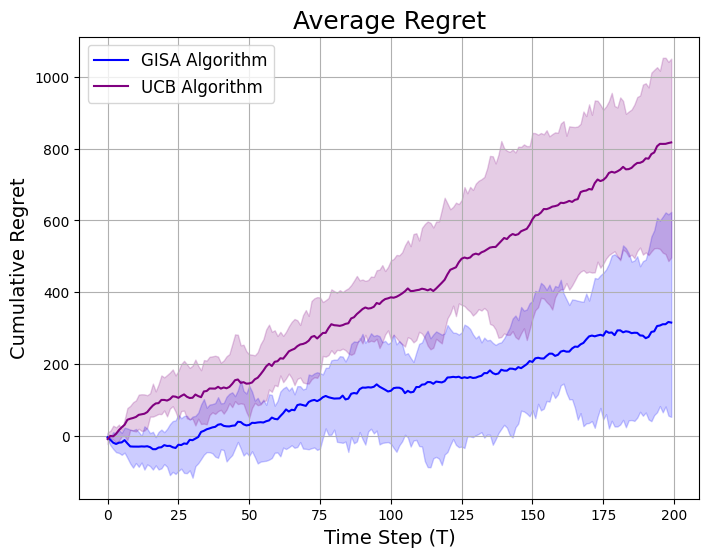}%
  \label{fig:market4}%
}
\caption{Average cumulative regret performance across three Stackelberg games. Uncertainty region denote upper and lower quartile of experimental results. (Parameters of the simulations are outlined in Appendices \ref{sec:r1_game_details} - \ref{sec:multi-dim-ssg}.) } \label{fig:avg_cumm_regret}
\vspace{-15pt}
\end{figure*}

\textbf{Geodesic Isoplanar Subspace Alignment (GISA):} We present an end-to-end procedure (in Algorithm \ref{alg:gisa}) for constructing the Stackelberg manifold, $\stEmb$, and subsequently enabling online learning within it. The process begins with an offline, data-independent, training phase (Step \ref{line:stack_embed} of Algorithm \ref{alg:gisa}) to construct $\stEmb$. This is followed by iterative online learning performed directly on the manifold. As a result, optimal strategy learning reduces to geodesic distance minimization, providing a significantly more computationally efficient alternative to traditional game-theoretic approaches, such as bi-level optimization.

The general methodology in which we can compute the optimal leader strategy is that the leader can anticipate the follower strategy based on knowledge of follower's reward parameters $\xi_{\theta_B}$ and the isoplane $\IsoPL{\aB}$. We denote this homeomorphism as $f_1(\IsoPL{\aB}, \xi_{\theta_B}): \IsoPL{\aB} \mapsto \IsoPL{\bB^*_{\aB}}$. Thereafter, we compute the geodesic distance minimizing distance from $\IsoPL{\bB^*_{\aB}}$ to $\xi_{\theta_A}$ via injective map $f_2(\IsoPL{\bB^*_{\aB}}, \xi_{\theta_A}): \IsoPL{\bB^*_{\aB}} \mapsto \mathbb{R}$. For illustration, referencing Fig. \ref{fig:conf-ball}, the $\IsoPL{\aB}$ represents the subspace induced by the leader, visualized as a red trace, from where the follower selects his best response within. The follower will attempt to act within the subspace, such that it minimizes his geodesic distance to $\theta'_B$. Thus, the leader's objective is to find $\aB \in \setA$ such that it minimizes the composition of $f_1 \circ f_2$, giving us the geodesic distance. This composition is defined as,

\vspace{-0.4cm}

\begin{align}
    \IsoPL{\aB} \underset{f_1(\cdot, \xi_{\theta_A})}{\longmapsto} \IsoPL{\bB^*_{\aB}} \underset{f_2(\cdot, \xi_{\theta_B})}{\longmapsto} \geoDis(\aB, \bB^*_{\aB}) \in \mathbb{R}, \quad   \label{eq:gisa_map_and_norm}
\end{align}
where $\theta' = \frac{\theta}{\norm{\theta}}$ for A and B. 



\begin{theorem} \label{thm:sphere_man_iso_reg}
    \textbf{Isoplane Stackelberg Regret:} For D-dimensional spherical manifolds embedded in $\mathbb{R}^D$ space, where $\phi(\aB, \cdot)$ generates an isoplanes $\IsoPL{\aB}$, and the linear relationship to the reward function in Eq. \eqref{eq:A_reward_inner_prod} and Eq. \eqref{eq:B_reward_inner_prod}, 
    the simple regret, defined in Eq. \eqref{eq:simple_regret_defn}, of any learning algorithm with uncertainty parameter uncertainty  $\uncerBall{t}$, refer to in Eq. \eqref{eq:param_uncertainty}, is bounded by $\mathcal{O}(\arccos(1-\uncerBall{t}^2/2))$. (Proof in Appendix \ref{prf:sphere_man_iso_reg}.)
\end{theorem}


\textit{Proof Sketch:} The proof of Theorem \ref{thm:sphere_man_iso_reg} focuses on analyzing the geodesic distances on $\stEmb$ due to uncertainty. First, we argue that any norm-like confidence ball in Cartesian coordinates, $\Ball(\cdot)$, can be transformed into a confidence bound into a geodesic distance-based confidence ball, $\Ball_{\geoDis}(\cdot)$, in spherical coordinates (discussed in Lemma \ref{lem:geodesic_uncertainty_ball} of the Appendix.) Due to orthogonality between $\IsoPL{\aB}$ and $\IsoPL{\bB}$, we argue that that the geodesic distance either remains the same or decreases when we projected from any $\Ball(\cdot)'$ from $\IsoPL{\aB}$ to $\IsoPL{\bB}$ (discussed in Lemma \ref{lem:distance_preserving_ortho_proj} of the Appendix.) This naturally extends to a bound on the maximum diameter of the projected confidence ball on $\IsoPL{\bB}$. This constitutes the best and worst possible outcomes due to misspecification in accordance with the formulas in Eq. \eqref{eq:min_H} and Eq. \eqref{eq:max_H}. Consequently, one can see that our methodology enables the leader to simplify the follower’s best response region on $\stEmb$, allowing derivation of high-probability worst-case outcomes and theoretical guarantees for simple regret under quantifiable uncertainty.

\begin{algorithm}[h!]
\caption{Geodesic Isoplanar Subspace Alignment (GISA) Algorithm}\label{alg:gisa}
\begin{algorithmic}[1]
\STATE \textbf{Input:} Time horizon $T$, and confidence ball $\uncerBall{\cdot}$.
\STATE \textbf{Output:} Estimated optimal leader action $\hat{\aB}$.
\STATE Initialize $\hat{\theta}_A$ and $\hat{\theta}_B$ uniformly at random.
\STATE Initialize reward and action histories, $\mathcal{U}$ and $\mathcal{H}$ as $\emptyset$.
\STATE Construct a Stackelberg embedding $\stEmb$ and feature map $\phi$ per specifications in Sec. \ref{sec:stack_embedding}. \label{line:stack_embed}
\FOR {$t \in 1 ... T$}
    \IF{$\geoDis(\hat{\theta}_A, \hat{\theta}_B) < 2 \uncerBall{t}$} 
        \STATE \texttt{Phase 1}: Select uniformly an action on the boundary of $A$'s geodesic confidence ball.
        \STATE $\aB \sim \texttt{Uniform}[\partial \, \Ball_{\geoDis}(\hat{\theta}_A,\uncerBall{t})] \,$ 
    \ELSE
        \STATE \texttt{Phase 2}: Select $\aB$ that minimizes the geodesic distance to $\hat{\theta}_B$ from $\Ball_{\geoDis}(\hat{\theta}_A, \uncerBall{t})$.
        \STATE $\aB \gets \underset{\aB \in \Ball_{\geoDis}(\hat{\theta}_A, \uncerBall{t})}{\argmin} \,
 \geoDis(\aB, \hat{\theta}_B)$
    \ENDIF
    \STATE $\bB \gets \underset{\bB \in \IsoPL{\aB}}{\argmin}\geoDis(\bB, \hat{\theta}_B)$
    \STATE $\hat{\aB}^t, \hat{\bB}^t \gets \phi^{-1}(\aB, \bB)$ Perform an inverse map.
    \STATE \textbf{yield} $\hat{\aB}^t, \hat{\bB}^t$, and obtain empirical reward $\mu_A^t, \mu_B^t$.
    \STATE $\mathcal{H} \gets \mathcal{H} \cup (\hat{\aB}^t, \hat{\bB}^t), \quad \mathcal{U} \gets \mathcal{U} \cup (\mu_A^t, \mu_B^t)$.
    \STATE Re-estimate $\hat{\theta}_A$ and $\hat{\theta}_B$ from $\mathcal{H}$ and $\mathcal{U}$, via Eq. \eqref{eq:theta_estimation_xtx}.
\ENDFOR
\RETURN $\hat{\aB}^t$
\end{algorithmic}
\end{algorithm}

\vspace{-0.3cm}

\section{Empirical Experiments} \label{sec:empirical_experiments_main}


We present three practical instances of Stackelberg games and benchmark the GISA algorithm (Alg. \ref{alg:gisa}) against a dual-UCB algorithm, where both agents use UCB-based no-regret learning \citep{blum:2007-external-internal-reg}. The $\stEmb$ transform abstracts away the need for exact reward structure knowledge, enabling well-behaved representations for online learning in new problem settings. GISA starts by sampling the action space until uncertainty intervals become disjoint, then enters a continuous learning phase, subject to the guarantees in Thm. \ref{thm:sphere_man_iso_reg}. This method generalizes beyond problem-specific solutions and addresses Stackelberg game learning methods lacking closed-form solutions and/or computational feasibility.


\textbf{$\mathbb{R}^1$ Stackelberg Game:} In this game, the leader chooses an action, anticipating the follower's best response. Both players have one-dimensional action spaces. Nonlinear rewards and penalties complicate the equilibrium, requiring numerical methods to find optimal strategies. A real-world example is an energy grid management, where a utility company (leader) sets prices or output levels, considering consumers' (followers) usage and nonlinear factors like demand fluctuations or storage constraints. (Details are presented in Appendix \ref{sec:multi-dim-ssg}.)


\textbf{The Newsvendor Pricing Game (NPG):} Modelled from \citep{cesa:2022-supp-chain-games, liu:2024_stacknews_adt}, the NPG models a supply chain consisting of a supplier (leader) and retailer (follower) in a repeated Stackelberg game. The leader's action space is $\mathbf{a} \in \mathbb{R}^1$, and the follower's is $\mathbf{b} \in \mathbb{R}^2$. The supplier dynamically prices the product to maximize reward, while the retailer optimizes pricing and order quantity based on stochastic demand, following classical Newsvendor theory \citep{arrow:1951newsboy, petruzzi:1999newsv}. The asymmetric reward function and stochastic demand complicates online learning significantly. (Details are presented in Appendix \ref{sec:multi-dim-ssg}.)


\textbf{Stackelberg Security Game (SSG) in $\mathbb{R}^5$:} Motivated by \citep{balcan:2015-stack-learn-sec, zhang:2021_bayesian_ssg}, this SSG models a defender (leader) allocating limited resources across multiple targets, anticipating an attacker’s (follower) strategy. Both players select actions from $\mathbb{R}^5$, with rewards driven by the difference between the actions (i.e. $\aB - \bB$) and quadratic penalties for overextension. Resource constraints are imposed via weighted $L_1$-norms, limiting feasible actions. The Stackelberg equilibrium is defined by the leader’s optimal resource allocation, accounting for the adversary’s best response. Nonlinear  constraints increase the problem complexity, conventionally requiring numerical solutions. (Details are presented in Appendix \ref{sec:multi-dim-ssg}.)

\vspace{-0.1cm}
\section{Conclusion}

This work establishes a foundational connection between Stackelberg games and normalizing neural flows, marking a significant advancement in the study of equilibrium learning and manifold learning. By utilizing normalizing flows to map joint action spaces onto Riemannian manifolds, particularly spherical ones, we offer a novel, theoretically grounded framework with formal guarantees on simple regret. This approach represents the first application of normalizing flows in game-theoretic settings, specifically Stackelberg games, thereby opening new avenues for learning on spherical manifolds. Current limitations include the restriction to spherical manifolds, however, the key principles of geodesic distance minimization apply, and warrants future investigation. Our empirical results, grounded in realistic simulation scenarios, highlight promising improvements in both computational efficiency and regret minimization, underscoring the broad potential of this methodology across multiple domains in economics and engineering. Despite potential challenges related to numerical accuracy for the neural flow network, this integration of manifold learning into game theory exhibits strong implications for online learning, positioning neural flows as a promising tool for both machine learning and strategic decision-making.


\clearpage
\appendix

\section*{Acknowledgements}

The affiliated authors thank the departmental research funding from the Technical University of Munich School of Computation and Information Technology for their generous support. We also acknowledge Stefanos Leonardos, Vinzenz Thoma, and the paper reviewers for their the insightful technical feedback helping to refine our work.

\section*{Ethical Statement}
We affirm that this research complies with the accepted ethical standards in scientific research. All simulations and methodologies were conducted with integrity and transparency, without harm to individuals, groups, or the environment. From a perspective of social impact, we ensured that the theoretical and practical contributions of this work are aimed at advancing knowledge in a responsible and ethical manner, with no misuse or malicious application of the techniques proposed. Additionally, no conflicts of interest or external influences have compromised the objectivity or scientific rigour of this work.





\bibliography{main}  
\clearpage

\newpage
\appendix
\onecolumn

\section{Technical Details and Definitions}

\subsection{Compact and Closed Sets}\label{sec:compac_and_closed_set}

In this formal definition, \(\stEmb\) is both compact and closed in the product space \(\mathcal{A} \times \mathcal{B}\). A set \(\stEmb\) is compact if for every open cover \(\{U_i\}_{i \in I}\) of \(\stEmb\), there exists a finite subcover such that \(\stEmb \subseteq \bigcup_{k=1}^n U_{i_k}\), where \(U_{i_k}\) are open sets in \(\mathcal{A} \times \mathcal{B}\). This ensures that \(\stEmb\) is "contained" in a finite manner within the space, even if \(\mathcal{A} \times \mathcal{B}\) is infinite. Furthermore, \(\stEmb\) is closed if its complement, \(\stEmb^c = (\mathcal{A} \times \mathcal{B}) \setminus \stEmb\), is open. This implies that \(\stEmb\) contains all its limit points, making it a complete set within the topological space. Thus, \(\stEmb\) is a compact and closed subset of \(\mathcal{A} \times \mathcal{B}\), meaning that it is both bounded and contains its boundary, providing useful properties for convergence and stability in this space.
\begin{align}
     \forall \{U_i\}_{i \in I}, \quad \stEmb \subseteq \bigcup_{i \in I} U_i \implies \exists \{U_{i_1}, U_{i_2}, \dots, U_{i_n}\} \text{ such that } \stEmb \subseteq \bigcup_{k=1}^{n} U_{i_k}, 
     (\mathcal{A} \times \mathcal{B}) \setminus \stEmb \text{ is open}. \label{eq:compact_closed_definition}
\end{align}

\subsection{Assumptions on Linear Reward Function} \label{sec:ass_lin_reward}

We ground our analysis in standard assumptions that enjoy broad consensus in statistics and online learning literature. First we ensure that the covariance matrix $\precMat$ is well-conditioned and positive semi-definite (PSD), with a regularization parameter $\lambda_{\text{reg}}$ balancing bias and variance, while the norm $||\phi(\aB^t,\bB^t) ||_{\precMat}$ must remain small to facilitate efficient uncertainty reduction. (These assumptions are outlined in detail in Appendix \ref{sec:ass_lin_reward}.) 

\begin{enumerate}

    \item \textbf{Covariance Matrix}: 
    \begin{align}
        \Sigma_T := \sum_{t=1}^{T} \phi(\aB^t,\bB^t) \phi(\aB^t,\bB^t)^\top + \lambda_{\text{reg}} I
    \end{align}
    $\phi(\aB^t,\bB^t)$ must ensure that the covariance matrix $\precMat$ (a.k.a. the inverse of the covariance matrix) is sufficiently large for effective learning. 
    
    \item \textbf{Norm Bounds}:
    \begin{align}
        \| \phi(\aB^t,\bB^t) \|_{\precMat}  \equiv \sqrt{\phi(\aB^t,\bB^t) \precMat \phi(\aB^t,\bB^t)^\top}
    \end{align}

    $\| \phi(\aB^t,\bB^t) \|_{\precMat}$ must be small to ensure efficient uncertainty reduction.
    
    \item \textbf{Regularization Effect}: Regularization parameter \(\lambda_{\text{reg}}\) balances bias and variance, affecting sample complexity.
    \item \textbf{Positive Semi-Definiteness}: $\precMat$ is positive semi-definite (PSD).
\end{enumerate}

\subsection{Discrete Measure Interpretation}
Let $\{x_1, x_2, \ldots, x_n\}$ be a set of discrete points in $\mathbb{R}^n$. We define the measure $\alpha$ on these points as,
\begin{align}
    \alpha = \sum_{i=1}^n \alpha(\{x_i\}) \delta_{x_i}
\end{align}
where $\delta_{x_i}$ is the Dirac measure centered at $x_i$. The integral of a function $f: \mathbb{R}^n \to \mathbb{R}$ with respect to the measure $\alpha$ is given by,
\begin{align}
    \int_{\mathbb{R}^n} f(x) \, d\alpha(x) = \sum_{i=1}^k \alpha(\{x_i\}) f(x_i)
\end{align}

\subsection{Definition of Riemann Manifold} \label{sec:riemmanian_manifold_definition}


A Riemannian manifold, expressed as \( \stEmb \), consists of a smooth manifold \( \stEmb \) equipped with a smoothly varying collection of inner products \( \omega_p \) on each tangent space \( T_p \stEmb \) at every point \( p \in \stEmb \). This assignment \( \omega_p: T_p \stEmb \times T_p \stEmb \rightarrow \mathbb{R} \) is positive-definite, meaning it measures angles and lengths in a consistent and non-degenerate manner. Consequently, each vector \( \mathbf{v} \in T_p \stEmb \) inherits a smoothly defined norm \( \|\mathbf{v}\|_p = \sqrt{\omega_p(\mathbf{v}, \mathbf{v})} \). This structure allows \( \stEmb \) to possess a locally varying yet smoothly coherent geometric framework.

\subsection{Stochastic Perturbation Function} \label{sec:perturb_function_defn}

To model uncertainty in the joint action space, we introduce a stochastic perturbation over the leader and follower actions. Specifically, we define a small, one-step random perturbation function \(\mathfrak{J}(\mathbf{a}, \mathbf{b})\), where \(\mathbf{a} \in \mathbb{R}^m\) and \(\mathbf{b} \in \mathbb{R}^n\) are the actions of the leader and follower, respectively. The perturbed joint action is given by:
\begin{align}
    \mathfrak{J}(\mathbf{a}, \mathbf{b}) &= \left( \mathbf{a}', \mathbf{b}' \right) = \left( \mathbf{a} + \epsilon_a, \mathbf{b} + \epsilon_b \right)
\end{align}
where \(\epsilon_a \in \mathbb{R}^m\) and \(\epsilon_b \in \mathbb{R}^n\) are independent Gaussian perturbations with zero mean and variance \(\sigma_a^2\) and \(\sigma_b^2\), respectively:
\begin{align}
    \epsilon_a &\sim \mathcal{N}(0, \sigma_a^2 I_m), \quad \epsilon_b \sim \mathcal{N}(0, \sigma_b^2 I_n)
\end{align}
Here, \(\sigma_a\) and \(\sigma_b\) are scalar diffusion parameters controlling the magnitude of the perturbation, and \(I_m\) and \(I_n\) are identity matrices of size \(m \times m\) and \(n \times n\), ensuring isotropic perturbations in each component of \(\mathbf{a}\) and \(\mathbf{b}\).

In component form, this perturbation can be written as:
\begin{align}
    a_i' = a_i + \epsilon_{a_i}, \quad \epsilon_{a_i} \sim \mathcal{N}(0, \sigma_a^2), \qquad
    b_j' = b_j + \epsilon_{b_j}, \quad \epsilon_{b_j} \sim \mathcal{N}(0, \sigma_b^2)
\end{align}

This formulation introduces small, independent, and isotropic random deviations from the original actions, modeling the stochastic uncertainty in the decision-making process.

\subsection{Geodesic Repulsion Loss} \label{sec:geodesic_repul_loss_defn}

To encourage an even distribution of points on the spherical manifold, we employ the \textit{Geodesic repulsion loss}, which penalizes pairs of points that are too close in geodesic distance. This loss function facilitates the spreading out of points uniformly over the sphere, preventing clustering.

\textbf{Geodesic Distance:} Let \( \mathbf{y}_i, \mathbf{y}_j \in \mathbb{R}^D \) be points on the surface of a Riemmanian manifold denoted as $\geoDis(\mathbf{y}_i, \mathbf{y}_j)$ in the abstract sense. For a unit sphere it would hold that \( \|\mathbf{y}_i\| = \|\mathbf{y}_j\| = 1 \)). The geodesic distance between two points on the sphere is the angle between them, which can be computed from their dot product,
$    \geoDis(\mathbf{y}_i, \mathbf{y}_j) = \arccos \left( \mathbf{y}_i^\top \mathbf{y}_j \right),
$
where \( \mathbf{y}_i^\top \mathbf{y}_j \) is the dot product of \( \mathbf{y}_i \) and \( \mathbf{y}_j \).

\textbf{Repulsion Term:} To penalize pairs of points that are close in geodesic distance, we use an exponential decay function, which strongly penalizes small distances:
\begin{equation}
    \exp\left(-\frac{\geoDis(\mathbf{y}_i, \mathbf{y}_j)}{\gamma}\right),
\end{equation}
where \( \gamma > 0 \) is a sensitivity parameter controlling how strongly the loss reacts to small distances. A smaller \( \gamma \) enforces stronger repulsion between nearby points.

\textbf{Geodesic Repulsion Loss:} 
The total Geodesic Repulsion Loss is computed as the sum of repulsion terms over all pairs of points, excluding the diagonal (self-repulsion),
\begin{equation}
    \repuLoss = \sum_{i=1}^{n} \sum_{j=1, j \neq i}^{n} \exp\left(-\frac{\arccos \left( \mathbf{y}_i^\top \mathbf{y}_j \right)}{\gamma}\right),
\end{equation}
where \( n \) is the number of points on the manifold. The geodesic distance \( \geoDis(\mathbf{y}_i, \mathbf{y}_j) \) is computed using the angle between \( \mathbf{y}_i \) and \( \mathbf{y}_j \), ensuring that points are uniformly spaced across the spherical manifold.

To avoid penalizing points for being close to themselves, we exclude the self-repulsion terms by masking the diagonal elements in the pairwise distance computation,

\begin{equation}
    \geoDis(\mathbf{y}_i, \mathbf{y}_i) = 0, \quad \text{for all } i.
\end{equation}

This formulation ensures that points are pushed apart when their geodesic distances are too small, leading to a more uniform distribution on the manifold, which is critical for preserving the geometry of the learned representation.

\subsection{Negative Log-Likelihood Loss for Normalizing Flows} \label{sec:nll_loss_defn}

Let \( x \in \mathbb{R}^d \) be an input data point, and let \( f: \mathbb{R}^d \to \mathbb{R}^d \) be an invertible transformation defined by the normalizing flow. The transformation \( f \) maps the input data \( x \) to a latent variable \( z = f(x) \) that follows a simple base distribution \( p_Z(z) \). Assume that the base distribution is a standard normal distribution, \( Z \sim \mathcal{N}(0, I_d) \), with the probability density function (PDF) given by,
\begin{equation}
    p_Z(z) = \frac{1}{(2\pi)^{d/2}} \exp \left( -\frac{1}{2} \|z\|^2 \right).
\end{equation}
The log probability under this distribution is,
$
    \log p_Z(z) = -\frac{1}{2} \|z\|^2 - \frac{d}{2} \log(2\pi).
$
Using the change of variables formula, the probability density of \( x \) under the model is related to the base distribution via the transformation \( f \) as follows,
\begin{equation}
    p_X(x) = p_Z(f(x)) \left| \det \frac{\partial f(x)}{\partial x} \right|.
\end{equation}
Where \( \frac{\partial f(x)}{\partial x} \) is the Jacobian matrix of \( f \) with respect to \( x \), and \( \left| \det \frac{\partial f(x)}{\partial x} \right| \) is the absolute value of the determinant of the Jacobian.

\textbf{NLL Loss:} The negative log-likelihood (NLL) loss for a single data point \( x \) is defined as,

\begin{equation}
    \nfLoss(x) = - \log p_X(x) = - \left[ \log p_Z(f(x)) + \log \left| \det \frac{\partial f(x)}{\partial x} \right| \right].
\end{equation}

Substituting the log probability of \( z = f(x) \) under the base distribution:
\begin{equation}
    \nfLoss(x) = \frac{1}{2} \|f(x)\|^2 + \frac{d}{2} \log(2\pi) - \log \left| \det \frac{\partial f(x)}{\partial x} \right|.
\end{equation}

For a dataset \( \{x_i\}_{i=1}^n \), the total NLL loss is the average over all data points:
\begin{equation}
    \nfLoss = \frac{1}{n} \sum_{i=1}^n \left( \frac{1}{2} \|f(x_i)\|^2 + \frac{d}{2} \log(2\pi) - \log \left| \det \frac{\partial f(x_i)}{\partial x_i} \right| \right).
\end{equation}

The objective of training is to minimize \( \nfLoss \), ensuring that the transformed latent variables \( z = f(x) \) follow the base distribution and the transformation \( f \) appropriately adjusts the volume of space via the Jacobian determinant.

\section{Optimization Algorithms}

\subsection{Optimization of Stackelberg Games} \label{sec:stack_games_opt_app}

\textbf{Optimization under Perfect Information:} We see that regardless of the convexity of $\setA$ or $\setB$, so long as we are dealing with compact spaces, under perfect information, we can solve the Stackelberg equilibrium by solving a bilevel optimization problem expressed as,

\begin{minipage}{0.5\textwidth}
  \begin{equation}
    \pi_A^* = \arg\max_{\polA\in \Pi_A} \, \mathbb{E}\big[\innerP{\theta^*_A, \, \phi(\polA, \pi_B^*(\polA)) }\big],
  \end{equation}
\end{minipage}
\hfill
\begin{minipage}{0.5\textwidth}
  \begin{equation}
    \pi_B^*(\polA) := \argmaX{\polB\in \Pi_B} \, \mathbb{E}\big[\innerP{\theta^*_B, \, \phi(\polA, \polB)}\big],\label{eq:bottom_bilevel}
  \end{equation}
\end{minipage}

with a slight abuse of notation, we use $\phi(\polA, \polB)$ and $\pi_B^*(\polA)$ to denote $\mathbb{E}_{\polA,\polB}[\phi]$ and  the best response function in response to policy $\polA$, respectively. The expectation are taken with respect to the sub-Gaussian noises.  For some no-regret learning algorithm, suppose that after observing $t$ samples, the uncertainty among the parameters $\theta$, is characterized by,
\begin{align}
    \Ball(\theta^*,\uncerBall{t}):=\Big\{\theta: \norm{\theta^* - \theta} \leq \uncerBall{t}\Big\},
\end{align}

with probability at least $1-\delta(t)$. In this formulation, $||\cdot||$ denotes some norm in the space of parameters. 

Assuming a \textit{pessimistic} leader, the optimization problem under parameter uncertainty at round $t$ can be expressed as,
\begin{align}
    &\pi_A^* := \arg\max_{\polA\in \Pi_A} \min_{\theta_A}\, \mathbb{E}\big[\innerP{\theta_A, \, \phi(\polA, \pi_B^*(\polA)) }\big],\quad &\text{s.t.} \quad  \theta_A \in \Ball(\theta^*_A,\uncerBall{t}), \label{eq:pi_a_star_app}
\end{align}
\begin{align}
     &\text{where}\quad \pi_B^*(\polA) := \argmaX{\polB\in \Pi_B} \max_{\theta_B}\,\mathbb{E}\big[\innerP{\theta_B, \, \phi(\polA, \polB)}\big],\, &\text{s.t.} \,  \theta_B \in \Ball(\theta^*_B,\uncerBall{t}). \label{eq:pi_b_star_app}
\end{align}

Given $\pi_B^*(\cdot)$ in Eq. \eqref{eq:pi_b_star}, let us define,
\begin{align}
   & \underline{\mathcal{H}}(\theta_A^*,t) := \max_{\pi_A\in\Pi_A}\min_{\theta_A}\, \mathbb{E}\big[\innerP{\theta_A, \, \phi(\polA, \pi_B^*(\polA)) }\big], \quad &\text{s.t.} \quad  \theta_A \in \Ball(\theta^*_A,\uncerBall{t}), \label{eq:min_H_app} \\
     & \overline{\mathcal{H}}(\theta_A^*,t) := \max_{\pi_A\in\Pi_A}\max_{\theta_A}\, \mathbb{E}\big[\innerP{\theta_A, \, \phi(\polA, \pi_B^*(\polA)) }\big], \quad &\text{s.t.} \quad  \theta_A \in \Ball(\theta^*_A,\uncerBall{t}). \label{eq:max_H_app}
\end{align}

We can see from the structure of Eq. \eqref{eq:pi_a_star_app} to Eq. \eqref{eq:max_H_app}, the resemblance to a bi-level optimization problem, which can be solved both under perfect information and uncertainty. We provide a discussion of such methods in Appendix \ref{sec:bilevel_extra_discuss}.

\subsection{Bi-level Optimization Structure}\label{sec:bilevel_extra_discuss}

\textbf{Bi-level Optimization Structure:} The optimization problems represented by Eqs. \eqref{eq:pi_a_star} and \eqref{eq:pi_b_star} exhibit the structure of a \textit{bi-level optimization} problem \citep{beck:2023_bilevel_survey, sinha:2017_bilevel_review, balling:1995_multilevel_algorithm}. Generally, a bilevel optimization problem comprises an upper-level optimization task with an embedded lower-level problem, where the solution to the upper-level problem depends on the solution to the lower-level one. Two conventional methods have been employed to address the bilevel optimization problem. The first leverages the Karush-Kuhn-Tucker (KKT) conditions to exploit the optimality of the lower-level problem (see Appendix \ref{sec:kkt_stack_formulation}). The second employs gradient-based algorithms like gradient ascent (discussed in Appendix \ref{sec:grad_ascent_stack_formulation}). Both approaches, however, have notable limitations. KKT conditions assume strong convexity or pseudo-convexity, making them unsuitable for many non-convex settings, while gradient-based methods, in addition to being computationally inefficient, often struggle or converge poorly when weak-convexity is not guaranteed. Moreover, these methods typically assume optimization under perfect information, whereas we focus on learning-based frameworks with uncertainty due to sampling.

\subsection{KKT Reformulation for Solving Stackelberg Optimization Problems} \label{sec:kkt_stack_formulation}

 The bi-level optimization structure can be solved via reformulating the problem as a bilevel optimization problem via the Karush-Kuhn-Tucker (KKT) conditions. It assumes convexity and differentiability in the embedded space and transforms the original bilevel problem into a single-stage optimization problem via the KKT conditions. 
\begin{equation}
\begin{aligned}
\max_{\polA,\polB,\lambda} \,& \innerP{\theta_A, \, \phi(\polA, \polB) } \\
\text { s.t. } & \polA \in \Pi_A \\
& \nabla_{\polB} \innerP{\theta_B, \, \phi(\polA, \polB)}+\sum_{i=1}^{\ell} \lambda_i \nabla_{\polB} g_{i}(\polB)=0, \\
& \lambda^{\top} g(\polB)=0, \quad  \lambda \geq 0, \quad  g(\polB) \geq 0,
\end{aligned}
\end{equation}
where $\Pi_B=\{\polB|g(\polB)\geq 0\}$ and $g_i$ represents the $i$-th constraint of $\Pi_B$.
Specifically, it requires the convexity of the lower level problem \eqref{eq:bottom_bilevel}. Otherwise, KKT complementarity conditions turns the problem into a nonconvex and nonlinear problem even $\pi_B$ is a set of linear constraints. And the problem is incapable to solve under normal nonconvex and nonlinear algorithm. In addition, Slater's constraint qualification is required to ensure that the solution under KKT reformulation is the solution of original bilevel problem.\citep{allende2013solving}The reformulation involves converting non-linear constraints into a convex hull, thus simplifying the problem into a linear program (LP). Sensitivity analysis can be then performed to understand how changes in constraints impact the solution, with particular attention to the effects of shrinking parameters on the objective function. The approach is utilizes the application of the Weak Duality Theorem to analyze sensitivity.

\subsection{Gradient Ascent Approach for Solving Bilevel Optimization Problems} \label{sec:grad_ascent_stack_formulation}

Another approaches is transforming Stackelberg game into the the bilevel optimization problem. Namely, we are interested in the following problem,
\begin{equation}
\begin{aligned}
& \min _{x \in \mathbb{R}^d, y \in y^*(x)} f(x, y),\ \text{(Upper-Level)} \\ 
& \text { s.t. } y^*(x) \equiv \arg \min _{y \in Y} g(x, y).\ \text{(Lower-Level)}
\end{aligned}
\end{equation}
The gradient-based algorithms have seen a growing interest in the bilevel problem \citep{ji2021bilevel,sato:2021gradient,xiao2023generalized,liu2021towards,huang2022enhanced,huang2024optimal}. To measure the stationarity of the lower-level problem, Polyak-Lojasiewicz(PL) condition on $g(x,\cdot)$ is widely applied to show the last-iterate convergence of $\|\nabla_x f(x^t,y^\star(x^t))\|$, i.e,
\begin{equation}
    \|\nabla_y g(x,y)\|^2\geq 2\mu(g(x,y)-\min_z g(x,z)).
\end{equation}
where $\mu$ is a positive constant. This condition relaxes the strong convexity but is still not satisfied for the polynomial function $g(x,y)=y^4$. Also, the lower level function $g(x,\cdot)$ needs to be differentiable in $R^d$. For the Stackelberg game, this is not the case since the follower's strategy $\polB\in\Pi_B$.

Interestingly, should the objective functions be differentiable, one strategy to do this optimization is via gradient descent. Of course the gradient descent algorithm would have to be reformulated to accommodate to a finite amount of traversals based on the gradient update \citep{sato:2021gradient} \citep{franceschi:2017forward} \citep{naveiro:2019_gradient_sg}.

\subsection{Technical Note: Conversion of Absolute Value Constraints into Regular LP Constraints} \label{sec:tech_note_absval_const_to_lp}

Suppose there exists $D$ dimensions on the L1 norm. Wnd we have the constraint,
\begin{align}
    \norm{\mathbf{x} - \mathbf{c}}_1 \leq D, \quad \text{expressed as,} \quad \, \sum_{i=1}^D |x_i - c_i| \leq C
\end{align}
This can be expressed as, 
\begin{align}
    &z_i \geq x_i - c_d, \quad \text{for } i = 1, 2, \ldots, D \\
    &z_i \geq -(x_i - c_d), \quad \text{for } i = 1, 2, \ldots, D \\
    &\sum_{i=1}^D z_i \leq C, \quad z_i\geq 0, \quad \text{for } i = 1, 2, \ldots, D
\end{align}
By introducing a new dummy variable $z_i$, we and adding $2D+1$ additional constraints, we can express this now as a standard linear program.

\section{Topology \& Geodesy}

\subsection{Convex manifold Definitions:} \label{sec:convex_manifold_desc}

\begin{definition} \label{def:geodesic_convex_subsets}
    \textbf{Geodesically Convex Sets:} Let $(\stEmb, h)$ be a Riemannian manifold, where $\stEmb$ is a smooth manifold and $h(\aB, \bB)$ is a Riemannian metric on $\stEmb$ (i.e. innner product). A subset $\convexGeoSet \subseteq \stEmb$ is said to be \textit{geodesically convex} if for any two points $\aB, \bB \in \convexGeoSet$, there exists a geodesic $\tau: d \in [0, 1] \to \stEmb$ parameterized by $d$ such that, 
    \begin{align}
        \tau(0) = \aB, \, \, \tau(1) = \bB, \, \qquad \text{and}, \, \, \tau(d) \in \convexGeoSet, \quad \forall d \in [0, 1].
    \end{align}
\end{definition}

Where $d$ can be viewed as a parameter proportional to the distance travelled along the geodesic. In other words, a set $\convexGeoSet$ is geodesically convex if for any two points in $\convexGeoSet$, there exists a geodesic between these points that lies entirely within $\convexGeoSet$. 

\begin{definition} \label{def:convex_manifold}
    \tocheck{\textbf{Locally Convex Manifolds:} A locally convex manifold is a manifold where the geodesic between any two points on the manifold falls within, or constitutes, a geodesically convex set $\mathcal{S}_\stEmb$, as per Definition \ref{def:geodesic_convex_subsets}.}
\end{definition}

\textbf{Remark.} It is key to note that a spherical surface is not, by formal definition, a convex manifold in the intrinsic (Riemannian) sense. The only violating principle is that the geodesic minimizing distance between two points on the manifold is not unique at the antipodal points of a spherical surface. Nevertheless, what is important to convey is that the ambient Euclidean space bounded by the spherical manifold is convex, the sphere itself can thus be viewed as an extrinsically convex hypersurface. Thereafter, the geodesic properties are used for efficient equilibrium computation, and the dimensional properties in the ambient space determine the reward output.

\subsection{Specifications of the Ideal Stackelberg Manifold} \label{sec:ideal_stack_manifold}

The total loss $\mathcal{L}(\phi)$ is composed of multiple loss functions added together in a linear convex combination, as described in Eq. \eqref{eq:manifold_total_loss}. To construct the Stackelberg manifold $\stEmb$, data is first sampled uniformly from the ambient Cartesian space. We then, fit a normalizing flow to $\stEmb$  based on the criteria in Table \ref{tab:stack-emb-dynamics}, minizing $\mathcal{L}(\phi)$. 

\begin{table*}[h]\centering 
    \begin{enumerate}[-, start=1,label={(\bfseries D\arabic*)}, wide, labelwidth=!, labelindent=0pt, topsep=0pt, itemsep=-8pt]
    \small
    \begin{tabular}{p{2.7in}p{3.7in}} \toprule
        \textbf{Definition} & \textbf{Expression} \\ \midrule
        \vspace{-1em} \item $\stEmb$ is measurable and reachable w.r.t. a $\sigma$-algebra over $\mathcal{A} \cross \mathcal{B}$ (denoted as $\mathfrak{E}_{\mathcal{A} \times \mathcal{B}}$).  \label{enu:reachability_dynamic} & \vspace{-2.3em} 
        \begin{align} \stEmb \subseteq \mathcal{A} \times \mathcal{B} \subseteq\mathbb{R}^D, \quad \stEmb \in \mathfrak{E}_{\mathcal{A} \times \mathcal{B}}. 
        \end{align}  \vspace{-2.3em} \\ \midrule 
        \item $\stEmb$ is compact and closed.  & \vspace{-2.3em}
        \begin{align}\text{See Appendix \ref{sec:compac_and_closed_set} for detailed definition.} \end{align} \vspace{-1.8em} \\ \midrule 
        \item $\stEmb$ is Lipschitz in the joint $\setA \cross \setB$. \label{enu:lipschitz_stack_emb}  & \vspace{-2.3em}
        \begin{align}
            \Big| \norm{\nabla_{\aB}\phi}_p + \norm{\nabla_{\bB}\phi}_p- C \Big| \leq L_c \label{eq:lipschitz_stack_emb} 
        \end{align}  \vspace{-1.8em} \\ \midrule 
        \item $\stEmb$ variational sensitivity in $\setA \cross \setB$. \label{enu:variational_sensitivity}  & \vspace{-2.3em}
        \begin{align}
            ||\aB - \aB'|| \leq \epsilon \implies ||\phi(\aB', \bB) - \phi(\aB, \bB)|| \leq \delta, \forall \bB \\
            ||\bB - \bB'|| \leq \epsilon \implies ||\phi(\aB, \bB') - \phi(\aB, \bB)|| \leq \delta, \forall \aB
            \label{eq:variational_sensitivity} 
        \end{align}  \vspace{-1.8em} \\ \midrule 
        \item $\stEmb$ forms a smooth Riemannian manifold. & \vspace{-2.3em} \begin{align} \text{See Appendix \ref{sec:riemmanian_manifold_definition} for detailed definition.}  \end{align}  \vspace{-1.8em} \\ \midrule
        \item $\stEmb$ has an approximate pullback. There exists $\phi^{-1}(\cdot): \stEmb \mapsto \setA \cross \setB$ such that, \label{enu:pullback_dynamic} & \vspace{-2.3em} \begin{align} \norm{\phi^{-1}(\phi(\aB,\bB)) - (\aB,\bB)} \leq \epsilon,\, \forall \aB,\bB 
        \end{align}   \vspace{-2.3em} \\ \midrule
    \end{tabular} \caption{Key criteria of the Stackelberg Manifold $\stEmb$.} \label{tab:stack-emb-dynamics}
\end{enumerate}
\end{table*}

The manifold learning process trains the bipartites normalizing flow architecture (Fig. \ref{fig:bipartite_nf}) to adhere to the desiderata \ref{enu:reachability_dynamic} to \ref{enu:pullback_dynamic} from Table \ref{tab:stack-emb-dynamics} as much as possible. This requires a trade-off between being well behaved on the manifold, as stipulated by \ref{enu:lipschitz_stack_emb} and \ref{enu:variational_sensitivity}, and having an accurate inverse \ref{enu:pullback_dynamic}. 

In general, this is not achievable via simplistic Euclidean maps, such as projections or affine maps. Without normalizing flows, the mapping of the ambient data to the manifold can become overly localized, sporadically clustered, sensitive to perturbation, and/or non-invertible, rendering the embedding ineffective (we present such scenarios in Appendix \ref{sec:bad_stackelberg_embeddings_viz}). Therefore, we must ensure the data is evenly spread across the manifold while maintaining the desired properties. Thus, we train a neural network to \changesMarker{represent} $\phi$, with the loss function,

The aforementioned losses in Eq. \eqref{eq:manifold_total_loss} are linearly combined in a convex combination to form the total loss $\mathcal{L}(\phi)$, denoted as $\alpha_N$, $\alpha_R$, $\alpha_P$, and $\alpha_L$ respectively. The hyperparameters were optimized via a selection process, leveraging empirical validation to identify the settings that maximized performance. Experimental hyperparameters and architecture of the normalizing neural flow network can be found in Appendix \ref{sec:nn_arch_details}. \tocheck{To summarize, we aim to learn a well behaved embedding $\stEmb$ via normalizing flows, which constructs a bijective map between ambient space and manifold. $\mathcal{L}_\phi^N$ is the loss function for the normalizing flow, and the additional loss functions, as outlined in Table \ref{tab:stack-emb-dynamics}, aim to achieve uniform spreading, Lipschitz continuity, and bijective properties across the manifold.}

\subsection{\tocheck{Linear Relationship of the Embedding Space $\stEmb$ to the Reward Space} } \label{sec:note_embed_reward_linear_appendix}

The linear bandit assumption that the expected reward is the inner product between an embedding space and some linear parameters is commonly accepted in the online learning literature \citep{amani:2019linear_bandit_safety, moradipari:2022feature_map, zanette:2021-doe_bandit, lattimore:2020_bandit_book}. In our problem setting, we make the similar assumption that there exists a linear relationship between the embedding space and the reward space. To demonstrate rigorously, let $Y$ represent an expectation over the reward variable. Given a linear relation in the action space \(X\), suppose the reward can be expressed as,
    
    \[
    Y = \mathfrak{F}(X) = \langle \tilde{\theta}, \tilde{\phi}(X) \rangle.
    \]
    $\mathfrak{F}(X)$ constitutes an abstract function, and $\tilde{\phi}(\cdot)$, represents a feature map, a common tool in linear multi-armed bandit learning \citep{zanette:2021-doe_bandit}. Suppose our feature map, $\tilde{\theta}$ can be arbitrarily complex, we can therefore replicate the result of $\mathfrak{F}(X)$ via $\langle \tilde{\theta}, \tilde{\phi}(X) \rangle$ for most well-behaved functions. In one perspective, we can view the linear feature map as the final layer in a neural network output. Instead of learning the action to reward relation directly via a neural network per se, we learn the relation of the ambient space (or the joint action space) to the embedding space. 


\subsection{Proof of Lemma \ref{lem:existence_of_theta}} \label{prf:existence_of_theta}

\textbf{Linear Relation for Smooth Invertible Maps:} Suppose that \(Y\) can be expressed as \(Y = \langle \tilde{\theta}, \tilde{\phi}(X) \rangle\), where \(\tilde{\phi}(X): \mathbb{R}^d \to \mathbb{R}^d\) is smooth and bijective, and \(\tilde{\theta} \in \mathbb{R}^d\). Then  there exists an alternative set of parameters \(\theta \in \mathbb{R}^k\) and corresponding map  \(\phi(X): \mathbb{R}^d \to \mathbb{R}^k\) such that for any \(k \geq d\) such that, $Y = \langle \theta, \phi(X) \rangle$.


\begin{proof}

The proof of this lemma proceeds by first establishing the existence of an equivalent higher-dimensional representation of $\tilde{\phi}(X)$. Next we demonstrate that, should $\tilde{\phi}$ be smooth and bijective, then given another $\phi: X \to \phi(X)$ that is also smooth and bijective, a bijection, $T(\cdot)$, must exist between their respective images $T: \tilde{\phi}(X) \to \phi(X)$. Subsequently, we demonstrate that should a bijective transformation $T(\cdot)$ exist between the proposed $\phi(X)$ and $\tilde{\phi}(X)$, then a corresponding valid parameter set $\theta$ also exists. Thus, $\phi(X)$ and $\theta$ can serve as equivalent representations to $\tilde{\phi}(X)$ and $\tilde{\theta}$, with no  restrictions on the dimension of $\phi(X)$. 


\paragraph{Higher-Dimensional Embedding of \(\tilde{\phi}(X)\):} Let \(\tilde{\phi}'(X): \mathbb{R}^d \to \mathbb{R}^k\) for $k \geq d$ be a higher-dimensional embedding of \(\tilde{\phi}(X)\), defined as,
\[
\tilde{\phi}'(X) = \begin{bmatrix} \tilde{\phi}(X) \\ g(X) \end{bmatrix},
\]
where \(g(X): \mathbb{R}^d \to \mathbb{R}^{k-d}\) is a smooth function. We further define the extended parameter vector as,
\[
\tilde{\theta}' = \begin{bmatrix} \tilde{\theta} \\ 0 \end{bmatrix} \in \mathbb{R}^k.
\]
Then, the original relationship is preserved:
$
Y = \langle \tilde{\theta}, \tilde{\phi}(X) \rangle = \langle \tilde{\theta}', \tilde{\phi}'(X) \rangle.
$
The consequence is that we can always find equivalent higher dimension exact representations of $\tilde{\phi}(X)$ and $\tilde{\theta}$. So long as the intrinsic dimension of the spaces is preserved, $\tilde{\theta}$ can have an equivalent representation in any higher dimensional space. (In our implementation, we apply a series of bijective stereographic transformations in the ambient space $X$ to increase the dimension of the embedding space where necessary.) For convenience moving forward, we shall refer to $\tilde{\phi}'(X)$ as $\tilde{\phi}(X)$ in the $k^{th}$ dimension, as their representations are equivalent.






\paragraph{Existence of \( \phi(X) \):}  Let $\phi: X \to \phi(X)$ be another smooth bijection along with $\tilde{\phi}: X \to \tilde{\phi}(X)$. Then, there exists a bijection between $\tilde{\phi}(X)$ and $\phi(X)$. Since $\tilde{\phi}$ is bijective, it has an inverse function $\tilde{\phi}^{-1}: \tilde{\phi}(X) \to X$. Similarly, $\phi$ has an inverse function $\phi^{-1}: \phi(X) \to X$. Let us now define the function,
\begin{equation}
    T: \tilde{\phi}(X) \to \phi(X), \quad T(y) = \phi(\tilde{\phi}^{-1}(y)).
\end{equation}
Since $\tilde{\phi}^{-1}$ is bijective, it uniquely maps each $y \in \tilde{\phi}(X)$ to some $x \in X$. Then, $\phi$ uniquely maps this $x$ to an element in $\phi(X)$. We next demonstrate that $T(\cdot)$ is both injective and surjective.

\begin{enumerate}[I.]
    \item \textbf{Injectivity:} If $T(y_1) = T(y_2)$, then $\phi(\tilde{\phi}^{-1}(y_1)) = \phi(\tilde{\phi}^{-1}(y_2))$. Since $\phi$ is bijective, it follows that $\tilde{\phi}^{-1}(y_1) = \tilde{\phi}^{-1}(y_2)$, and applying $\tilde{\phi}$ to both sides gives $y_1 = y_2$.

    \item \textbf{Surjectivity:} For every $z \in \phi(X)$, there exists $x \in X$ such that $z = \phi(x)$. Since $\tilde{\phi}$ is bijective, there exists $y = \tilde{\phi}(x) \in \tilde{\phi}(X)$. Thus, $T(y) = \phi(\tilde{\phi}^{-1}(y)) = z$, demonstrating surjectivity.

\end{enumerate}

Therefore, given that $T$ is both an injection and surjection, $T$ constitutes a bijection between $\tilde{\phi}(X)$ and $\phi(X)$.

\paragraph{Existence of $\langle \theta, \phi(X) \rangle$:} Given there there exists a bijection \(T: \mathbb{R}^k \to \mathbb{R}^k\), where for any higher-dimensional smooth and bijective mapping \(\phi(X): \mathbb{R}^d \to \mathbb{R}^k\), the following relation holds,
$
\phi(X) = T(\tilde{\phi}(X)).
$ 
Since \(T\) is bijective, the inverse \(T^{-1}\) exists and is smooth. This implies that
$
\tilde{\phi}(X) = T^{-1}(\phi(X)).$ 
Given that there exists some bijective transformation $T: \tilde{\phi}(X) \mapsto \phi(X)$, and its inverse $T^{-1}: \phi(X) \mapsto \tilde{\phi}(X)$, that is smooth and differentiable w.r.t. X. We can then express, 
    \[
    Y = \langle \tilde{\theta}, T^{-1}(\phi(X)) \rangle
    \]
    
\paragraph{Change of Variables:} To obtain a swap from $\tilde{\theta}$ to $\theta$ ,we select $\theta$ as,
$        \theta = (J_T^\top)^{-1} \tilde{\theta},$     
    where $J_T$ is the Jacobian of $T(\cdot)$ w.r.t. $\tilde{Y}$, where $\tilde{Y} \in \text{Im}(\tilde{\phi})$. Therefore, we show the existence of $\theta$ s.t.,
    \[
    Y = \langle \theta, \phi(X) \rangle,
    \]
    where $\phi(X) \equiv T \circ \tilde{\phi}(X)$. This argument relies on the bijectiveness and smoothness of $T(\cdot)$ consequently being a diffeomorphism, a fair assumption so long as we consider $\tilde{\phi}(X)$ to be smooth and well-behaved.







\end{proof}

\subsection{Notes on Scaling and Approximation Error for Non-Smooth $\stEmb$}

For non-smooth $\tilde{\phi}(X)$ we could settle for a continuous-to-discrete approximation, particularly useful for large discrete action spaces. As there are no limitations to what the learned feature map $\phi$ can be, it gives us great flexibility when it comes to constructing $\phi$ to approximate $T \circ \tilde{\phi}(\cdot)$. For this purpose, we apply normalizing flows  \citep{brehmer:2020_manifold_flows, durkan:2020_nflows, dinh:2016density} a  technology developed specifically for mapping ambient data to a desired manifold, $\stEmb$, of which multiple sufficient approximations for $\phi$ could potentially exist. This enables the construction of $\stEmb$ such that the linear relation between the embedding space and reward space will hold.

\paragraph{Scaling and Approximation of \(T(\cdot)\):} Suppose we are uncertain if $T(\cdot)$ exists in the higher dimension $k$. For any proposal of \(\phi(X)\) that is smooth, bijective, and generally well-behaved, should \(\phi(X)\) be a sufficiently overparameterized mapping (i.e., \(k \gg d\)), we can then approximate any bijective transformation between \(\tilde{\phi}(X)\) and \(\phi(X)\) with arbitrarily small error. Specifically, for any \(\epsilon > 0\), there exists a sufficiently large \(k\) such that:
\begin{align}
    \|T(\tilde{\phi}(X)) - \phi(X)\| < \epsilon, \qquad \text{and similarly,} \qquad \|T^{-1}(\phi(X)) - \tilde{\phi}(X)\| < \epsilon. \label{eq:tmap_est_error_v2}
\end{align}


This follows from the universal approximation property \citep{hornik:1989_universal_approx_multilayer, cybenko:1989_universal_approximation}, which states that increasing dimensionality provides additional degrees of freedom to represent complex transformations, effectively reducing approximation error.

\paragraph{Diminishing Approximation Error:} Let $\epsilon_{\texttt{MAP}}$ denote the approximation error in the representation of \(Y\) arising from the error in \(T(\cdot)\) and is given by:
\begin{align}
    \epsilon_{\texttt{MAP}} = \langle \tilde{\theta}, \tilde{\phi}(X) \rangle - \langle \theta, \phi(X) \rangle. \label{eq:e_map_inner_prod_expr}
\end{align}
Should we then apply the Cauchy-Schwarz inequality over Eq. \eqref{eq:e_map_inner_prod_expr}, and substitute the approximation of \(T(\cdot)\) from Eq. \eqref{eq:tmap_est_error_v2}, the error $\epsilon_{\texttt{MAP}}$ can be bounded by,
\[
\epsilon_{\texttt{MAP}} \leq \|\tilde{\theta}\| \cdot \|T^{-1}(\phi(X)) - \tilde{\phi}(X)\| = \|\tilde{\theta}\| \cdot \epsilon,
\]

which is upper-bounded by a factor of $\epsilon$ w.r.t. $\|T^{-1}(\phi(X)) - \tilde{\phi}(X)\| \leq \epsilon$. By scaling \(k\) sufficiently large, we reduce \(\epsilon\) to an arbitrarily small value, making \(\epsilon_{\texttt{MAP}}\) negligible. Thus, given the initial expression $Y = \langle \tilde{\theta}, \tilde{\phi}(X) \rangle$, there always exists a higher-dimensional \(\phi(X)\) and a parameter vector \(\theta\) such that,
$
Y = \langle \theta, \phi(X) \rangle,$
with negligible approximation error, $\epsilon_{\texttt{MAP}}$.

With respect to Lemma \ref{lem:existence_of_theta}, the absence of explicit knowledge of the transformation \(T: \tilde{\phi}(X) \mapsto \phi(X)\) does not undermine the validity of the result, provided there are no restrictions on the dimensionality of \(\phi(X)\). In our methodology, no such restrictions are imposed. Through \textit{overparameterization}, embedding \(X\) into a higher-dimensional space \(\phi(X)\) enables the mapping to indirectly capture the latent structure of \(\tilde{\phi}(X)\) and \(T(\cdot)\). Additionally, by constructing \(\phi(X)\) as a sufficiently complex approximation (using techniques such as normalizing flows) we ensure it can approximate any smooth, bijective transformation $T: \tilde{\phi}(X) \mapsto \phi(X)$. This flexibility allows \(\phi(X)\) to accommodate a wide range of mappings, ensuring that the representation \(Y = \langle \theta, \phi(X) \rangle\) holds even in the absence of explicit knowledge of \(T(\cdot)\). Simply put, if \(\phi(X)\) is sufficiently complex, an equivalent representation of the original \(Y = \langle \theta, \phi(X) \rangle\) can always be achieved. An alternative interpretation can also be drawn from the \textit{universal approximation theorem} \citep{hornik:1989_universal_approx_multilayer, cybenko:1989_universal_approximation}, which states that a sufficiently large neural network with a nonlinear activation function can approximate any continuous function on a compact domain to arbitrary precision.

\subsection{Proof for Lemma \ref{lem:geodesic_and_closeness_phi}} \label{prf:geodesic_and_closeness_phi}

\tocheck{Optimization on a convex manifold, such as a sphere presents unique advantages to simplify the problem. The first simplification, Lemma \ref{lem:geodesic_and_closeness_phi}, lies in the fact that the dot product optimization problem can first be represented as a geodesic distance minimizing problem, which provides geometric interpretations. Next, from Lemma \ref{lem:pure_strategy_convex_manifold}, the solution to any optimal condition for either player is unique due to convexity.}

\tocheck{Lemma \ref{lem:geodesic_and_closeness_phi} posits that the divergence angle between objective vectors correlates with geodesic distance on the manifold, simplifying the best response computation under uncertainty and enabling characterization of rational follower behaviour in constrained spaces.}

\begin{proof}
    \textbf{Geodesic Distance and Closeness to $\xi_\theta$:} Since $\mathcal{M}$ is a smooth, compact manifold bounding a convex region, the geodesic distance between two points on $\mathcal{M}$, say $\xi_1, \xi_2 \in \mathcal{M}$, is defined as the shortest path along the manifold $\geoDis(\xi_1, \xi_2)$ between $\xi_1$ and $\xi_2$. For convex manifolds, the geodesic distance behaves similarly to the distance on the surface of a sphere: an increase in the geodesic distance from $\xi_\theta$ to another point on the manifold corresponds to an increase in the angle between the tangent vector at $\xi_\theta$ and the vectors corresponding to points on the manifold. Hence, if $\geoDis(\xi_\theta, \xi_{\theta_A}) > \geoDis(\xi_\theta, \xi_{\theta_B})$, the angle between $\xi_\theta$ and $\xi_{\theta_A}$ is larger than the angle between $\xi_\theta$ and $\xi_{\theta_B}$.

    \textbf{Dot Product and Angle:} The dot product $\langle \theta, \xi \rangle$ between a normal vector $\theta$ at $\xi_\theta$ and a point $\xi$ on the manifold is given by:
    $\langle \theta, \xi \rangle = \|\theta\| \|\xi\| \cos(\alpha)$, 
    where $\alpha$ is the angle between the vectors $\theta$ and $\xi$. The angle between $\theta$ and any point $\xi$ on the manifold depends only on the angle between $\xi_\theta$ and $\xi$. Since $\geoDis(\xi_\theta, \xi_{\theta_A}) > \geoDis(\xi_\theta, \xi_{\theta_B})$ implies that the angle between $\xi_\theta$ and $\xi_{\theta_A}$ is larger than the angle between $\xi_\theta$ and $\xi_{\theta_B}$, we have,
    \begin{align}
        \cos(\alpha_{\xi_{\theta_A}}) < \cos(\alpha_{\xi_{\theta_B}}),
    \end{align}
    where $\alpha_{\xi_{\theta_A}}$ is the angle between $\theta$ and $\xi_{\theta_A}$, and $\alpha_{\xi_{\theta_B}}$ is the angle between $\theta$ and $\xi_{\theta_B}$. 
    
    \textbf{Conclusion on Dot Products:} Since the dot product is proportional to the cosine of the angle between the vectors, and $\cos(\alpha_{\xi_{\theta_A}}) < \cos(\alpha_{\xi_{\theta_B}})$, it \tocheck{follows that,}
    \begin{align}
        \langle \xi_{\theta}, \xi_{\theta_A} \rangle = \| \xi_{\theta} \| \| \xi_{\theta_A} \| \cos(\alpha_{\xi_{\theta_A}}) < \langle \xi_{\theta}, \xi_{\theta_B} \rangle = \| \xi_{\theta} \| \| \xi_{\theta_B} \| \cos(\alpha_{\xi_{\theta_B}}).
    \end{align}
    
\end{proof}

\begin{lemma} \label{lem:pure_strategy_convex_manifold}
    \textbf{Pure Strategy of the Follower:} While optimizing over an convex manifold, given any objective vector $\theta$, the linear structure of the reward functions from Eq. \eqref{eq:A_reward_inner_prod} and Eq. \eqref{eq:B_reward_inner_prod}, and that the subspace induced by $\aB \in \setA$ forms a geodesically convex subset of
    the optimal strategy of the follower will be that of a pure strategy, such that $\polA(\bB | \aB) \in \{0, 1\}$. 
\end{lemma}

\tocheck{Lemma \ref{lem:pure_strategy_convex_manifold} claims that on the convex manifolds, each objective vector and convex subspace has a unique reward-maximizing solution, ensuring a single equilibrium point (pure strategy).} Combined together, the intuition behind Lemmas \tocheck{\ref{lem:geodesic_and_closeness_phi} and \ref{lem:pure_strategy_convex_manifold} is that the maximum dot product between $\xi_{\theta_A}$ and $\xi_{\theta_B}$ on the convex manifold occurs when $\xi_{\theta_A}$ and $\xi_{\theta_B}$ are collinear, ensuring the optimal reward.} In the case of a convex subspace, the follower acting optimally has no viable alternatives other than a single choice. 
\begin{proof}
    The goal is to show that the follower's optimal strategy $\polA(\bB|\aB)$ is a pure strategy, such that $\polA(\bB|\aB) \in \{0, 1\}$. Let the objective vector $\theta \in \mathbb{R}^D$ define the direction of optimization, with the reward function given by, 
    \begin{align}
        \mu(\aB, \bB) = \langle \phi(\aB, \bB), \theta \rangle,
    \end{align}
    where $\phi: \setA \times \setB \to \mathbb{R}^D$ is a feature map. 
    
    
    Since $\stEmb$ is geodesically convex, for any point $\aB \in \setA$, there exists a unique geodesic that connects the subspace formed by fixing $\aB$, denoted as $\stEmb_{\aB} \equiv \phi(\aB, \cdot)$ to any other point $g \in \stEmb$. By Lemma \ref{lem:geodesic_and_closeness_phi} in order to maximize the follower's reward $\utlB$, we must find the shortest geodesic distance, $\geoDis(\cdot)$, to $\xi_{\theta_A}$ within $\convexGeoSet$. We express this as,  
    \begin{align}
        \phi(\aB, \bB^*) = \arg \min_{g \in \convexGeoSet} \geoDis(\aB, \bB),
    \end{align}
    Since $\stEmb$ is convex, this minimizer is unique. The reward function $\utlB(\aB, \bB)$ depends on the inner product $\langle \phi(\aB, \bB), \theta_B \rangle$. As this structure is linear with respect to $\phi(\aB, \bB)$, maximizing the reward is equivalent to minimizing the geodesic distance from $\phi(\aB, \bB)$ to the objective vector $\theta$. Since this minimizer is unique by geodesic convexity, the follower's optimal strategy will correspond to this unique solution $\bB^*$ given $\aB$. As there are no alternative solutions for $\phi(\bB^*, \cdot)$ given $\aB$. Because $\phi(\cdot)$ is a bijective mapping, we conclude that any probablistic mapping function must adhere to $\polA(\bB | \aB) \in \{0, 1\}$.

    
\end{proof}

\subsection{Proof of Lemma \ref{lem:intersect_submanifold_AB}} \label{prf:intersect_submanifold_AB}

\textit{\textbf{Intersection of $\IsoPL{\aB}$ and $\IsoPL{\bB}$:} Given a bipartite spherical map from Definition \ref{def:bipartite_sphere_map}, with $\aB$ parameterizing the azimuthal (latitudinal) coordinates, the cardinality of the intersect between $\IsoPL{\aB}$ and $\IsoPL{\bB}$ will be non-empty. That is, $|\IsoPL{\aB} \cap \IsoPL{\bB}| > 0$.}


\begin{proof}
    Given two distinct points $\xi_{\theta_A}$ and $\xi_{\theta_B}$, we define the \emph{isoplane}, $\IsoPL{\aB}$, as the submanifold formed by fixing a subset of spherical coordinates $(\gamma_1^{(A)}, \dots, \gamma_k^{(A)})$, including the azimuthal angle $\nu^{(A)}$, and allowing the remaining coordinates to vary. Similarly, the isoplane at $\xi_{\theta_B}$ is formed by fixing a different subset of spherical coordinates $(\gamma_{k+1}^{(B)}, \dots, \gamma_{D-2}^{(B)})$, while allowing the rest to vary.

 If $\stEmb$ is a compact, orientable, smooth manifold without boundary, and $\vec{X}$ is a smooth vector field on $\stEmb$ with isolated zeros, the \textit{Poincaré-Hopf theorem} states that,
\begin{align}
    \sum_{\mathbf{P} \in \text{Zeroes}(\vec{X})} \text{Index}(\vec{X}, \mathbf{P}) = \chi(\stEmb),
\end{align}
where $\chi(\stEmb)$ is the Euler characteristic of the manifold, and $\text{Index}(\vec{X}, \mathbf{P})$ denotes the index of the vector field at point $\mathbf{P}$. \tocheck{The Euler characteristic of the $D$-dimensional sphere $S^D$ is given by,}
\[
\chi(\mathcal{S}^D) = 
\begin{cases}
2 & \text{if } D \text{ is even}, \\
0 & \text{if } D \text{ is odd}.
\end{cases}
\]
The compactness of $\mathcal{S}^{D-1}$ imposes strong geometric constraints: subspaces or submanifolds (such as isoplanes) embedded within $\mathcal{S}^{D-1}$ must intersect unless they are specifically configured to avoid each other (e.g., in certain degenerate cases of orthogonality). To provide a more fundamental and intuitive analysis, let $\Psi_\theta$ represent the intersection of isoplanar subspaces,

\begin{align}
    \Psi_\theta = \IsoPL{\aB} \cap \IsoPL{\bB}.
\end{align}

First, the compactness of the unit sphere $\mathcal{S}^{D-1}$ implies that any sufficiently dimensional subspaces embedded in the manifold cannot be disjoint. The intersection may be a single point or a higher-dimensional subset, depending on the number of coordinates fixed and the degrees of freedom allowed for the remaining coordinates. In the case where the isoplanes at $\IsoPL{\aB}$ and $\IsoPL{\bB}$ are orthogonal, the fact that the subspaces are embedded in a compact, orientable manifold forces them to intersect. This intersection result is a consequence of the general principles of intersection theory in compact manifolds, which asserts that two subspaces of sufficient dimension within a compact manifold must intersect unless they are orthogonal in all directions. However, since we are working with constrained isoplanes that do not span the entire manifold, even orthogonal subspaces are forced to intersect due to the lack of space for complete disjointness. Therefore,

\begin{align}
    |\Psi_\theta| > 0.
\end{align} 

The following part of the proof relies on ensuring that the criteria are met on the manifold subspace such that the Poincaré-Hopf theorem can applied. 

\paragraph{Construction of Intersecting Subspaces:} \tocheck{ First, suppose we have a spherical manifold which is compact. The azimuthal subspaces, $\IsoPL{\aB}$, can be represented by a vector field with a fixed index, and the latitudinal subspaces, $\IsoPL{\bB}$, could be represented as trajectories from a vector field circumnavigating small circles, around the sphere, also of a fixed index.} Therefore, we can see that any single trajectory of these two vector fields, representing our subspace $\IsoPL{\aB}$ must intersect with at least one other trajectory in $\IsoPL{\bB}$.
    
\paragraph{Extension to Higher Dimensions:} Next, we extend this argument for manifolds of arbitrarily high dimensions, where this is not so trivial to see. First, by the Poincaré-Hopf theorem, the index remains constant when elevating to a higher dimension (as is the case with the D-sphere), preventing any  modification of the spherical geometry, and the corresponding topological properties of vector flows. Next, we argue that since we are working with constrained isoplanes that do not span the entire manifold (i.e. submanifolds or subspaces), even in the case of orthogonal subspaces, the cardinality of the intersecting subspace must be greater than 0 - this is not a direct consequence of the Poincaré-Hopf theorem rather a conclusion which follows.

\end{proof}

\subsection{Orthogonality of Subspaces $\IsoPL{\aB}$ and $\IsoPL{\bB}$}

\begin{lemma} \label{lem:orthogonality_submanifold_AB}
    \textbf{Orthogonality of Subspaces $\IsoPL{\aB}$ and $\IsoPL{\bB}$:} The two submanifolds $\IsoPL{\aB}$ and $\IsoPL{\bB}$, are orthognal to each other within $\stEmb$.
\end{lemma} 

Lemma \ref{lem:orthogonality_submanifold_AB} is proven by isolating and taking the partial derivatives of the cartesian coordinates with respect to their spherical coordinates to obtain tangent vectors. Afterwards, by computing the dot product between these two tangents and demonstrating that it equates to 0, we establish their orthogonality.


\begin{proof}
We consider the spherical manifold $S^{D-1}$, embedded in $\mathbb{R}^D$, where points are parameterized using $D-1$ angular coordinates. These coordinates are composed of latitude-like angles $\nu_1, \dots, \nu_{D-2}$ and a longitude-like angle $\gamma$. The Cartesian coordinates, $\mathbf{x} = [x_1, x_2, \dots, x_D]^\intercal$, of a point on $S^{D-1}$ are expressed as:

\[
\begin{aligned}
x_1 &= \prod_{i=1}^{D-2} \sin(\nu_i) \cos(\gamma), \\
x_2 &= \prod_{i=1}^{D-2} \sin(\nu_i) \sin(\gamma), \\
x_3 &= \prod_{i=1}^{D-3} \sin(\nu_i) \cos(\nu_{D-2}), \\
&\vdots \\
x_{D-1} &= \sin(\nu_1) \cos(\nu_2), \\
x_D &= \cos(\nu_1).
\end{aligned}
\]

We aim to show that the subspaces generated by fixing $\xi_{\theta_A}$, the set of latitude-like angles, and fixing $\xi_{\theta_B}$, the longitude-like angle, are orthogonal. To this end, we compute the tangent vectors of the manifold in the directions of these angular coordinates.

First, we compute the partial derivative of each coordinate with respect to $\gamma$. The coordinates $x_1$ and $x_2$ explicitly depend on $\gamma$, while the other coordinates $x_3, \dots, x_D$ do not. Therefore, we have,

\[
\frac{\partial x_1}{\partial \gamma} = \frac{\partial}{\partial \gamma} \left( \prod_{i=1}^{D-2} \sin(\nu_i) \cos(\gamma) \right) = -\prod_{i=1}^{D-2} \sin(\nu_i) \sin(\gamma),
\]

\[
\frac{\partial x_2}{\partial \gamma} = \frac{\partial}{\partial \gamma} \left( \prod_{i=1}^{D-2} \sin(\nu_i) \sin(\gamma) \right) = \prod_{i=1}^{D-2} \sin(\nu_i) \cos(\gamma),
\]

\[
\frac{\partial x_j}{\partial \gamma} = 0, \quad \forall j \geq 3.
\]

Thus, the complete partial derivative with respect to $\gamma$ is,

\[
\frac{\partial}{\partial \gamma} \left( x_1, x_2, \dots, x_D \right) = \left( -\prod_{i=1}^{D-2} \sin(\nu_i) \sin(\gamma), \ \prod_{i=1}^{D-2} \sin(\nu_i) \cos(\gamma), \ 0, \dots, 0 \right).
\]


Next, we compute the partial derivative of the coordinates with respect to $\nu_1$. This affects all coordinates $x_1, x_2, \dots, x_D$. Specifically,

\[
\frac{\partial x_1}{\partial \nu_1} = \frac{\partial}{\partial \nu_1} \left( \prod_{i=1}^{D-2} \sin(\nu_i) \cos(\gamma) \right) = \cos(\nu_1) \prod_{i=2}^{D-2} \sin(\nu_i) \cos(\gamma),
\]

\[
\frac{\partial x_2}{\partial \nu_1} = \frac{\partial}{\partial \nu_1} \left( \prod_{i=1}^{D-2} \sin(\nu_i) \sin(\gamma) \right) = \cos(\nu_1) \prod_{i=2}^{D-2} \sin(\nu_i) \sin(\gamma),
\]

\[
\frac{\partial x_3}{\partial \nu_1} = \frac{\partial}{\partial \nu_1} \left( \prod_{i=1}^{D-3} \sin(\nu_i) \cos(\nu_{D-2}) \right) = \cos(\nu_1) \prod_{i=2}^{D-3} \sin(\nu_i) \cos(\nu_{D-2}),
\]

\[
\frac{\partial x_4}{\partial \nu_1} = \cdots = \frac{\partial x_D}{\partial \nu_1} = -\sin(\nu_1).
\]

Thus, the complete partial derivative with respect to $\nu_1$ is:

\[
\frac{\partial}{\partial \nu_1} \left( x_1, x_2, \dots, x_D \right) = \left( \cos(\nu_1) \prod_{i=2}^{D-2} \sin(\nu_i) \cos(\gamma), \ \cos(\nu_1) \prod_{i=2}^{D-2} \sin(\nu_i) \sin(\gamma), \ -\sin(\nu_1), \dots\right).
\]


\paragraph{Dot Product of Tangent Vectors:} To prove orthogonality of the subspaces spanned by these vectors, we compute the dot product of the tangent vectors $\frac{\partial}{\partial \gamma}$ and $\frac{\partial}{\partial \nu_1}$. The dot product is given by,

\[
\frac{\partial}{\partial \gamma} \cdot \frac{\partial}{\partial \nu_1} = \left( -\prod_{i=1}^{D-2} \sin(\nu_i) \sin(\gamma) \right) \cdot \left( \cos(\nu_1) \prod_{i=2}^{D-2} \sin(\nu_i) \cos(\gamma) \right) + \dots,
\]

which simplifies to zero, as the terms corresponding to the components in $x_1$, $x_2$, and $x_3$ do not align. Consequently, we have,

\[
\frac{\partial}{\partial \gamma} \cdot \frac{\partial}{\partial \nu_1} = 0.
\]

Since the dot product of the tangent vectors is zero, the subspaces spanned by fixing $A$ and fixing $B$ are orthogonal at every point on $S^{D-1}$. This orthogonality arises from the fact that the angular coordinates for latitude and longitude parameterize independent directions in the tangent space of the spherical manifold. Thus, we conclude that the subspaces resulting from fixing $A$ and $B$ are mutually orthogonal.

\end{proof}

\subsection{Proof of Lemma \ref{lem:leader_pure_strategy_spherical}} \label{prf:leader_pure_strategy_spherical}

\textbf{Pure Strategy of the Leader:} \textit{Given a spherical manifold, $\stEmb$, and isoplanar subspace, $\IsoPL{\aB}$ and $\IsoPL{\bB}$ for the longitudinal and lattitudinal subspaces respectively, the optimal strategy of the leader is that of a pure strategy, that is, $\polA^*(\aB) \in \{0, 1\}$.}

\begin{proof}
Let $\mathcal{S}^{D-1} \subset \mathbb{R}^D$ be the unit sphere embedded in $D$-dimensional Euclidean space. Consider two distinct points $\xi_{\theta_A}$ and $\xi_{\theta_B}$ on the manifold, each with spherical coordinates $(\gamma_1^{(A)}, \gamma_2^{(A)}, \dots, \gamma_{D-2}^{(A)}, \nu^{(A)})$ and $(\gamma_1^{(B)}, \gamma_2^{(B)}, \dots, \gamma_{D-2}^{(B)}, \nu^{(B)})$, respectively. We aim to demonstrate that the isoplanes formed by fixing half of the spherical coordinates at $\xi_{\theta_A}$ and $\xi_{\theta_B}$ must intersect, and this intersection $\Psi_\theta$ is a singleton. By Lemma \ref{lem:intersect_submanifold_AB} we infer that $\IsoPL{\aB}$ and $\IsoPL{\bB}$ must form a non-empty intersect in $\stEmb$. Follower by Lemma \ref{lem:orthogonality_submanifold_AB}, $\IsoPL{\aB}$ and $\IsoPL{\bB}$ are orthogonal to each other in $\stEmb$.




\paragraph{Singleton Intersection due to Orthogonality:} Consider the isoplanes formed by fixing the angular coordinates $\xi_{\theta_A}$ (latitude-like) and $\xi_{\theta_B}$ (longitude-like) on the unit sphere $S^{D-1}$. These isoplanes correspond to submanifolds of the sphere, which are defined by holding certain angular coordinates constant while allowing others to vary. In the special case where the isoplanes at $\xi_{\theta_A}$ and $\xi_{\theta_B}$ are orthogonal, we argue that the intersection set of these submanifolds is reduced to a single element (singleton). Let $\mathbf{P}$ be the point where the isoplanes associated with fixed $\xi_{\theta_A}$ and $\xi_{\theta_B}$ intersect. The tangent space at $\mathbf{P}$, denoted as $T_{\mathbf{P}} S^{D-1}$, consists of vectors tangent to the sphere at $\mathbf{P}$.


The isoplane formed by fixing $\xi_{\theta_A}$ corresponds to a submanifold $\IsoPL{\aB}$ whose tangent space at $\mathbf{p}$, denoted $T_{\mathbf{p}} \IsoPL{\aB}$, is spanned by the partial derivatives with respect to the longitude-like angular coordinates $\gamma_i$. Similarly, the isoplane formed by fixing $\xi_{\theta_B}$ corresponds to a submanifold $\IsoPL{\bB}$, and the tangent space $T_{\mathbf{p}} \IsoPL{\bB}$ is spanned by the partial derivatives with respect to the latitude-like angular coordinates $\nu_j$. Orthogonality between the isoplanes at $\xi_{\theta_A}$ and $\xi_{\theta_B}$ implies that the tangent spaces $T_{\mathbf{p}} \IsoPL{\aB}$ and $T_{\mathbf{p}} \IsoPL{\bB}$ are mutually orthogonal. This means that the dot product of any vector from $T_{\mathbf{p}} \IsoPL{\aB}$ with any vector from $T_{\mathbf{p}} \IsoPL{\bB}$ is zero:
\[
\mathbf{v}_A \cdot \mathbf{v}_B = 0, \quad \forall \mathbf{v}_A \in T_{\mathbf{p}} \IsoPL{\aB}, \quad \mathbf{v}_B \in T_{\mathbf{p}} \IsoPL{\bB}.
\]
Geometrically, this implies that the submanifolds $\IsoPL{\aB}$ and $\IsoPL{\bB}$ intersect at a right angle at $\mathbf{P}$. Since the submanifolds are orthogonal, no other points of intersection can occur, and the intersection set is reduced to the single point $\mathbf{P}$. Therefore,
\begin{align}
    |\Psi_\theta| = 1. \label{eq:cardinality_intersec_subspaces}
\end{align}

\paragraph{Minimal Geodesic Distance from $\Psi_\theta$:} Let $\Psi_\theta = (x_1^{(\text{int})}, x_2^{(\text{int})}, \dots, x_D^{(\text{int})})$ be the unique intersection point of the two isoplanes. Now, we consider the geodesic distance from this intersection point to any other point on the sphere. The geodesic distance between two points $\mathbf{P}_1 = (x_1^{(1)}, x_2^{(1)}, \dots, x_D^{(1)})$ and $\mathbf{P}_2 = (x_1^{(2)}, x_2^{(2)}, \dots, x_D^{(2)})$ on the unit sphere is given by,
\[
\geoDis(\mathbf{P}_1, \mathbf{P}_2) = \arccos(\innerP{\mathbf{P}_1,  \mathbf{P}_2}).
\]
At the intersection point $\Psi_\theta$, the geodesic distance is minimized, thus, 
\[
\mathbf{P}_1 = \Psi_\theta \implies \geoDis(\mathbf{P}_1, \Psi_\theta) = 0.
\]
Suppose we move away from $\Psi_\theta$ along either the longitude isoplanes (by changing $x_1$) or the latitude isoplanes (by changing $x_2, x_3, \dots, x_D$). Any such deviation implies a change in the dot product $\mathbf{P}_1 \cdot \mathbf{P}_2$, which results in an increase in the geodesic distance. Specifically, if we move along the longitude isoplanes, we are changing $x_1$, while the other coordinates remain constant, resulting in a decrease in the dot product. Similarly, if we move along the latitude isoplanes, we are changing $x_2, x_3, \dots, x_D$, again causing a decrease in the dot product. Since the geodesic distance is a monotonically increasing function of the angular separation, any deviation from $\Psi_\theta$ leads to an increase in the geodesic distance,
\[
\geoDis( \mathbf{P}_2, \mathbf{P}_1) > \geoDis(\mathbf{P}_1, \Psi_\theta) = 0.
\]
Thus, any deviation from the intersection point of the longitude and latitude isoplaness must result in an increase in the geodesic distance, $\geoDis(\cdot)$. By Lemma \ref{lem:geodesic_and_closeness_phi}, this increase in the geodesic distance will decrease the expected reward $\utlA$. As $|\Psi_\theta| = 1$ from Eq. \eqref{eq:cardinality_intersec_subspaces}, then no optimal mixed strategies exist for the leader, and thus, $\polA^*(\aB) \in \{0, 1\}$.

\end{proof}

\subsection{Conversion of Cartesian Uncertainty to Spherical}

\begin{lemma} \label{lem:geodesic_uncertainty_ball}
    Given two points \(\theta_A, \tilde{\theta}_A \in \mathbb{R}^D\), denoting points on the surface of a unit spherical manifold, the uncertainty in Cartesian coordinates expressed as $\|\theta_A- \tilde{\theta}_A\| < \uncerBall{t}$ can be expressed as uncertainty in geodesic distance as $\geoDis(A, \tilde{\theta}_A) < \cos^{-1}\left(1 - \frac{\uncerBall{t}^2}{2}\right)$.
\end{lemma}

\begin{proof}
    Given two points \(\theta_A, \tilde{\theta}_A \in \mathbb{R}^D\), with \(\|A\| = \|\tilde{\theta}_A\| = 1\), denoting points on the surface of a unit sphere, the uncertainty in Cartesian coordinates is expressed as:
 $
    \|\theta_A- \tilde{\theta}_A\| < \uncerBall{t}, 
  $
    where \(\uncerBall{t} \in \mathbb{R}^+\) is the uncertainty bound. We aim to translate this uncertainty into spherical coordinates.
    
    \paragraph{Cartesian Coordinates on the Unit Sphere:} In \( \mathbb{R}^D \), the spherical coordinates of a point \(\theta_A\) on the surface of the unit sphere can be represented as:
    \[
    \theta_A^{(1)} = \cos(\nu_1),
    \]
    \[
    \theta_A^{(2)} = \sin(\nu_1)\cos(\nu_2),
    \]
    \[
    \theta_A^{(3)} = \sin(\nu_1)\sin(\nu_2)\cos(\nu_3),
    \]
    \[
    \vdots
    \]
    \[
    \theta_A^{(D-1)} = \sin(\nu_1) \sin(\nu_2) \ldots \sin(\nu_{D-2}) \cos(\gamma),
    \]
    \[
    \theta_A^{(D)} = \sin(\nu_1) \sin(\nu_2) \ldots \sin(\nu_{D-2}) \sin(\gamma),
    \]
    where \(\nu_1, \nu_2, \ldots, \nu_{D-2}\) represent the latitude angles, and \(\gamma\) represents the longitude angle. Similarly, the point \(\tilde{\theta}_A\) can be written in terms of spherical angles \(\nu'_1, \nu'_2, \ldots, \gamma'\).
    
    \paragraph{Uncertainty in Cartesian Coordinates:}
    
    The uncertainty in Cartesian space is given by:
    \[
    \|\theta_A- \tilde{\theta}_A\|^2 = (\theta_A^{(1)} - \tilde{\theta}_A^{(1)})^2 + (\theta_A^{(2)} - \tilde{\theta}_A^{(2)})^2 + \ldots + (\theta_A^{(D)} - \tilde{\theta}_A^{(D)})^2 < \uncerBall{t}^2.
    \]
    However, it is more efficient to relate this uncertainty directly to spherical angular distance.
    
    \paragraph{Spherical Angular Distance:}
    
    The squared Euclidean distance between two points \(\theta_A\) and \(\tilde{\theta}_A\) on the surface of the unit sphere is related to their angular distance \(\nu\) by the spherical law of cosines:
  $
    \|\theta_A- \tilde{\theta}_A\|^2 = 2(1 - \cos(\nu)),
   $
    where \(\nu\) is the angular distance between the two points, and \(\cos(\nu)\) is given by:
    \[
    \cos(\nu) = \cos(\nu_1)\cos(\nu'_1) + \sin(\nu_1)\sin(\nu'_1) \Big( \cos(\nu_2)\cos(\nu'_2) + \sin(\nu_2)\sin(\nu'_2) \cdots \Big).
    \]
    This expression provides the exact angular distance between points \(\theta_A\) and \(\tilde{\theta}_A\) on the unit sphere.
    
    \paragraph{Uncertainty in Spherical Coordinates:}
    
    The inequality \(\|\theta_A- \tilde{\theta}_A\| < \uncerBall{t}\) implies that the angular distance \(\nu\) between the two points satisfies:
    \[
    2(1 - \cos(\nu)) < \uncerBall{t}^2\implies
    \cos(\nu) > 1 - \frac{\uncerBall{t}^2}{2}.
    \]
    Since \(\cos(\nu)\) ranges from 1 (when \(\theta_A = \tilde{\theta}_A\)) to -1 (for antipodal points), the angular distance \(\nu\) is bounded by:
    \[
    \nu < \cos^{-1}\left(1 - \frac{\uncerBall{t}^2}{2}\right).
    \]
    This inequality describes the exact spherical uncertainty region. Thus, the uncertainty \(\|\theta_A- \tilde{\theta}_A\| < \uncerBall{t}\) in Cartesian space corresponds to an angular uncertainty \(\nu < \cos^{-1}\left(1 - \frac{\uncerBall{t}^2}{2}\right)\) on the unit sphere.
\end{proof}
 
\subsection{Distance Preserving Orthogonal Projection:}

\begin{lemma} \label{lem:distance_preserving_ortho_proj}
    Consider a unit sphere \( S^{D-1} \subset \mathbb{R}^D \). Given a point \(\theta_A \in S^{D-1}\) and a geodesic ball \( B_J \subset S^{D-1} \) centered at \(\theta_A\), we are interested in the behaviour of this ball under orthogonal projection onto a subspace of \(\mathbb{R}^D\). Specifically, we aim to rigorously show that the diameter of the orthogonally projected ball does not exceed the diameter of the original geodesic ball. 
\end{lemma}

\begin{proof}
    
    \textbf{Geodesic Uncertainty Balls:} Let \( \theta_A, \tilde{\theta}_A \in \mathbb{R}^D \) be two points on the unit sphere, i.e., \( \| \theta_A\| = \| \tilde{\theta_A}\| = 1 \), and let the geodesic distance between \(\theta_A\) and \(\tilde{\theta}_A\) be denoted by \(\gamma(\theta_A, \tilde{\theta}_A)\). The geodesic distance between any two points on \(S^{D-1}\) is given by,
    \[
    \gamma(\theta_A, \tilde{\theta}_A) = \arccos(\theta_A \cdot \tilde{\theta}_A),
    \]
    where \( \theta_A \cdot \tilde{\theta}_A \) is the Euclidean dot product between \(\theta_A\) and \(\tilde{\theta}_A\). A geodesic ball \( B_J(\theta_A) \) centered at \(\theta_A\) with radius \(J\) is defined as the set of points on the unit sphere such that their geodesic distance from \(\theta_A\) is less than or equal to \(J\):
    \[
    B_J(\theta_A) = \{ \xi_{\theta_A} \in S^{D-1} \mid \gamma(\theta_A, \tilde{\theta}_A) \leq J \}.
    \]
    We are particularly interested in the case where \( J \leq \arccos\left(1 - \frac{\uncerBall{t}^2}{2}\right) \), where \(\uncerBall{t}\) is a positive value corresponding to the uncertainty radius in the Euclidean distance.
    
    \textbf{Orthogonal Projection and Geodesic Distance:} Given a subspace \( V \subset \mathbb{R}^D \), let \( P_V: \mathbb{R}^D \to V \) denote the orthogonal projection onto \(V\). For any points \( \theta_A, \tilde{\theta}_A \in \mathbb{R}^D \), the Euclidean distance between their projections is bounded by:
    \[
    \| P_V(\theta_A) - P_V(\tilde{\theta}_A) \| \leq \| \theta_A - \tilde{\theta}_A \|.
    \]
    Since the geodesic distance on the unit sphere is a measure of arc length between points, it follows that the geodesic distance between two points is non-increasing under orthogonal projection. We aim to show that the diameter of the projected geodesic ball onto the subspace \(V\) does not exceed the diameter of the original ball.
    
    \textbf{Diameter of a Geodesic Ball:} The diameter of a set \( S \subset S^{D-1} \) is defined as the greatest geodesic distance between any two points in \(S\):
    \[
    \text{diam}(S) = \sup_{x, y \in S} \gamma(x, y).
    \]
    For a geodesic ball \( B_J(\theta_A) \), the maximum geodesic distance occurs between two antipodal points on the boundary of the ball. Therefore, the diameter of the geodesic ball is:
    \[
    \text{diam}(B_J(\theta_A)) = 2J.
    \]
    In particular, for \( J = \arccos\left(1 - \frac{\uncerBall{t}^2}{2}\right) \), we have:
    \[
    \text{diam}(B_J(\theta_A)) = 2 \arccos\left(1 - \frac{\uncerBall{t}^2}{2}\right).
    \]
\end{proof}

\subsection{Diameter Preserving Orthogonal Projection}

We now formalize the behaviour of the geodesic ball under orthogonal projection.

\begin{lemma} \label{lem:diameter_preserving_ortho_proj}
Let \( B_J(\theta_A) \) be a geodesic ball of radius \(J \leq \arccos\left(1 - \frac{\uncerBall{t}^2}{2}\right) \) on the unit sphere \( S^{D-1} \subset \mathbb{R}^D \). Let \( V \subset \mathbb{R}^D \) be a subspace, and let \( P_V: \mathbb{R}^D \to V \) be the orthogonal projection onto \(V\). Then, the diameter of the orthogonally projected ball \( P_V(B_J(\theta_A)) \) satisfies:
\[
\text{diam}(P_V(B_J(\theta_A))) \leq \text{diam}(B_J(\theta_A)) = 2J.
\]
\end{lemma}

\begin{proof}
    Consider two points \( \theta_A, \tilde{\theta}_A \in B_J(\theta_A) \). By the definition of a geodesic ball, we know that:
    \[
    \gamma(\theta_A, \tilde{\theta}_A) \leq 2J.
    \]
    Next, project \( \theta_A \) and \( \tilde{\theta}_A \) orthogonally onto the subspace \(V\), yielding the points \( P_V(\theta_A) \) and \( P_V(\tilde{\theta}_A) \). Since orthogonal projection reduces or preserves Euclidean distances, we have:
    \[
    \| P_V(\theta_A) - P_V(\tilde{\theta}_A) \| \leq \| \theta_A - \tilde{\theta}_A \|.
    \]
    Moreover, since the geodesic distance between points on the sphere is a function of their Euclidean distance, it follows that the geodesic distance between the projected points \( P_V(\theta_A) \) and \( P_V(\tilde{\theta}_A) \) is also bounded by:
    \[
    \gamma(P_V(\theta_A), P_V(\tilde{\theta}_A)) \leq \gamma(\theta_A, \tilde{\theta}_A).
    \]
    Thus, for all pairs \( \theta_A, \tilde{\theta}_A \in B_J(\theta_A) \), we have:
    \[
    \gamma(P_V(\theta_A), P_V(\tilde{\theta}_A)) \leq 2J.
    \]
    This shows that the diameter of the projected geodesic ball \( P_V(B_J(\theta_A)) \) is at most \( 2J \), i.e.,
    \[
    \text{diam}(P_V(B_J(\theta_A))) \leq \text{diam}(B_J(\theta_A)) = 2J.
    \]
\end{proof}

\subsection{Proof of Theorem \ref{thm:sphere_man_iso_reg}} \label{prf:sphere_man_iso_reg}

\textit{\textbf{Isoplane Stackelberg Regret:} For D-dimensional spherical manifolds embedded in $\mathbb{R}^D$ space, where $\phi(\aB, \cdot)$ generates an isoplanes $\IsoPL{\aB}$, and the linear relationship to the reward function in Eq. \eqref{eq:A_reward_inner_prod} and Eq. \eqref{eq:B_reward_inner_prod} holds, the simple regret, defined in Eq. \eqref{eq:simple_regret_defn}, of any learning algorithm with uncertainty parameter uncertainty  $\uncerBall{t}$, refer to in Eq. \eqref{eq:param_uncertainty}, is bounded by $\mathcal{O}(2\arccos(1-\uncerBall{t}^2/2))$.}

\begin{proof}
    The proof of Theorem \ref{thm:sphere_man_iso_reg} hinges on the aforementioned arguments in Lemma \ref{lem:geodesic_uncertainty_ball}, Lemma \ref{lem:distance_preserving_ortho_proj}, and Lemma \ref{lem:diameter_preserving_ortho_proj} sequentially, but in the context of parameter estimation.

    First, Lemma \ref{lem:geodesic_uncertainty_ball} argues that one can transform a confidence bound $|\theta_A - \hat{\theta}_A| \leq \uncerBall{t}$ into a confidence bound on geodesic distance $\geoDis(\theta_A, \hat{\theta}_A) \leq \cos^{-1}\left(1 - \frac{\uncerBall{t}^2}{2}\right)$. Let us denote this as the geodesic confidence ball $\Ball_{\geoDis}(\theta^*, \uncerBall{t})$. Nevertheless, due to the separation of subspaces $\IsoPL{\aB}$ and $\IsoPL{\bB}$, we must find the projection of $\Ball_{\geoDis}(\theta^*, \uncerBall{t})$ onto $\IsoPL{\bB}$ such that we can obtain a diameter measure on the new intersecting subspace $\IsoPL{\aB} \cap \IsoPL{\bB}$. Next, Lemma \ref{lem:distance_preserving_ortho_proj} argues that geodesic distances will either be preserved or reduced when making a projection to an orthogonal subspace $\IsoPL{\bB}$, the orthogonality of this subspace was previously established in Lemma \ref{lem:orthogonality_submanifold_AB}. Thereafter, Lemma \ref{lem:diameter_preserving_ortho_proj} specifies that the maximum diameter of this new confidence ball $\Ball'_{\geoDis}(\theta^*, \uncerBall{t})$ that is projected onto $\IsoPL{\bB}$ is confined to a maximum diameter of $2\cos^{-1}\left(1 - \frac{\uncerBall{t}^2}{2}\right)$.

    Thus, this constitutes the best and worst possible outcomes due to misspecification in accordance with the formulation in Eq. \eqref{eq:min_H} and Eq. \eqref{eq:max_H}, denoted as $\bar{\mathcal{H}}(\theta_A^*,t) - \underline{\mathcal{H}}(\theta_A^*,t)$, also expressed in Eq. \eqref{eq:simple_regret_defn}, which upper bounds the simple regret.
    
\end{proof}

\section{Neural Flow Architectural Specifications} \label{sec:nn_arch_details}

We present the mathematical foundations of the normalizing flow architecture used to model spherical mappings. Our method combines a spherical coordinate transformation with normalizing flows to provide an invertible mapping between input features and a latent space, with applications to tasks requiring smooth transformations on a manifold. 

\textbf{Mapping to a Spherical Manifold:} The transformation from Cartesian coordinates to spherical coordinates is used to map input features onto an $D$-dimensional spherical manifold. In order to freely select $D$, we can optionally apply a series of bijective stereographic projections from the ambient space in order to increase the dimension of the embedding space. We define two heads in the neural network input, the head from A specifically controls the azimuthal spherical coordinate and additional coordinates, and  the head from B specifically controls other coordinates. The output sizes of the neural network that transforms the inputs are $\floor{\frac{D-1}{2}}+1$ for A and $\floor{\frac{D-1}{2}}$ for B. The conversion from Cartesian coordinates to spherical coordinates, and vice-versa, $\mathbf{x} \in \mathbb{R}^D$, is defined in Appendix \ref{sec:sphere_cart_conv_algo}. 



\textbf{Affine Coupling Layers:} A normalizing flow consists of a series of invertible transformations, including affine coupling layers, which divide the input into two parts and transform one part conditioned on the other. Let the input be $\mathbf{x} = [\mathbf{x}_1, \mathbf{x}_2]$, where $\mathbf{x}_1$ and $\mathbf{x}_2$ are disjoint subsets of the input. The affine coupling transformation is defined as,
\begin{align}
    \mathbf{y}_1 = \mathbf{x}_1, \quad  \mathbf{y}_2 = \mathbf{x}_2 \odot \exp(s(\mathbf{x}_1)) + t(\mathbf{x}_1),
\end{align}
where $\odot$ denotes element-wise multiplication, and $s(\mathbf{x}_1)$ and $t(\mathbf{x}_1)$ are the scaling and translation functions, respectively, parameterized by a neural network. The inverse of this transformation is straightforward:
\begin{align}
    \mathbf{x}_1 = \mathbf{y}_1, \quad
    \mathbf{x}_2 = \left( \mathbf{y}_2 - t(\mathbf{y}_1) \right) \odot \exp(-s(\mathbf{y}_1)).
\end{align}
This transformation is invertible by design, making it suitable for use in flow-based models.

\textbf{Log Determinant of the Jacobian:} The log-likelihood calculation requires computing the log determinant of the Jacobian matrix for the transformation. For the affine coupling layer, the Jacobian matrix is triangular, and the log determinant is simply the sum of the scaling terms:
\begin{equation}
    \log \big| \det \frac{\partial \mathbf{y}}{\partial \mathbf{x}} \big| = \sum_{i} s(\mathbf{x}_1).
\end{equation}
This term contributes to the overall log probability during training.

\textbf{Normalizing Flow Forward Transform:} A normalizing flow is constructed by stacking several affine coupling layers and random permutation layers. Let $\mathbf{x} \in \mathbb{R}^d$ be the input, and $\mathbf{z} \in \mathbb{R}^d$ be the transformed latent variable after $L$ layers of flow. Each layer applies a transformation $f_l$ such that: $\mathbf{z}^{(l+1)} = f_l(\mathbf{z}^{(l)}),$ where $f_l$ represents either an affine coupling transformation or a random permutation. After $L$ layers, the final output is denoted as $\mathbf{z} = \mathbf{z}^{(L)}$. The forward transformation can thus be written as:
\begin{equation}
    \mathbf{z}, \log \det J = f_{\text{flow}}(\mathbf{x}),
\end{equation}
where $\log \det J$ is the log determinant of the Jacobian matrix for the entire flow. 

To compute the log-likelihood of the input $\mathbf{x}$, we map it to the latent space $\mathbf{z}$ under the flow transformation. The probability of $\mathbf{x}$ is computed as:
\begin{equation}
    p(\mathbf{x}) = p(\mathbf{z}) \big| \det \frac{\partial \mathbf{z}}{\partial \mathbf{x}} \big|,
\end{equation}
where $p(\mathbf{z})$ is the probability of $\mathbf{z}$ under the base distribution (typically a standard normal distribution):
$
    p(\mathbf{z}) = \mathcal{N}(\mathbf{z}; 0, I).
$
The log probability is then given by:
\begin{equation}
    \log p(\mathbf{x}) = \log p(\mathbf{z}) + \log \big| \det \frac{\partial \mathbf{z}}{\partial \mathbf{x}} \big|.
\end{equation}
\textbf{Inverse Transform:} The invertibility of the flow allows for both density estimation and sampling. To sample from the model, we draw samples $\mathbf{z} \sim \mathcal{N}(0, I)$ from the base distribution and apply the inverse transformation:
$
    \mathbf{x} = f_{\text{flow}}^{-1}(\mathbf{z}).
$ 
Each affine coupling layer and random permutation is applied in reverse order to recover the original inputs.

\textbf{Random Permutation Layer:} The random permutation layer permutes the features of the input vector to ensure that different parts of the input are transformed at each layer. Let $\mathbf{x} \in \mathbb{R}^d$ be the input, and let $P$ be a permutation matrix. The permutation transformation is defined as:
$    \mathbf{x}' = P \mathbf{x}.$

Since permutation matrices are orthogonal, the Jacobian determinant of this transformation is always $1$, and it does not contribute to the log determinant calculation.

\begin{table}[h!]
\centering
\begin{tabular}{@{}>{\raggedright\arraybackslash}p{4.3cm} 
                >{\raggedright\arraybackslash}p{6cm} 
                >{\raggedright\arraybackslash}p{3cm}@{}}
\toprule
\textbf{Layer}         & \textbf{Description}                                     & \textbf{Output Size}          \\ 
\midrule
\texttt{Input Head A}         & Input head A.                                           & $N_B \times |\setA|$         \\ 
\midrule
\texttt{Input Head B}         & Input head B.                                           & $N_B \times |\setB|$         \\ 
\midrule
\texttt{Input Features}         & Input features.                                           & $N_B \times D$         \\ 
\midrule
\texttt{Affine Coupling Layer} & No. of Affine Coupling layers.                         & $N_B \times 64$         \\ 
\midrule
\texttt{fc\_A1 Hidden Dim.}    & Number of hidden dimensions in first fully connected layer A.       & $B \times 1024$ \\ 
\midrule
\texttt{fc\_B1 Hidden Dim.}    & Number of hidden dimensions in first fully connected layer B.       & $B \times 1024$ \\ 
\midrule
\texttt{Hidden Dim.}    & No. of hidden layers for A and B.       & $N_B \times 16$ \\ 
\midrule
\texttt{fc\_A1 Final Layer Dim.}    & Number of hidden dimensions in final layer A.      & $N_B \times \Big( \floor{\frac{D-1}{2}} + 1 \Big)$\\ 
\midrule
\texttt{fc\_B1 Final Layer Dim.}    & Number of hidden dimensions in final layer B.      & $N_B \times \Big( \floor{\frac{D-1}{2}} \Big)$\\ 
\midrule
\texttt{Output}        & Output features after flow transformation.                & $N_B \times D$         \\ 
\bottomrule
\end{tabular}
\caption{Normalizing Flows Neural Architecture Specifications.}
\end{table}

\textbf{Overview:} In summary, the normalizing flow architecture combines spherical mapping, affine coupling transformations, and random permutations to form a powerful framework for invertible transformations. The model leverages the flexibility of normalizing flows to map inputs to a spherical manifold, enabling efficient density estimation and sampling from a base Gaussian distribution.

\begin{table}[H]
    \centering
    \begin{tabular}{||c c||} 
     \hline
     Parameter & Value \\ [0.5ex] 
     \hline\hline
     $N_B$ Batch Size & 2048 \\
     $\alpha_N $ (Negative Log Liklihood Loss Coef.)  & 0.5 \\
     $\alpha_R $ (Repulsion Loss Coef.)  & 1.0 \\
     $\alpha_P $ (Perturb. Loss Coef.)  & 0.5 \\
     $\alpha_L $ (Lipschitz Loss Coef.)  & 1.5 \\
     No. Epochs & 20,000 \\
     $\alpha_{LR}$ (Learning Rate) & 0.05 \\
     $C_L$ (Lipschitz Constant) & 0.5 \\
     \hline
    \end{tabular}
    \caption{Hyper parameters used for normalizing neural flow network training. }
    \label{table:hyperparam}
\end{table}

\clearpage

\section{Visualizations}

\subsection{Computational Results of Isoplane Behaviour} \label{sec:isoplanes_computational_viz}

\begin{figure}[!htb]
\minipage{0.47\textwidth}
  \includegraphics[width=60mm]{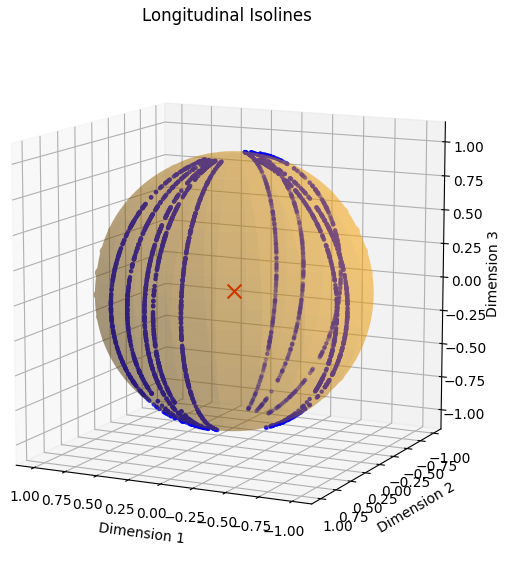}
  \caption*{\textbf{Longitudinal Isolines: } Visualization of longitudinal isolines generated by the normalizing neural flow network.}
\endminipage\hfill
\minipage{0.47\textwidth}
  \includegraphics[width=60mm]{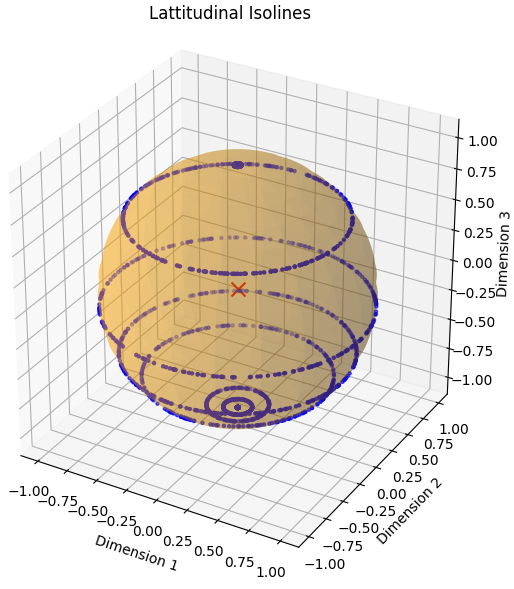}
  \caption*{\textbf{Latitudinal Isoplanes:} Visualization of lattitudinal isolines generated by the normalizing neural flow network.}
\endminipage\hfill
\caption{Formation of isolines (or isoplanes in higher dimensions) forming on the spherical manifold $\stEmb$ as we fix $\aB$ and vary $\bB$ (longitudinal), and fix $\bB$ and vary $\aB$ (lattitudinal). } \label{fig:isoplanes_line_sphere_regret}
\end{figure}

\subsection{Visualization of Ambient Space to Stackelberg Embedding mapping} \label{sec:ambient_to_stemb_viz}

\begin{figure}[!htb]
\minipage{0.47\textwidth}
  \includegraphics[width=70mm]{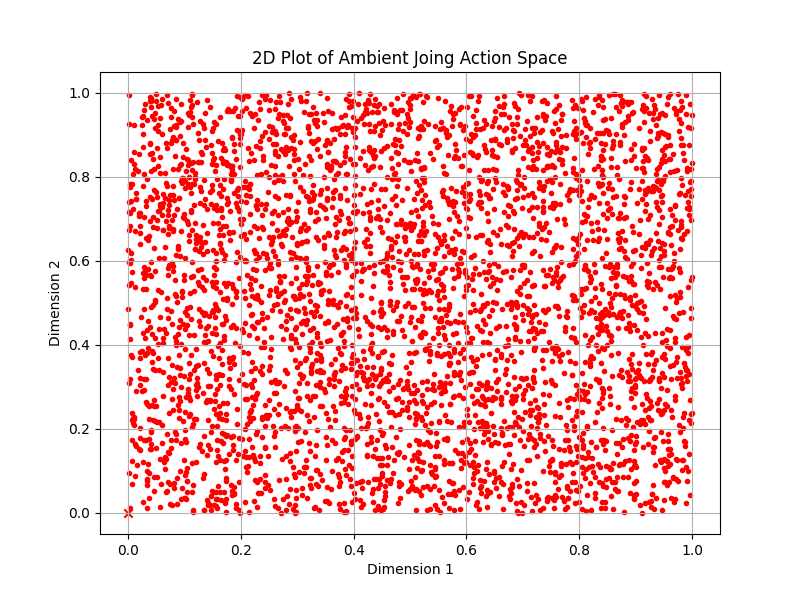}
  \caption*{\textbf{Ambient Space: } Visualization uniformly distributed data in ambient space (native joint action space).}
\endminipage\hfill
\minipage{0.47\textwidth}
  \includegraphics[width=60mm]{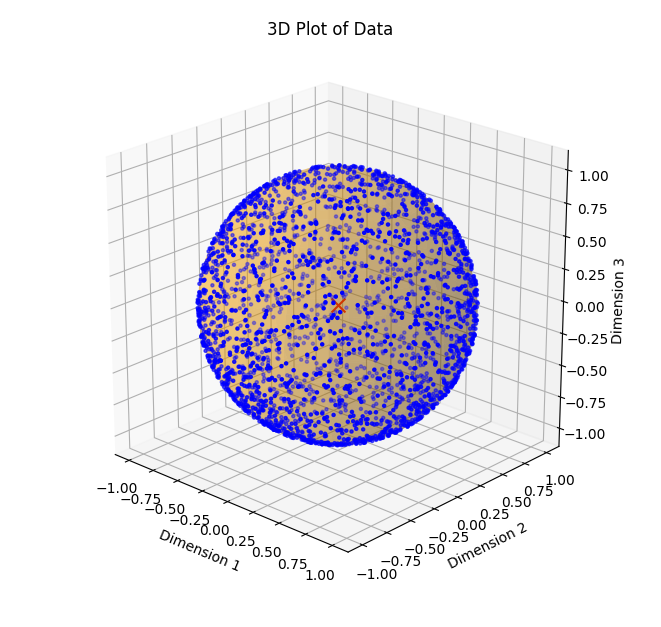}
  \caption*{\textbf{Stackelberg Embedding $\stEmb$:} Visualization of data mapped from the ambient joint action space to a learned Stackelberg embedding $\stEmb$.} 
\endminipage\hfill
\caption{For purely visualization purposes, we present the mapping from the ambient data 3D space to a 2-Sphere manifold $\stEmb$. In addition to being sufficiently biijective, the Lipschitz and spreading properties from Table \ref{tab:stack-emb-dynamics} are maximized via normalizing neural flows.} \label{fig:stackelberg_embedding_mapping_2fig_viz}
\end{figure}

\subsection{Poor Examples of Stackelberg Embeddings} \label{sec:bad_stackelberg_embeddings_viz}

\begin{figure}[H]
\minipage{0.47\textwidth}
  \includegraphics[width=60mm]{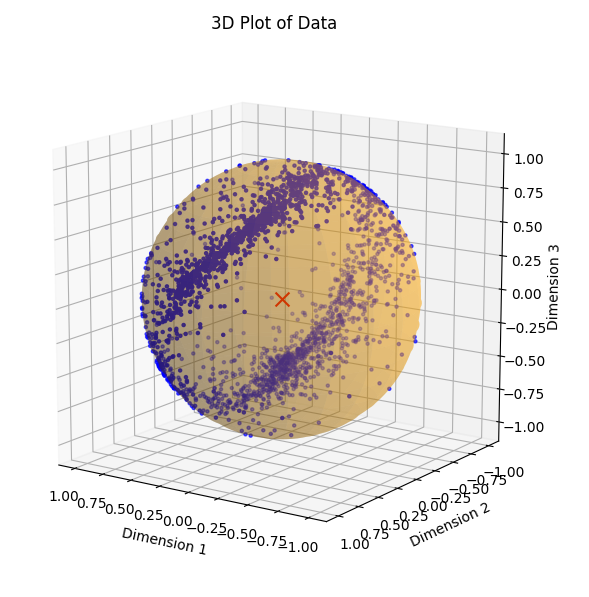}
  \caption*{\textbf{Arbitrary Projection: } The ambient data passes through the normalizing flow network which contains untrained model weights resulting in arbitrary conversion of Cartesian to spherical coordinates.}
\endminipage\hfill
\minipage{0.47\textwidth}
  \includegraphics[width=60mm]{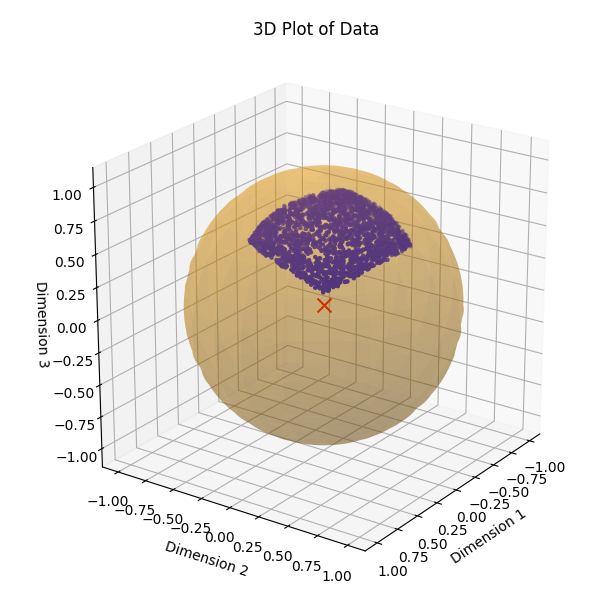}
  \caption*{\textbf{Affine Map:} The ambient data is subject to an arbitrary affine map followed by a projection to $\stEmb$.} 
\endminipage\hfill
\caption{We present two examples of when improper and/or simplistic Euclidean transformations are not sufficient to fulfill the requirements outlined in Table \ref{tab:stack-emb-dynamics}. The transformed data lies in arbitrarily clustered regions, and/or are insufficiently spread across the 2-sphere despite the ambient data being uniformly distributed. Moreover, sufficient bijection may not always be possible.} \label{fig:bad_stackelberg_embedding_2fig_viz}
\end{figure}

\clearpage

\section{Algorithms}

\subsection{Mapping between Spherical and Cartesian Coordinates} \label{sec:sphere_cart_conv_algo}


\begin{algorithm}[h!]
\caption{Spherical to Cartesian Conversion in $n$-Dimensions}\label{alg:spherical_to_cartesian}
\begin{algorithmic}[1]
    \STATE \textbf{Input:} $r$ (radius), $\nu$ (Spherical coordinates $D - 1$ dimensions.)
    \STATE \textbf{Output:} Cartesian coordinates $p = [x_1, x_2, \dots, x_D]$
    \STATE $x_1 \gets r \cdot \cos(\nu_1)$
    \FOR{$i = 2$ to $D-1$}
        \STATE $x_i \gets r \cdot \sin(\nu_1) \cdot \sin(\nu_2) \dots \cdot \sin(\nu_{i-1}) \cdot \cos(\nu_i)$
    \ENDFOR
    \STATE $x_n \gets r \cdot \sin(\nu_1) \cdot \dots \cdot \sin(\nu_{D-1})$
    \RETURN $[x_1, x_2, \dots, x_D]$
    \end{algorithmic}
\end{algorithm}


\begin{algorithm}[h!]
\caption{Cartesian to Spherical Conversion in $n$-Dimensions}\label{alg:cartesian_to_spherical}
\begin{algorithmic}[1]
    \STATE \textbf{Input:} Cartesian coordinates $p = [x_1, x_2, \dots, x_D]$
    \STATE \textbf{Output:} $r$ (radius), $\nu = [\nu_1, \nu_2, \dots, \nu_{D-1}]$ (Spherical coordinates $D - 1$ dimensions.)
    \STATE $r \gets \sqrt{x_1^2 + x_2^2 + \dots + x_D^2}$ \COMMENT{Compute the radius}
    \STATE $\nu_1 \gets \arccos\left( \frac{x_1}{r} \right)$ \COMMENT{First spherical angle}
    \FOR{$i = 2$ to $D-1$}
        \STATE $\nu_i \gets \arctan2\left( \sqrt{x_1^2 + x_2^2 + \dots + x_{i}^2}, x_{i+1} \right)$ \COMMENT{Spherical angles for $i = 2$ to $D-1$}
    \ENDFOR
    \STATE \textbf{Return} $r$, $\nu = [\nu_1, \nu_2, \dots, \nu_{D-1}]$
\end{algorithmic}
\end{algorithm}

\clearpage

\section{Experimental Results}

\subsection{$\mathbb{R}^1$ Stackelberg Game} \label{sec:r1_game_details}

\textbf{Problem Setup:} We consider a Stackelberg game with a leader \(\theta_A\) and a follower \(B\), both operating in continuous action spaces \(a, b \in \mathbb{R}^1\). The leader chooses an action $a$, and the follower responds by choosing an action $b$ based on the leader's decision. The reward functions for both players are linear in structure but include nonlinear components to model real-world constraints and interactions.

\textbf{Leader’s Reward Function:} The leader’s reward function \( \mu_A(a, b) \) is defined as follows:
\begin{align}
    \mu_A(a, b) = \theta_1 a + \theta_2 \log(1 + b^2) - \frac{\theta_3}{2} a^2 + \epsilon, \quad \epsilon \in \mathcal{N}(0, \sigma) 
\end{align}
where $\theta_1, \theta_2 > 0$ are weight parameters that control the trade-off between the leader’s direct action \(\theta_A\) and the follower’s response \(b\), \( \log(1 + b^2) \) introduces nonlinearity with respect to the follower’s action \(b\), and \( -\frac{\theta_3}{2} a^2 \) is a quadratic penalty on large leader actions to avoid extreme behaviour by the leader.

\textbf{Follower’s Reward Function:} The follower’s reward function \( \mu_B(\aB, \bB) \) is given by:
\begin{align}
    \mu_B(a, b) = \alpha_1 (-b^2) + \alpha_2 a b + \epsilon, \quad \epsilon \in \mathcal{N}(0, \sigma) 
\end{align}

where, $\alpha_1, \alpha_2 > 0$ are parameters that determine the influence of the follower’s own action \(b\) and the leader’s action \(\theta_A\) on the follower’s reward, $-b^2$ represents a concave cost function for the follower, preferring smaller values of $b$, and $ab$ introduces an interaction term between the leader’s action and the follower’s action.

\textbf{Follower's Best Response:} The follower maximizes their reward function \( \mu_B(\aB, \bB) \) by choosing \(b\) given \(a\). To determine the follower's best response \( \bR{a} \), we compute the first-order condition with respect to \(b\):
\begin{align}
\frac{\partial \expeC[\mu_B(\aB, \bB)]}{\partial b} = -2 \alpha_1 b + \alpha_2 a = 0
\end{align}
Solving for \(b\), the follower's best response is:
\begin{align}
\bR{a} = \frac{\alpha_2 a}{2 \alpha_1}
\end{align}

\textbf{Leader’s Optimization Problem:} Given that the follower’s best response is \( \bR{a} = \frac{\alpha_2 a}{2 \alpha_1} \), the leader maximizes their reward function \( \mu_A(a, \bR{a}) \) as,
\begin{align}
\expeC[\mu_A(a, \bR{a})] = \theta_1 a + \theta_2 \log\left(1 + \left(\frac{\alpha_2 a}{2 \alpha_1}\right)^2\right) - \frac{\theta_3}{2} a^2.
\end{align}
This results in the following optimization problem for the leader,
\begin{align}
\max_a \, \left( \theta_1 a + \theta_2 \log\left(1 + \frac{\alpha_2^2 a^2}{4 \alpha_1^2}\right) - \frac{\theta_3}{2} a^2 \right).
\end{align}

\textbf{Non-Trivial Solution for the Leader:} To solve for the leader’s optimal action \(a^*\), we take the derivative of the leader's reward function with respect to \(a\) and set it equal to zero,
\begin{align}
\frac{d}{da} \left( \theta_1 a + \theta_2 \log\left(1 + \frac{\alpha_2^2 a^2}{4 \alpha_1^2}\right) - \frac{\theta_3}{2} a^2 \right) &= 0 \\
\theta_1 - \theta_3 a + \theta_2 \cdot \frac{2 \cdot \left(\frac{\alpha_2 a}{2 \alpha_1}\right) \cdot \left(\frac{\alpha_2}{2 \alpha_1}\right)}{1 + \frac{\alpha_2^2 a^2}{4 \alpha_1^2}} &= 0
\end{align}
Which simplifies to,
\begin{align}
\theta_1 - \theta_3 a + \frac{\theta_2 \cdot \frac{\alpha_2^2 a}{\alpha_1^2}}{1 + \frac{\alpha_2^2 a^2}{4 \alpha_1^2}} = 0.
\end{align}
This equation has no simple closed-form analytical solution and must be solved numerically. The interplay between the nonlinear logarithmic term and the quadratic penalty introduces complexity into the leader's optimization, making the optimal value of \(a^*\) non-trivial.

\raggedbottom

\subsubsection{$\mathbb{R}^1$ Stackelberg Game} \label{sec:r2_game_results}

\begin{figure}[t]
\minipage{0.43\textwidth}
  \includegraphics[width=\linewidth]{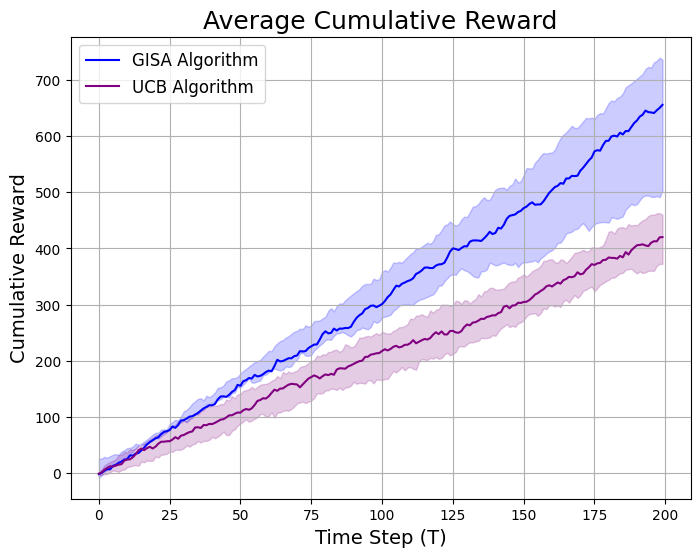}
\endminipage\hfill
\minipage{0.43\textwidth}
  \includegraphics[width=\linewidth]{figures/s1_reg_ex1_v3.png}
\endminipage\hfill
\caption*{Parameters: $\theta_1 = 4.0, \theta_2 = 1.0, \theta_3 = 0.9, \alpha_1 = 1.0, \alpha_2 = 2.0, \sigma=6.0$.}
\minipage{0.43\textwidth}
  \includegraphics[width=\linewidth]{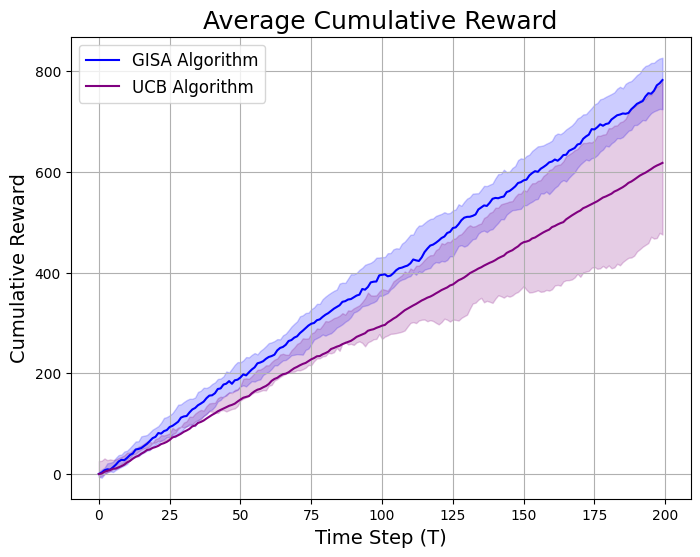}
\endminipage\hfill
\minipage{0.43\textwidth}
  \includegraphics[width=\linewidth]{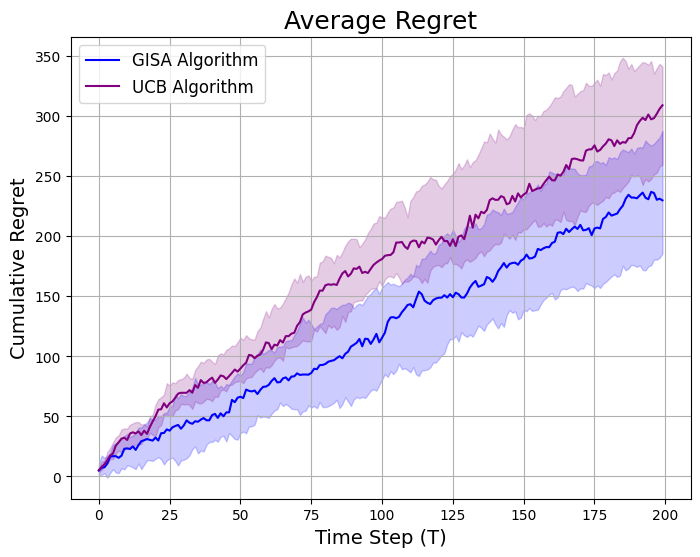}
\endminipage\hfill
\caption*{Parameters: $\theta_1 = 4.0, \theta_2 = 2.0, \theta_3 = 0.9, \alpha_1 = 4.0, \alpha_2 = 2.0, \sigma=6.0$.}
\minipage{0.43\textwidth}
  \includegraphics[width=\linewidth]{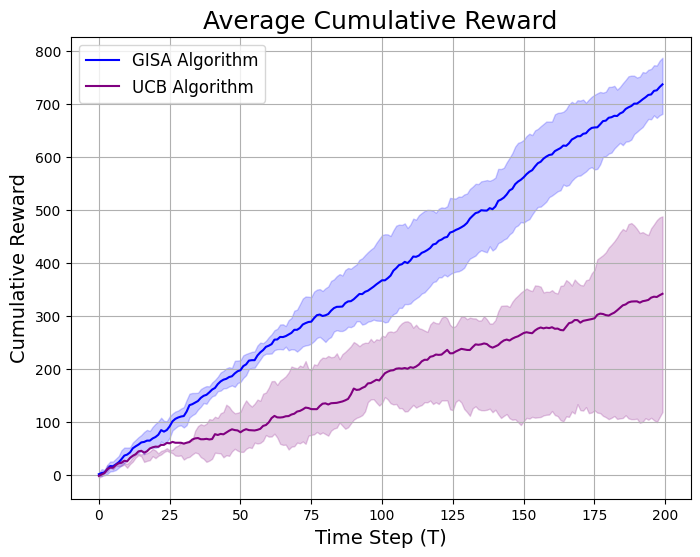}
\endminipage\hfill
\minipage{0.43\textwidth}
  \includegraphics[width=\linewidth]{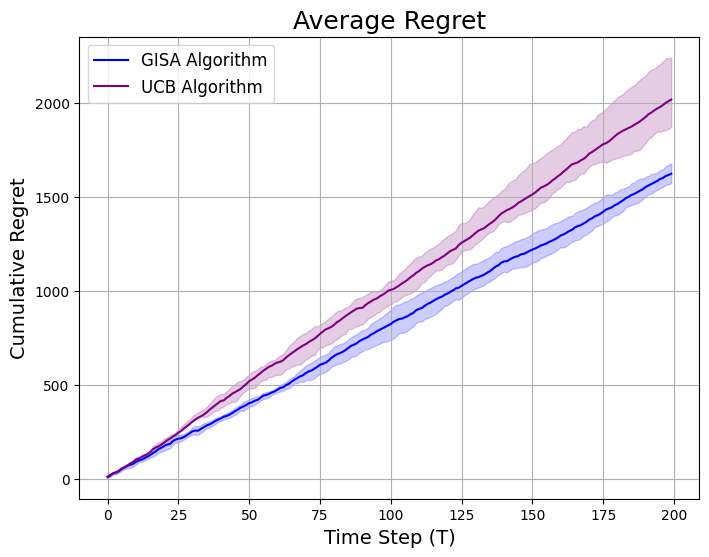}
\endminipage\hfill
\caption*{Parameters: $\theta_1 = 3.0, \theta_2 = 3.0, \theta_3 = 0.9, \alpha_1 = 1.0, \alpha_2 = 2.0, \sigma=4.0$.}
\caption{Mean values are calculated over 1,000 trials, with shaded regions representing confidence intervals, all of which fall within the first quartile. $\stEmb$ dimension is set to $D = 10$.} \label{fig:r1_game_summary_plots}
\end{figure}

\clearpage

\subsection{The Newsvendor Pricing Game Specifcations (NPG)} \label{sec:npg-appendix}

We model the two learning agents in a \textit{Newsvendor pricing game}, involving a supplier $A$ and a retailer $B$. The leader, a supplier, is learning to dynamically price the product for the follower, a retailer, aiming to maximize his reward. To achieve this, the follower adheres to classical Newsvendor theory, which involves finding the optimal order quantity given a known demand distribution before the realization of the demand.

\textbf{Rules of the Newsvendor Pricing Game:} We explicitly denote $a \equiv \mathbf{a} \in \mathbb{R}^1$, and $\mathbf{b} \equiv [b, p]^\intercal \in \mathbb{R}^2$. Where $a$ denotes wholesale price from the supplier firm, $p$ and $b$ denote the retail price and order amount of the retail firm.

\vspace{-0.2cm}

\begin{enumerate}[itemsep=2pt, parsep=2pt]
    \setlength\itemsep{-0.1em}
    \item The supplier selects wholesale price $a$, and provides it to the retailer. \label{item:supplier-contract-price}
    \item Given wholesale cost $a$, the retailer reacts with his best response $[b, p]^\intercal$, consisting of retail price $p$, and order amount $b$.
    \item As the retailer determines the optimal order amount $b$, he pays $\mathcal{G}_A(a, b) = ab$ to the supplier.
    \item At time $t$, nature draws demand $d^t \sim d_\rho(p)$, and it is revealed to the retailer.
    \item The retailer makes a profit of $\mathcal{G}_B(a, b) = p \ \min \{ d^t, b \} - ab$. \label{item:retail-profit}
    \item Steps \ref{item:supplier-contract-price} to \ref{item:retail-profit} are repeated for $t \in 1 ... T$ iterations.
\end{enumerate}

\begin{figure}[ht!]
    \centering
    \tikzset{every picture/.style={line width=0.75pt}} 
    
        \begin{tikzpicture}[x=0.75pt,y=0.75pt,yscale=-0.6,xscale=0.6]
        
        \draw   (60,214) .. controls (60,200.19) and (71.19,189) .. (85,189) -- (95,189) .. controls (108.81,189) and (120,200.19) .. (120,214) -- (120,241) .. controls (120,241) and (120,241) .. (120,241) -- (60,241) .. controls (60,241) and (60,241) .. (60,241) -- cycle ;
        \draw   (65,163) .. controls (65,149.19) and (76.19,138) .. (90,138) .. controls (103.81,138) and (115,149.19) .. (115,163) .. controls (115,176.81) and (103.81,188) .. (90,188) .. controls (76.19,188) and (65,176.81) .. (65,163) -- cycle ;
        \draw    (141.2,190.6) -- (228.2,190.6) ;
        \draw [shift={(230.2,190.6)}, rotate = 180] [color={rgb, 255:red, 0; green, 0; blue, 0 }  ][line width=0.75]    (10.93,-3.29) .. controls (6.95,-1.4) and (3.31,-0.3) .. (0,0) .. controls (3.31,0.3) and (6.95,1.4) .. (10.93,3.29)   ;
        \draw   (261,215) .. controls (261,201.19) and (272.19,190) .. (286,190) -- (296,190) .. controls (309.81,190) and (321,201.19) .. (321,215) -- (321,242) .. controls (321,242) and (321,242) .. (321,242) -- (261,242) .. controls (261,242) and (261,242) .. (261,242) -- cycle ;
        \draw   (266,164) .. controls (266,150.19) and (277.19,139) .. (291,139) .. controls (304.81,139) and (316,150.19) .. (316,164) .. controls (316,177.81) and (304.81,189) .. (291,189) .. controls (277.19,189) and (266,177.81) .. (266,164) -- cycle ;
        \draw    (231.2,210.6) -- (212.4,210.6) -- (143.2,210.6) ;
        \draw [shift={(141.2,210.6)}, rotate = 360] [color={rgb, 255:red, 0; green, 0; blue, 0 }  ][line width=0.75]    (10.93,-3.29) .. controls (6.95,-1.4) and (3.31,-0.3) .. (0,0) .. controls (3.31,0.3) and (6.95,1.4) .. (10.93,3.29)   ;
        \draw    (342.2,190.6) -- (429.2,190.6) ;
        \draw [shift={(431.2,190.6)}, rotate = 180] [color={rgb, 255:red, 0; green, 0; blue, 0 }  ][line width=0.75]    (10.93,-3.29) .. controls (6.95,-1.4) and (3.31,-0.3) .. (0,0) .. controls (3.31,0.3) and (6.95,1.4) .. (10.93,3.29)   ;
        \draw   (457,179.79) .. controls (457,172.38) and (463.01,166.37) .. (470.42,166.37) -- (475.78,166.37) .. controls (483.19,166.37) and (489.2,172.38) .. (489.2,179.79) -- (489.2,194.28) .. controls (489.2,194.28) and (489.2,194.28) .. (489.2,194.28) -- (457,194.28) .. controls (457,194.28) and (457,194.28) .. (457,194.28) -- cycle ;
        \draw   (459.68,152.42) .. controls (459.68,145.01) and (465.69,139) .. (473.1,139) .. controls (480.51,139) and (486.52,145.01) .. (486.52,152.42) .. controls (486.52,159.83) and (480.51,165.83) .. (473.1,165.83) .. controls (465.69,165.83) and (459.68,159.83) .. (459.68,152.42) -- cycle ;
        \draw   (487,235.79) .. controls (487,228.38) and (493.01,222.37) .. (500.42,222.37) -- (505.78,222.37) .. controls (513.19,222.37) and (519.2,228.38) .. (519.2,235.79) -- (519.2,250.28) .. controls (519.2,250.28) and (519.2,250.28) .. (519.2,250.28) -- (487,250.28) .. controls (487,250.28) and (487,250.28) .. (487,250.28) -- cycle ;
        \draw   (489.68,208.42) .. controls (489.68,201.01) and (495.69,195) .. (503.1,195) .. controls (510.51,195) and (516.52,201.01) .. (516.52,208.42) .. controls (516.52,215.83) and (510.51,221.83) .. (503.1,221.83) .. controls (495.69,221.83) and (489.68,215.83) .. (489.68,208.42) -- cycle ;
        \draw   (520,179.79) .. controls (520,172.38) and (526.01,166.37) .. (533.42,166.37) -- (538.78,166.37) .. controls (546.19,166.37) and (552.2,172.38) .. (552.2,179.79) -- (552.2,194.28) .. controls (552.2,194.28) and (552.2,194.28) .. (552.2,194.28) -- (520,194.28) .. controls (520,194.28) and (520,194.28) .. (520,194.28) -- cycle ;
        \draw   (522.68,152.42) .. controls (522.68,145.01) and (528.69,139) .. (536.1,139) .. controls (543.51,139) and (549.52,145.01) .. (549.52,152.42) .. controls (549.52,159.83) and (543.51,165.83) .. (536.1,165.83) .. controls (528.69,165.83) and (522.68,159.83) .. (522.68,152.42) -- cycle ;
        \draw    (430.2,209.6) -- (342.2,209.6) ;
        \draw [shift={(340.2,209.6)}, rotate = 360] [color={rgb, 255:red, 0; green, 0; blue, 0 }  ][line width=0.75]    (10.93,-3.29) .. controls (6.95,-1.4) and (3.31,-0.3) .. (0,0) .. controls (3.31,0.3) and (6.95,1.4) .. (10.93,3.29)   ;
        
        \draw (176,165) node [anchor=north west][inner sep=0.75pt]   [align=left] {$a$};
        \draw (171,214) node [anchor=north west][inner sep=0.75pt]   [align=left] {$b_a$};
        \draw (376,164) node [anchor=north west][inner sep=0.75pt]   [align=left] {$p_a$};
        \draw (371,215) node [anchor=north west][inner sep=0.75pt]   [align=left] {$d(p_a)$};
        \draw (183,252) node [anchor=north west][inner sep=0.75pt]   [align=left] {$\mathcal{G}_B(p_a, b_a) = p_a \min \{ d(p_a), b_a \}$};
        \draw (20,251) node [anchor=north west][inner sep=0.75pt]   [align=left] {$\mathcal{G}_A(a) = a b_a$};
        \draw (15,110) node [anchor=north west][inner sep=0.75pt]   [align=left] {Leader (Supplier)};
        \draw (209,108) node [anchor=north west][inner sep=0.75pt]   [align=left] {Follower (Retailer)};
        \draw (464,108) node [anchor=north west][inner sep=0.75pt]   [align=left] {Market};
    \end{tikzpicture}
    \caption{\textbf{The Newsvendor Pricing Game.} From \citep{liu:2024_stacknews_adt}, in this Stackelberg game, there a logistics network between a supplier (leader) and retailer (follower), where utility functions are not necessarily supermodular, the supplier issues a wholesale price $a$, and the retailer issues a purchase quantity $b$, and a retail price $p$ in response.} \label{fig:supplier-retailer-game}
\end{figure}

\vspace{-0.4cm}

\textbf{Demand Function:} Stochastic demand is represented in Eq. \ref{eq:additive-exp-demand}, which is governed by a linear additive demand function $\Gamma_\rho(p)$ representing the expected demand, $\mathbb{E}[d(p)]$, as a function of $p$ in Eq. \ref{eq:additive-exp-demand}. The demand function is governed by parameters $\rho$.

\begin{align}
    \Gamma_\rho(p) &= \max \{0, \rho_0 - \rho_1 p \}, \quad \rho_0 \geq 0, \ \rho_1 \geq 0 \label{eq:demand-theta-func} \\
    d_\rho(p) &= \Gamma_\rho(p) + \epsilon, \quad \epsilon \in \mathcal{N}(0, \sigma) \label{eq:additive-exp-demand} 
\end{align}

This problem combines the problem of the \textit{price-setting Newsvendor} \citep{petruzzi:1999newsv} \citep{arrow:1951newsboy}, with that of a bilateral Stackelberg game under imperfect information. Even in the scenario of perfect information, the \textit{price-setting Newsvendor} has no closed-form solution, therefore no exact solution to the Stackelberg equilibrium. We apply the algorithm from \citep{liu:2024_stacknews_adt} to learn a Stackelberg equilibrium under a \textit{risk-free pricing} strategy assumption, and apply Algorithm \ref{alg:se-newsv} from \citep{liu:2024_stacknews_adt}, with scaling factor set as $\kappa = 1.0$, as a baseline against Algorithm \ref{alg:gisa} (GISA).

\begin{algorithm}[h!]
\caption{Learning Algorithm for Newsvendor Pricing Game from \citep{liu:2024_stacknews_adt}}\label{alg:se-newsv}
\begin{algorithmic}[1]
    \FOR{$t \in 1 \dots T$}
        \STATE Leader and follower estimates a confidence interval $\uncerBall{t}$ from available data.
        \STATE $\mathcal{H}(\rho) = \ \hat{\rho}_0/\hat{\rho}_1$.
        \STATE Leader plays action $a$, where $a = \underset{a \in \mathcal{A}, \rho \in \mathcal{C}^t} {\mathrm{argmax}} \ a F^{-1}_{\bar{\rho}_a} \Big( 1 - \frac{2a}{ \mathcal{H}(\rho) + a} \Big)$ from Eq. (3.8) in \citep{liu:2024_stacknews_adt}.
        \STATE Follower sets price $p = (\mathcal{H}(\rho) + a)/2$.
        \STATE Follower estimates their optimistic parameters $\bar{\rho}_a$, and best response $\bar{b}_a$ from Eq. (3.4) and (3.5a) respectively in \citep{liu:2024_stacknews_adt}.
        \STATE Leader obtains reward, $\mathcal{G}_A = ab$.
        \STATE Follower obtains reward, $\mathcal{G}_B = p \min \{ b, d(p) \}$.
    \ENDFOR
\end{algorithmic}
\end{algorithm}

\raggedbottom


\subsubsection{NPG Results} \label{sec:npg_results}

\begin{figure}[t]
\minipage{0.43\textwidth}
  \includegraphics[width=\linewidth]{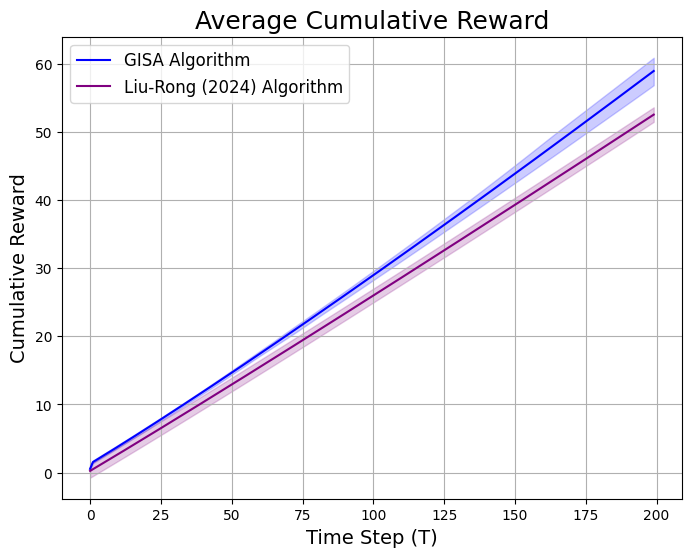}
\endminipage\hfill
\minipage{0.43\textwidth}
  \includegraphics[width=\linewidth]{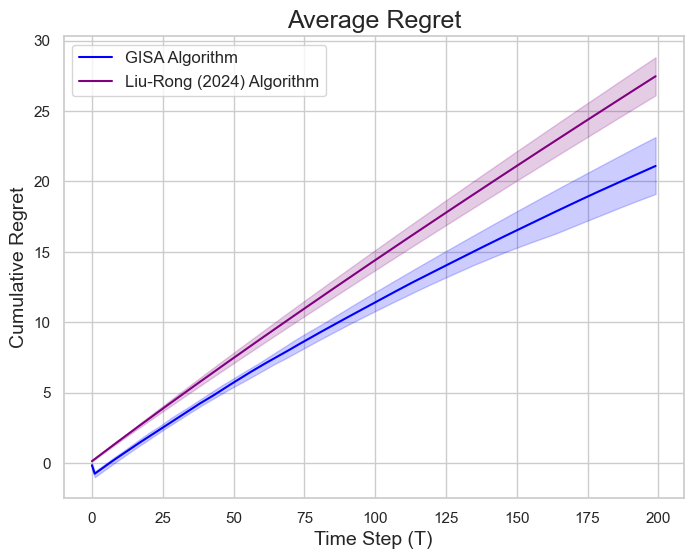}
\endminipage\hfill
\caption*{Parameters: $\rho_0 = 1, \rho_1 = -0.1, \sigma=0.1$.}
\minipage{0.43\textwidth}
  \includegraphics[width=\linewidth]{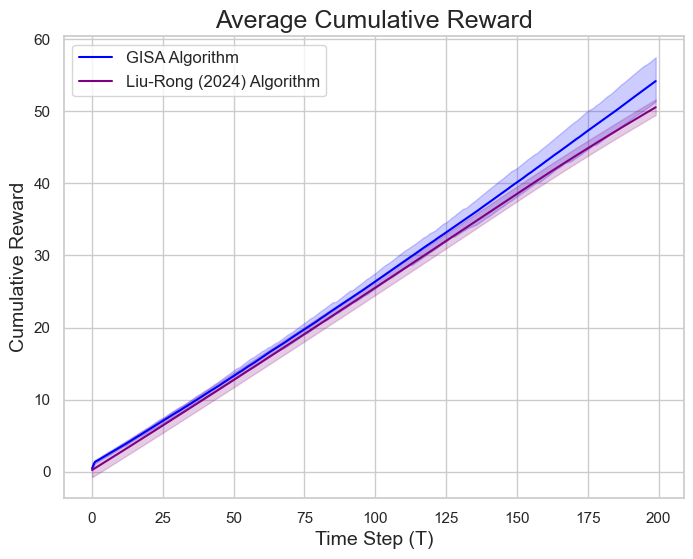}
\endminipage\hfill
\minipage{0.43\textwidth}
  \includegraphics[width=\linewidth]{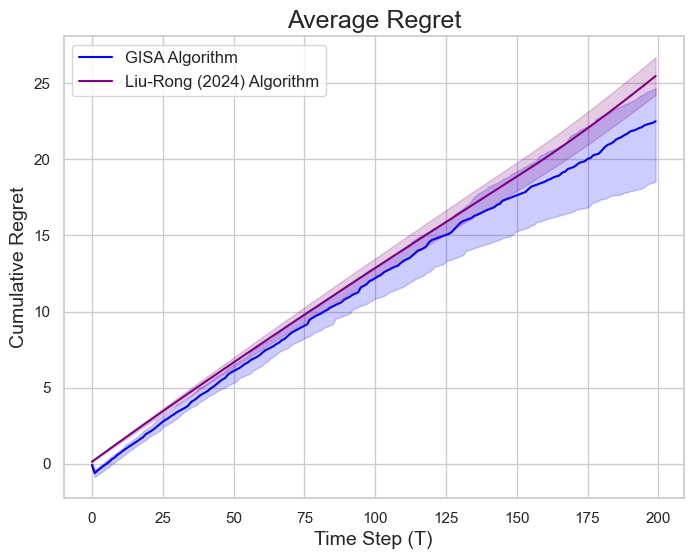}
\endminipage\hfill
\caption*{Parameters: $\rho_0 = 1, \rho_1 = -0.5, \sigma=0.1$.}
\minipage{0.43\textwidth}
  \includegraphics[width=\linewidth]{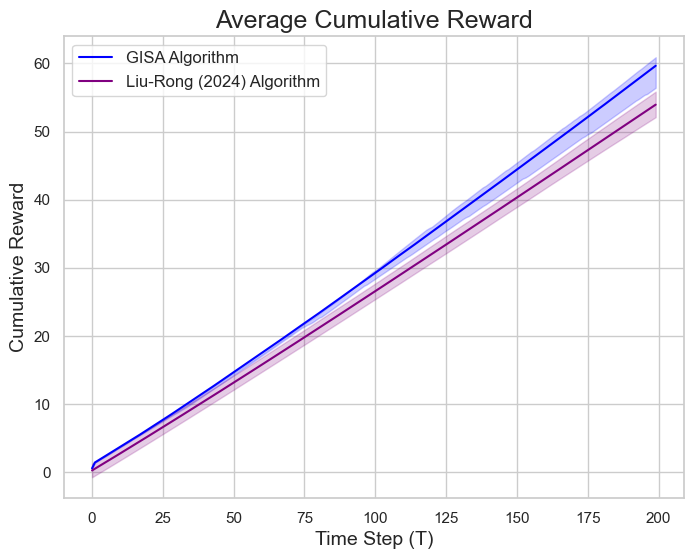}
\endminipage\hfill
\minipage{0.43\textwidth}
  \includegraphics[width=\linewidth]{figures/npg_reg_ex3_v2.png}
\endminipage\hfill
\caption*{Parameters: $\rho_0 = 1, \rho_1 = -0.3, \sigma=0.5$.}
\caption{Mean values are calculated over 1,000 trials, with shaded regions representing confidence intervals, all of which fall within the first quartile. $\stEmb$ dimension is set to $D = 6$.} \label{fig:loss_plot}
\end{figure}

\clearpage

\subsection{Multi-Dimensional Stackelberg Game (SSG)} \label{sec:multi-dim-ssg}

We consider a two-player Stackelberg game where the leader $ A $ and the follower $ B $ choose their actions from a shared action space $ \mathbb{R}^n $. The leader chooses an action $ \aB \in \mathbb{R}^n $, anticipating the follower’s response $ \bB \in \mathbb{R}^n $, where $n=5$. Both players' rewards are influenced by a combination of the difference in their actions and quadratic penalties on their individual actions. The problem is constrained by weighted $ L_1 $-norm bounds on both $ \aB $ and $ \bB $, which limit the magnitude of their respective actions.

The leader’s reward function $ \mu_A $ is defined as:
\begin{align}
\mu_A(\aB, \bB) = \theta_A^\top (\aB - \bB) - \theta_A^\top f(\aB) + \epsilon, \quad \epsilon \in \mathcal{N}(0, \sigma) 
\end{align}
where:
\begin{itemize}
    \item $ \aB \in \mathbb{R}^n $ is the leader’s action,
    \item $ \bB \in \mathbb{R}^n $ is the follower’s action,
    \item $ \theta_A \in \mathbb{R}^n $ is a weight vector for the leader,
    \item $ f(\aB)$ is the quadratic penalty function applied elementwise, such that $ f(\aB) = [\aB_1^2, \aB_2^2, \dots, \aB_n^2] $.
\end{itemize}
The leader seeks to maximize $ \mu_A(a, b) $ by selecting $ \aB $, knowing that the follower will respond optimally.

The follower’s reward function $ \mu_B $ is defined as:
\begin{align}
\mu_B(\aB, \bB) = \theta_B^\top (\aB - \bB) - \theta_B^\top g(\bB)
\end{align}
where:
\begin{itemize}
    \item $ \aB \in \mathbb{R}^n $ is the leader’s action,
    \item $ \bB \in \mathbb{R}^n $ is the follower’s action,
    \item $ \theta_B \in \mathbb{R}^n $ is a weight vector for the follower,
    \item $ g(\bB) $ is the quadratic penalty function applied elementwise, such that $ g(\bB) = [\bB_1^2, \bB_2^2, \dots, b_n^2] $.
\end{itemize}
The follower seeks to maximize $ \mu_B(\aB, \bB) $ by choosing $ \bB $, given the leader’s action $ \aB $.

Both players are subject to weighted $ L_1 $-norm constraints on their actions:
\begin{align}
\sum_{i=1}^{n} |\theta_{A,i} a_i| \leq C_A \quad \text{for the leader} \\
\sum_{i=1}^{n} |\theta_{B,i} b_i| \leq C_B \quad \text{for the follower}
\end{align}
where $ C_A $ and $ C_B $ are constants that limit the magnitude of the actions $ \aB $ and $ \bB $, respectively, and $ \theta_{A,i} $, $ \theta_{B,i} $ are the elements of $ \theta_A $ and $ \theta_B $.

\textbf{Follower’s Optimization Problem (Best Response):} Given the leader’s action $ \aB $, the follower solves the following optimization problem:
\begin{align}
b^*(\aB) = \arg\max_b \left( \theta_B^\top (\aB - \bB) - \theta_B^\top g(\bB) \right)
\end{align}
subject to:
\begin{align}
\sum_{i=1}^{n} |\theta_{B,i} b_i| \leq C_B
\end{align}
This is a quadratic optimization problem due to the quadratic penalty $ g(\bB) $, and the constraint enforces that the weighted $ L_1 $-norm of the follower's action does not exceed $ C_B $.

\textbf{Leader’s Optimization Problem:} Given the follower’s best response $ \bB^*(\aB) $, the leader solves the following optimization problem:
\begin{align}
a^* = \arg\max_a \left( \theta_A^\top (\aB - \bB^*(\aB)) - \theta_A^\top f(\aB) \right)
\end{align}
subject to:
\begin{align}
\sum_{i=1}^{n} |\theta_{A,i} a_i| \leq C_A
\end{align}
This is also a quadratic optimization problem due to the quadratic penalty $ f(a) $, and the constraint enforces that the weighted $ L_1 $-norm of the leader’s action does not exceed $ C_A $.

\textbf{Stackelberg equilibrium: }The Stackelberg equilibrium is reached when:
\begin{align}
a^* = \arg\max_{\aB} \left( \theta_A^\top (\aB - \bB^*(\aB)) - \theta_A^\top f(\aB) \right), \quad b^*(\aB) = \arg\max_{\bB} \left( \theta_B^\top (\aB - \bB) - \theta_B^\top g(\bB) \right)
\end{align}
subject to the respective $ L_1 $-norm constraints. At equilibrium, the leader chooses $ \aB^* $ that maximizes their reward given the follower’s optimal response $ \bB^*(\aB) $, and the follower chooses $ \bB^*(\aB) $ that maximizes their reward given the leader’s action.

\raggedbottom

\subsubsection{SSG Empirical Results} \label{sec:ssg_empirical}

\begin{figure}[t]
\minipage{0.43\textwidth}
  \includegraphics[width=\linewidth]{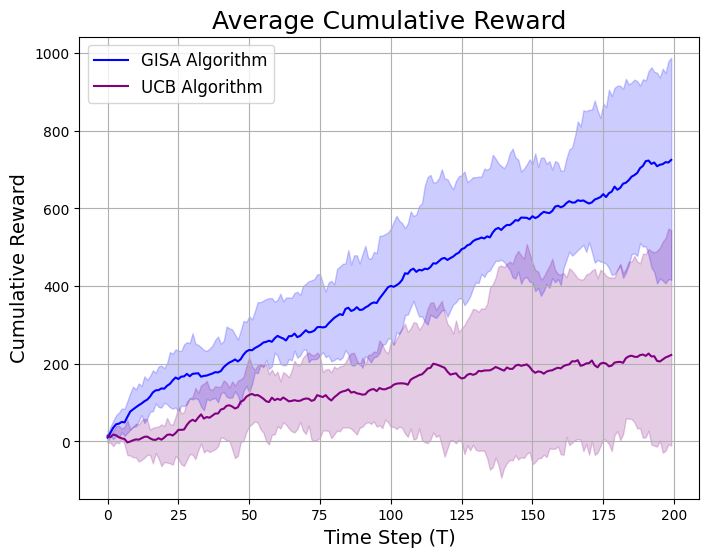}
\endminipage\hfill
\minipage{0.43\textwidth}
  \includegraphics[width=\linewidth]{figures/ssg_reg_ex1_v2.png}
\endminipage\hfill
\caption*{Parameters: $\theta_{A} = \left[ -0.850, -0.049, 0.620, -0.535, -0.313 \right], \quad \theta_{B} = \left[ -1.554, -0.176, 0.576, 0.803, 0.358 \right], \sigma = 0.1$}
\minipage{0.43\textwidth}
  \includegraphics[width=\linewidth]{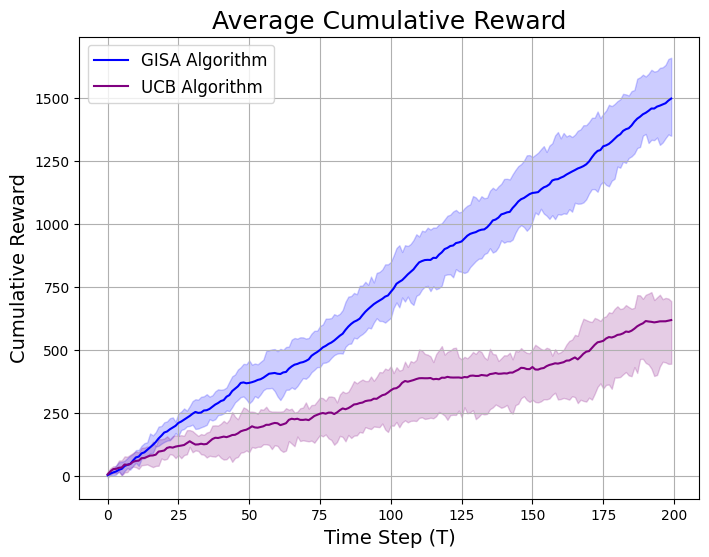}
\endminipage\hfill
\minipage{0.43\textwidth}
  \includegraphics[width=\linewidth]{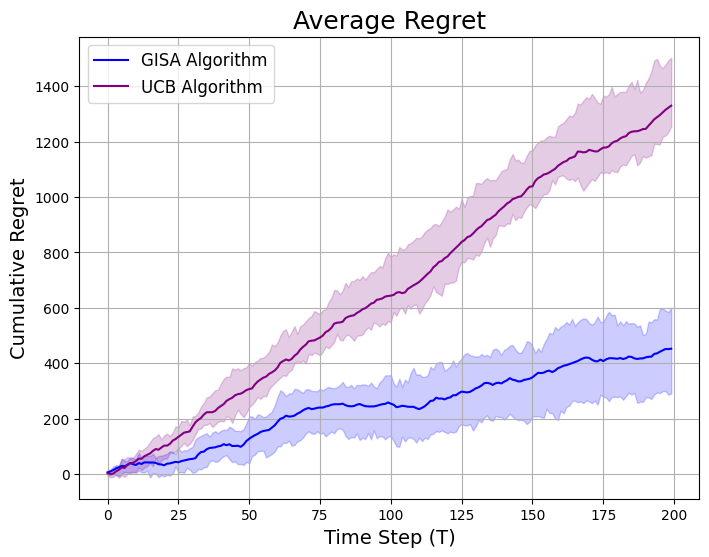}
\endminipage\hfill
\caption*{Parameters: $\theta_{A} = \left[ -1.557, -0.011, 0.821, -1.307, -0.262 \right], \quad \theta_{B} = \left[ -1.499, 0.317, -0.106, 0.465, -0.476 \right], \sigma = 0.1$}
\minipage{0.43\textwidth}
  \includegraphics[width=\linewidth]{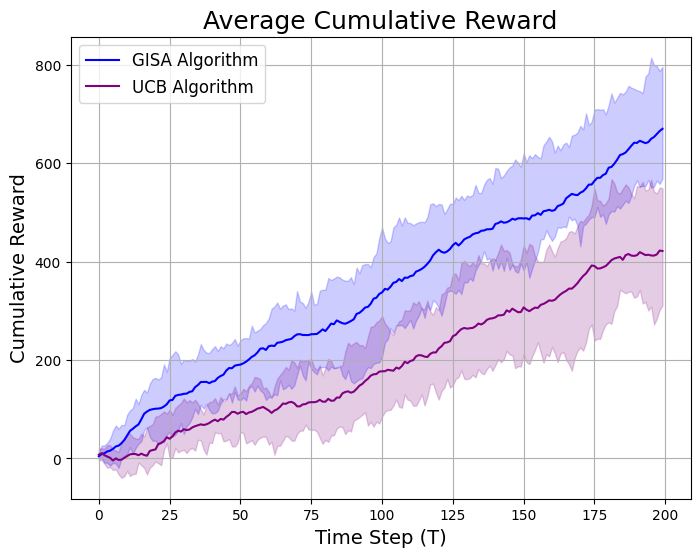}
\endminipage\hfill
\minipage{0.43\textwidth}
  \includegraphics[width=\linewidth]{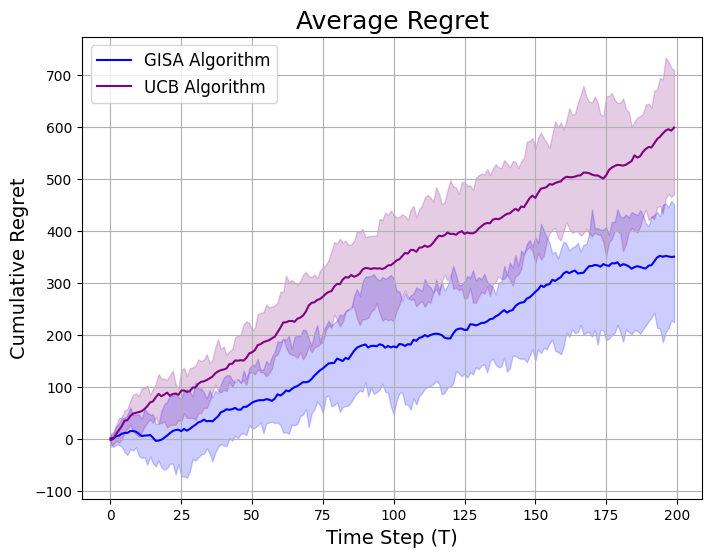}
\endminipage\hfill
\caption*{Parameters: $\theta_{A} = \left[ -0.599, -0.951, 0.156, -0.732, 0.375 \right], \quad \theta_{B} = \left[ -0.866, 0.708, -0.156, 0.601, -0.058 \right], \sigma = 0.1$}
\caption{Mean values are computed over 1,000 trials. All shaded areas, denoting confidence intervals, are within a quarter quantile. UCB arms were discretized to increments of 200, with an exploration constant $\alpha_{UCB} = 0.01.$ $\stEmb$ dimension is set to $D = 10$.} \label{fig:r5_experimental_results}
\end{figure}

\clearpage





\end{document}